\documentclass{article}

\usepackage{microtype}
\usepackage{graphicx}
\usepackage{booktabs} 

\usepackage{hyperref}

\usepackage[utf8]{inputenc} 
\usepackage[T1]{fontenc}    
\usepackage{hyperref}       
\usepackage{url}            
\usepackage{booktabs}       
\usepackage{amsfonts}       
\usepackage{nicefrac}       
\usepackage{microtype}      
\usepackage{float}
\usepackage[english]{babel}
\usepackage{mathtools}
\usepackage{verbatim}
\usepackage{hyperref}
\usepackage{bm}
\usepackage{tabu}
\usepackage{enumitem}
\usepackage{mathrsfs}
\usepackage{enumitem}
\usepackage{amsfonts}
\usepackage{graphicx}
\usepackage{parskip}
\usepackage{color}
\usepackage{tikz}
\usepackage{amssymb}
\usepackage{bbm}
\usepackage{caption}
\usepackage{subcaption}

\newtheorem{theorem}{Theorem}
\newtheorem{lemma}{Lemma}
\newtheorem{remark}{Remark}
\newtheorem{assumption}{Assumption}

\newtheorem{corollary}[lemma]{Corollary}
\newtheorem{definition}{Definition}
\newenvironment{proof}[1][Proof]{\begin{trivlist}
\item[\hskip \labelsep {\bfseries #1}]$ $\newline }{\qed\end{trivlist}}

\newcommand*\del[0]{\partial}

\renewcommand*\div[0]{\textbf{div}}
\newcommand*\ddt[0]{\frac{d}{d t}}

\newcommand*\tr[0]{\text{tr}}

\newcommand*\lin[1]{\bm{\left\langle} #1 \bm{\right\rangle}}

\newcommand*\E[1]{\mathbb{E}\left[{#1}\right]}
\newcommand*\Ep[2]{\mathbb{E}_{#1}\left[#2\right]}

\newcommand*\Pspace[0]{\mathscr{P}}

\newcommand*\N[0]{\mathcal{N}}

\newcommand\numberthis{\addtocounter{equation}{1}\tag{\theequation}}
\newcommand*\lrb[1]{{\left[#1\right]}}
\newcommand*\lrbb[1]{\left\{#1\right\}}
\newcommand*\lrp[1]{{\left(#1\right)}}
\newcommand*\lrn[1]{{\left\|#1\right\|}}
\newcommand*\lrabs[1]{{\left|#1\right|}}

\newcommand*\cvec[2]{\begin{bmatrix} #1\\#2\end{bmatrix}}

\newcommand*\ind[1]{{\mathbbm{1}\lrbb{#1}}}

\newcommand*{\qed}{\hfill\ensuremath{\blacksquare}}
\renewcommand*{\Re}{\mathbb{R}}

\newcommand*{\R}{\mathcal{R}}

\newcommand*\circled[1]{\tikz[baseline=(char.base)]{
\node[shape=circle,draw,inner sep=1pt] (char) {#1};}}

\newcommand*{\twocase}[2]{\left\{\begin{array}{ll}
        #1\\
        #2
        \end{array}\right.}

\newcommand*{\threecase}[6]{\left\{\begin{array}{ll}
        #1, & \text{for } #2\\
        #3, & \text{for } #4\\
        #5, & \text{for } #6
        \end{array}\right.}
\newcommand*{\fourcase}[8]{\left\{\begin{array}{ll}
        #1, & \text{for } #2\\
        #3, & \text{for } #4\\
        #5, & \text{for } #6\\
        #7, & \text{for } #8
        \end{array}\right.}

\newcommand*{\bx}{\bar{x}}
\newcommand*{\by}{\bar{y}}

\newcommand*{\bw}{\bar{w}}

\newcommand{\aq}{\alpha_q}

\newcommand{\Rq}{{\mathcal{R}_q}}

\newcommand{\bo}{\bigotimes}

\newcommand{\cm}{c_m}

\newcommand{\LR}{L_R}
\newcommand{\LN}{L_N}

\renewcommand{\by}{\bar{y}}
\newcommand{\Law}{\mathsf{Law}}

\usepackage[accepted]{icml2020}

\icmltitlerunning{Stochastic Gradient and Langevin Processes}

\begin{document}

\twocolumn[
\icmltitle{Stochastic Gradient and Langevin Processes}



\icmlsetsymbol{equal}{*}

\begin{icmlauthorlist}
\icmlauthor{Xiang Cheng}{uc}
\icmlauthor{Dong Yin}{uc}
\icmlauthor{Peter Bartlett}{uc}
\icmlauthor{Michael Jordan}{uc}
\end{icmlauthorlist}

\icmlcorrespondingauthor{Xiang Cheng}{x.cheng@berkeley.edu}

\icmlaffiliation{uc}{Department of Electrical Engineering and Computer Science, University of California, Berkeley}

\icmlkeywords{Machine Learning, ICML}

\vskip 0.3in
]



\printAffiliationsAndNotice{} 

\begin{abstract}
We prove quantitative convergence rates at which discrete Langevin-like processes converge to the invariant distribution of a related stochastic differential equation. We study the setup where the additive noise can be non-Gaussian and state-dependent and the potential function can be non-convex. We show that the key properties of these processes depend on the potential function and the second moment of the additive noise. We apply our theoretical findings to studying the convergence of Stochastic Gradient Descent (SGD) for non-convex problems and corroborate them with experiments using SGD to train deep neural networks on the CIFAR-10 dataset.
\end{abstract}
\vspace{-0.15in}
\begin{section}{Introduction}
Stochastic Gradient Descent (SGD) is one of the workhorses of modern machine learning. In many nonconvex optimization problems, such as training deep neural networks, SGD is able to produce solutions with good generalization error; indeed, there is evidence that the generalization error of an SGD solution can be significantly better than that of Gradient Descent (GD) \cite{keskar2016large, jastrzkebski2017three, he2019control}. 
This suggests that, to understand the behavior of SGD, it is not enough to consider the limiting cases such as small step size or large batch size where it degenerates to GD. In this paper, we take an alternate view of SGD as a sampling algorithm, and aim to understand its convergence to an appropriate stationary distribution.

There has been rapid recent progress in understanding the finite-time behavior of MCMC methods, by comparing them to stochastic differential equations (SDEs), such as the Langevin diffusion. It is natural in this context to think of SGD as a discrete-time approximation of an SDE.  There are, however, two significant barriers to extending previous analyses to the case of SGD. First, these analysis are often restricted to isotropic Gaussian noise, whereas the noise in SGD can be far from Gaussian. 
Second, the noise depends significantly on the current state (the optimization variable). For instance, if the objective is an average over training data with a nonnegative loss, as the objective approaches zero the variance of minibatch SGD goes to zero. Any attempt to cast SGD as an SDE must be able to handle this kind of noise.

This motivates the study of Langevin MCMC-like methods that have a state-dependent noise term:
\begin{align*}
w_{(k+1)\delta} = w_{k\delta} - \delta \nabla U(w_{k\delta}) + \sqrt{\delta} \xi(w_{k\delta}, \eta_k),
\numberthis \label{e:discrete-clt}
\end{align*}
where $ w_t \in \Re^d $ is the state variable at time $t$, $\delta$ is the step size, $U:\Re^d\to\Re$ is a (possibly nonconvex) potential, $\xi:\Re^d \times \Omega \to \Re^d$ is the {\em noise function}, and $\eta_k$ are sampled i.i.d.\ according to some distribution over $\Omega$ (for example, in minibatch SGD, $\Omega$ is the set of subsets of indices in the training sample).

Throughout this paper, we assume that $\Ep{\eta}{\xi(x,\eta)}=0$ for all $x$. We define a matrix-valued function $M(\cdot):\Re^d \to \Re^{d\times d}$ to be the square root of the covariance matrix of $\xi$; i.e., for all $x$, $M(x) := \sqrt{\Ep{\eta}{\xi(x,\eta) \xi(x,\eta)^T}}$, where for a positive semidefinite matrix $G$, $A=\sqrt{G}$ is the unique positive semidefinite matrix such that $A^2 = G$.

In studying the generalization behavior of SGD, earlier work~\cite{jastrzkebski2017three,he2019control} propose that \eqref{e:discrete-clt} be approximated by the stochastic process $y_{(k+1)\delta} = y_{k\delta} - \delta \nabla U(y_{k\delta}) + \sqrt{\delta} M(y_{k\delta}) \theta_k$ where $\theta_k\sim \N(0,I)$, or, equivalently:
\begin{align*}
&d y_t = -\nabla U(y_{k\delta}) dt + M(y_{k\delta}) dB_t
\numberthis \label{e:discrete-langevin}\\
&\quad \text{for } t\in[k\delta, (k+1)\delta],
\end{align*}
with $B_t$ denoting standard Brownian motion~\cite{karatzas1998brownian}. Specifically, the non-Gaussian noise $\xi(\cdot,\eta)$ is approximated by a Gaussian variable $M(\cdot) \theta$ with the same covariance, via an assumption that the minibatch size is large and an appeal to the central limit theorem.

The process in \eqref{e:discrete-langevin} can be seen as the Euler-Murayama discretization of the following SDE:
\begin{align*}
\numberthis \label{e:exact-sde}
d x_t = -\nabla U(x_t) dt + M(x_t) dB_t.
\end{align*}
We let $p^*$ denote the invariant distribution of \eqref{e:exact-sde}.

We prove quantitative bounds on the discretization error between \eqref{e:discrete-langevin}, \eqref{e:discrete-clt} and \eqref{e:exact-sde}, as well as convergence rates of \eqref{e:discrete-langevin} and \eqref{e:discrete-clt} to $p^*$. Our bounds are in Wasserstein-1 distance (denoted by $W_1(\cdot, \cdot)$ in the following). We present the full theorem statements in Section \ref{s:main_results}, and summarize our contributions below:

\begin{enumerate}[leftmargin=3mm]
\item
In Theorem \ref{t:main_gaussian}, we bound the discretization error between \eqref{e:discrete-langevin} and \eqref{e:exact-sde}. Informally, Theorem \ref{t:main_gaussian} states:
\begin{align*}
    &1.\ \textit{If $x_0 = y_0$, then for all $k$, $W_1(x_{k\delta}, y_{k\delta}) = O(\sqrt{\delta})$ }; \\
    &2.\ \textit{For $n \geq \tilde{O}\lrp{\frac{1}{\delta}}$, $W_1(p^*, \Law(y_{n\delta})) = O(\sqrt{\delta})$},
\end{align*}
where $\Law(\cdot)$ denotes the distribution of a random vector. This is a crucial intermediate result that allows us to prove the convergence of \eqref{e:discrete-clt} to \eqref{e:exact-sde}. We highlight that the variable diffusion matrix: 1) leads to a very large discretization error, due to the scaling factor of $\sqrt{\delta}$ in the $M(y_{k\delta}) \theta_k$ noise term, and 2) makes the stochastic process non-contractive (this is further compounded by the non-convex drift). Our convergence proof relies on a carefully constructed Lyapunov function together with a specific coupling. Remarkably, the $\epsilon$ dependence in our iteration complexity is the same as that in Langevin MCMC with constant isotropic diffusion \cite{durmus2016high}.

\item In Theorem \ref{t:main_nongaussian}, we bound the discretization error between \eqref{e:discrete-clt} and \eqref{e:exact-sde}.
Informally, Theorem \ref{t:main_nongaussian} states:
\begin{align*}
    &1.\ \textit{If $x_0 = w_0$, then for all $k$, $W_1(x_{k\delta}, w_{k\delta}) = O(\delta^{1/8})$ }; \\
    &2.\ \textit{For $n \geq \tilde{O}\lrp{\frac{1}{\delta}}$, $W_1(p^*, \Law(w_{n\delta})) = O(\delta^{1/8})$}.
\end{align*}
Notably, the noise in each step of \eqref{e:discrete-clt} may be far from Gaussian, but for sufficiently small step size, \eqref{e:discrete-clt} is nonetheless able to approximate \eqref{e:exact-sde}. This is a weaker condition than earlier work, which must assume that the batch size is sufficiently large so that CLT ensures that the per-step noise is approximately Gaussian.

\item 
Based on Theorem 2, we predict that for sufficiently small $\delta$, two different processes of the form \eqref{e:discrete-clt} will have similar distributions if their noise terms $\xi$ have the same covariance matrix, as that leads to the same limiting SDE \eqref{e:exact-sde}. In Section \ref{s:sgd}, we evaluate this claim empirically: we design a family of SGD-like algorithms and evaluate their test error at convergence. We observe that the noise covariance alone is a very strong predictor for the test error, regardless of higher moments of the noise. This corroborates our theoretical prediction that the noise covariance approximately determines the distribution of the solution. This is also in line with, and extends upon, observations in earlier work that the ratio of batch size to learning rate correlates with test error~\cite{jastrzkebski2017three,he2019control}.
\end{enumerate}
\end{section}

\begin{section}{Related Work}
Previous work has drawn connections between SGD noise and generalization~\cite{mandt2016variational, jastrzkebski2017three, he2019control, hoffer2017train,keskar2016large}. Notably, \citet{mandt2016variational, he2019control, jastrzkebski2017three} analyze favorable properties of SGD noise by arguing that in the neighborhood of a local minimum, \eqref{e:discrete-langevin} is roughly the discretization of an Ornstein-Uhlenbeck (OU) process, and so the distribution of $y_{k\delta}$ approximates is approximately Gaussian. However, empirical results \cite{keskar2016large, hoffer2017train} suggest that SGD generalizes better by finding better local minima, which may require us to look beyond the ``OU near local minimum'' assumption to understand the global distributional properties of SGD. Indeed, \citet{hoffer2017train} suggest that SGD performs a random walk on a random loss landscape, \citet{kleinberg2018alternative} propose that SGD noise helps smoooth out ``sharp minima.'' \citet{jastrzkebski2017three} further note the similarity between \eqref{e:discrete-clt} and an Euler-Murayama approximation of \eqref{e:exact-sde}. \citet{chaudhari2018stochastic} also made connections between SGD and SDE. Our work tries to make these connections rigorous, by quantifying the error between \eqref{e:exact-sde}, \eqref{e:discrete-langevin} and \eqref{e:discrete-clt}, without any assumptions about \eqref{e:exact-sde} being close to an OU process or being close to a local minimum.

Our work builds on a long line of work establishing the convergence rate of Langevin MCMC in different settings \cite{dalalyan2017theoretical,durmus2016high,ma2018sampling,gorham2016measuring,cheng2018sharp,erdogdu2018global,li2019stochastic}. We will discuss our rates in relation to some of this work in detail following our presentation of Theorem \ref{t:main_gaussian}. We note here that some of the techniques used in this paper were first used by \citet{eberle2011reflection, gorham2016measuring}, who analyzed the convergence of \eqref{e:exact-sde} to $p^*$ without log-concavity assumptions. \citet{erdogdu2018global} studied processes of the form \eqref{e:discrete-langevin} as an approximation to \eqref{e:exact-sde} under a distant-dissipativity assumption, which is similar to the assumptions made in this paper. For the sequence \eqref{e:discrete-langevin}, they prove an $O(1/\epsilon^2)$ iteration complexity to achieve $\epsilon$ integration error for any pseudo-Lipschitz loss $f$ with polynomial growth derivatives up to fourth order. In comparison, we prove $W_1$ convergence between $\Law(y_{k\delta})$ and $p^*$, which is equivalent to $\sup_{\lrn{\nabla f}_{\infty} \leq 1} \lrabs{\E{f(y_{k\delta})} - \Ep{y\sim p^*}{f(y)}}$, also with rate $\tilde{O}(1/\epsilon^2)$. By smoothing the $W_1$ test function, we believe that the results by~\citet{erdogdu2018global} can imply a qualitatively similar result to Theorem \ref{t:main_gaussian}, but with a worse dimension and $\epsilon$ dependence.

In concurrent work by~\citet{li2019stochastic}, the authors study a process based on a stochastic Runge-Kutta discretization scheme of \eqref{e:exact-sde}. They prove an $\tilde{O}\lrp{\frac{d}{\epsilon^{-2/3}}}$ iteration complexity to achieve $\epsilon$ error in $W_2$ for an algorithm based on Runge-Kutta discretization of \eqref{e:exact-sde}. They make a strong assumption of \textit{uniform dissipativity} (essentially assuming that the process \eqref{e:exact-sde} is uniformly contractive), which is much stronger than the assumptions in this paper, and may be violated in the settings of interest considered in this paper.

There has been a number of work \cite{chen2016statistical,li2018statistical,anastasiou2019normal} which establish CLT results for SGD with very small step size (rescaled to have constant variance). These work generally focus on the setting of "OU process near a local minimum", in which the diffusion matrix is constant.

Finally, a number of authors have studied the setting of heavy-tailed gradient noise in neural network training. \cite{zhang2019adam} showed that in some cases, the heavy-tailed noise can be detrimental to training, and a clipped version of SGD performs much better. \cite{simsekli2019tail} argue that when the SGD noise is heavy-tailed, it should not be modelled as a Gaussian random variable, but instead as an $\alpha$-stable random variable, and propose a Generalized Central Limit Theorem to analyze the convergence in distribution. Our paper does not handle the setting of heavy-tailed noise; our theorems require that the norm of the noise term uniformly bounded, which will be satisfied, for example, if gradients are explicitly clipped at a threshold, or if the optimization objective has Lipschitz gradients and the SGD iterates stay within a bounded region.
\end{section}

\begin{section}{Motivating Example}
     It is generally difficult to write down the invariant distribution of \eqref{e:exact-sde}. In this section, we consider a very simple one-dimensional setting which does admit an explicit expression for $p^*$, and serves to illustrate some remarkable properties of anisotropic diffusion matrices.

    Let us define $D(x) := M^2(x)$. Our analysis will be based on the Fokker-Planck equation, which states that $p^*$ is the invariant distribution of \eqref{e:exact-sde} if
    \begin{align*}
    0 = \div(p^*(x) \nabla U(x)) + \div\lrp{p^*(x) \Gamma(x) + D(x) \nabla p^*(x)},
    \numberthis \label{e:fokker_planck}
    \end{align*}
    where $\Gamma(x)$ is a vector whose $i^{th}$ coordinate equals $\sum_{j=1}^d \frac{\del}{\del x_j} \lrb{D(x)}_{i,j}$.
    In the one-dimensional setting, we can explicitly write down the density of $p^*(x)$. Note that in this case, $\Gamma(x) = \nabla D(x)$.
    Let $V(x) := \int_{0}^{x} \lrp{\frac{\nabla U(x)}{D(x)} + \frac{\nabla D(X)}{D(x)}} dx = \int_{0}^{x} \lrp{\frac{\nabla U(x)}{D(x)}} dx + \log D(x) - \log D(0)$. We can verify that $p^*(x) \propto e^{-V(x)}$ satisfies  \eqref{e:fokker_planck}.

    For a concrete example, let the potential $U(x)$ and the diffusion function $M(x)$ be defined as
    \begin{align*}
    & U(x) := \threecase{\frac{1}{2} x^2}{x\in[-1,4]}{\frac{1}{2} (x+2)^2 - 1}{x\leq -1}{\frac{1}{2} (x-8)^2 -16}{x\geq 4}  \\
    &  M(x) = \threecase
    {\frac{1}{2} (x+2)}{x\in[-2,8]}
    {1}{x \leq -2}
    {6}{x \geq 8}.
    \end{align*}

\begin{figure*}[h!]
  \centering
  \begin{subfigure}[b]{0.23\linewidth}
    \includegraphics[scale=0.2]{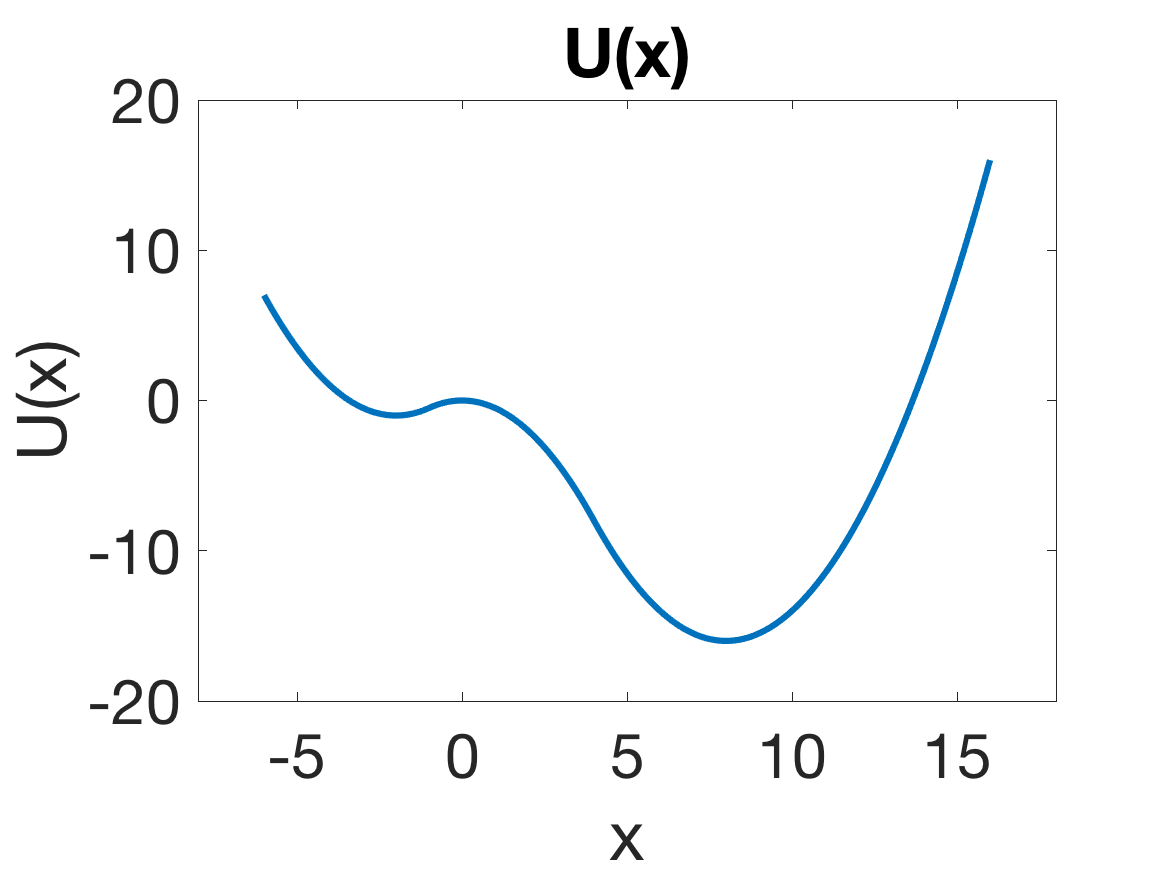}
     \caption{$U(x)$}\label{fig:1d_ux}
  \end{subfigure}
  ~
  \begin{subfigure}[b]{0.23\linewidth}
    \includegraphics[scale=0.2]{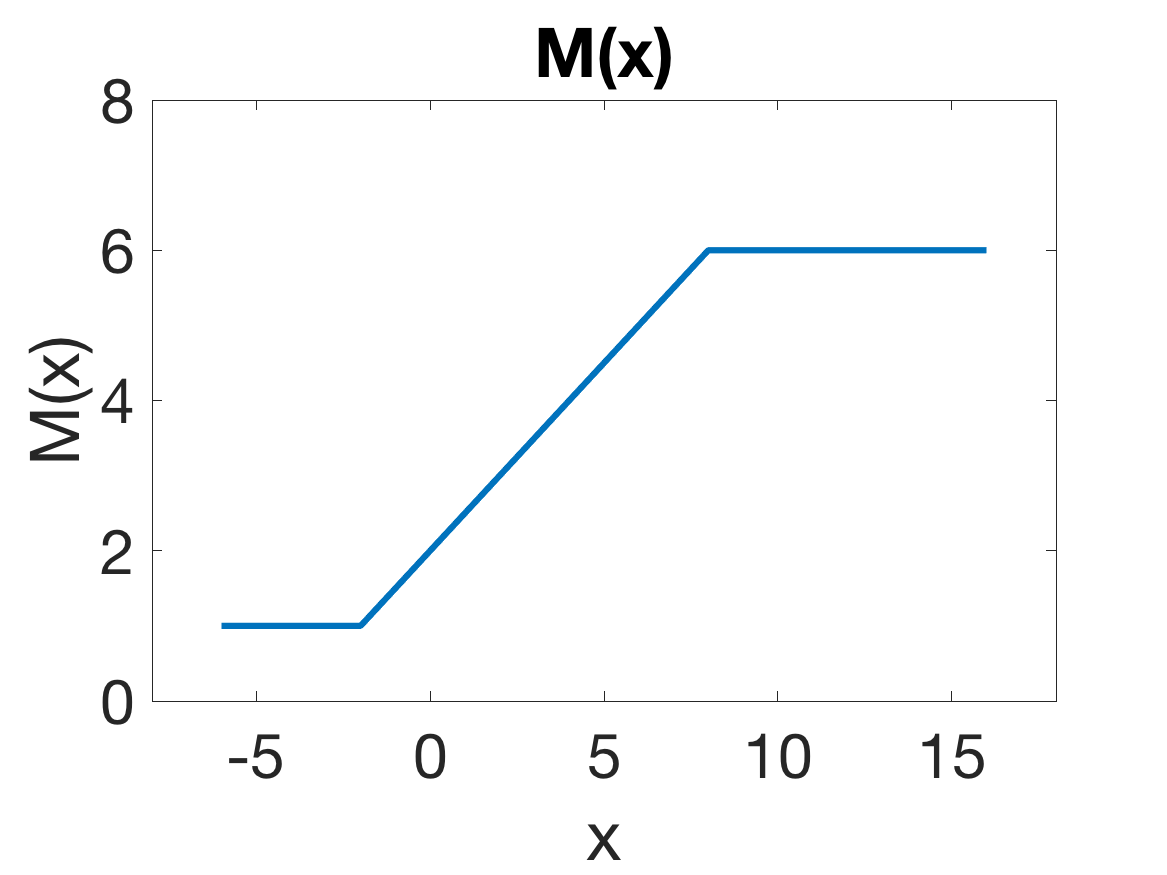}
    \caption{$M(x)$}\label{fig:1d_mx}
  \end{subfigure}
  ~
  \begin{subfigure}[b]{0.23\linewidth}
    \includegraphics[scale=0.2]{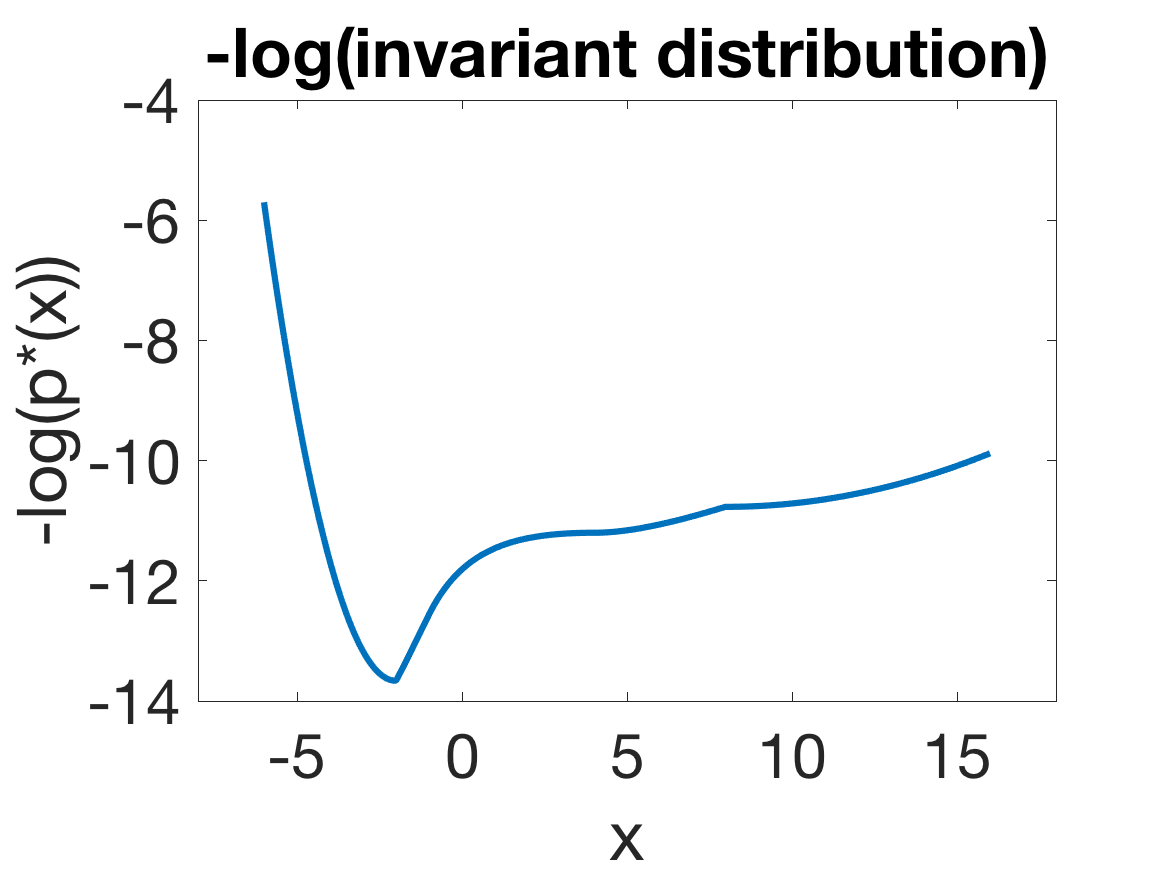}
    \caption{$V(x)$}\label{fig:1d_vx}
  \end{subfigure}
  ~
  \begin{subfigure}[b]{0.23\linewidth}
    \includegraphics[scale=0.2]{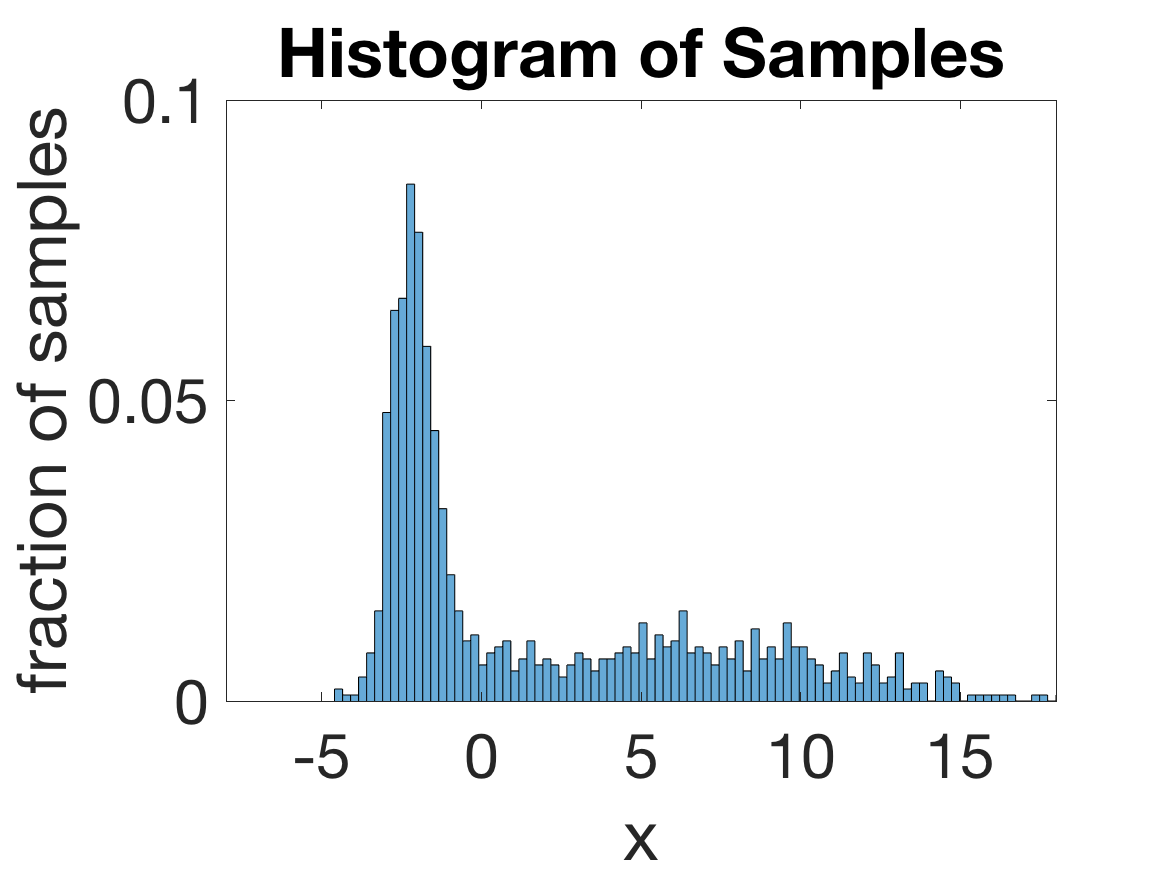}
    \caption{Samples}\label{fig:1d_histo}
  \end{subfigure}
  \caption{One-dimensional example exhibiting the importance of state-dependent noise: A simple construction showing how $M(x)$ can affect the shape of the invariant distribution. While $U(x)$ has two local minima, $V(x)$ only has the smaller minimum at $x=-2$. Figure~\ref{fig:1d_histo} represents samples obtained from simulating using the process~\eqref{e:discrete-langevin}. We can see that most of the samples concentrate around $x=-2$.}
  	\label{fig:1d example}
\end{figure*}

We plot $U(x)$ in Figure \ref{fig:1d_ux}. Note that $U(x)$ has two local minima: a shallow minimum at $x=-2$ and a deeper minimum at $x=8$. A plot of $M(x)$ can be found in Figure \ref{fig:1d_mx}. $M(x)$ is constructed to have increasing magnitude at larger values of $x$. This has the effect of biasing the invariant distribution towards smaller values of $x$.

We plot $V(x)$ in Figure \ref{fig:1d_vx}. Remarkably, $V(x)$ has only one local minimum at $x=-2$. The larger minimum of $U(x)$ at $x=8$ has been smoothed over by the effect of the large diffusion $M(x)$. This is very different from when the noise is homogeneous (e.g., $M(x)=I$), in which case $p^*(x) \propto e^{-U(x)}$. We also simulate \eqref{e:exact-sde} (using \eqref{e:discrete-langevin}) for the given $U(x)$ and $M(x)$ for 1000 samples (each simulated for 1000 steps), and plot the histogram in Figure \ref{fig:1d_histo}.

\end{section}

\begin{section}{Assumptions and Definitions}
    In this section, we state the assumptions and definitions that we need for our main results in Theorem \ref{t:main_gaussian} and Theorem \ref{t:main_nongaussian}.
    \label{ss:ass}
        \begin{assumption}
            We assume that $U(x)$ satisfies
            \label{ass:U_properties}
            \begin{enumerate}
                \item The function $U(x)$ is continuously-differentiable on $\mathbb{R}^d$ and has
                Lipschitz continuous gradients; that is, there exists a positive constant $L\geq0$ such that \\
                for all $x,y\in \mathbb{R}^d$,
                $
                \lVert \nabla U(x) - \nabla U(y) \rVert_2 \le L \lVert x-y\rVert_2.
                $ 
                \item $U$ has a stationary point at zero: $\nabla U(0) = 0.$
                \item There exists
                a constant $m>0, \LR, R$ such that for all $\lrn{x-y}_2 \geq R$,
                \begin{equation}
                \label{e:convexity_outside_ball}
                \lin{ \nabla U(x) - \nabla U(y),x-y} \geq m \lrn{x-y}_2^2.
                \end{equation}
                and for all $\lrn{x-y}_2 \leq R$, 
                $
                \lVert \nabla U(x) - \nabla U(y) \rVert_2 \le \LR \lVert x-y\rVert_2.
                $
            \end{enumerate}
        \end{assumption}

        \begin{remark}
        This assumption, and minor variants, is common in the nonconvex sampling literature \cite{eberle2011reflection,eberle2016reflection,cheng2018sharp,ma2018sampling,erdogdu2018global,gorham2016measuring}.
        \end{remark}

        \begin{assumption}\label{ass:xi_properties}
            We make the following assumptions on $\xi$ and $M$:
            \begin{enumerate}
                \item For all $x$, 
                $\E{\xi(x,\eta)}=0$.
                \item For all $x$, $\lrn{\xi(x,\eta)}_2 \leq \beta$ almost surely.
                \item For all $x,y$,
                $\lrn{\xi(x,\eta) - \xi(y,\eta)}_2\leq L_{\xi} \lrn{x-y}_2$ almost surely.
                \item There is a positive constant $\cm$ such that for all $x$, $2 \cm I \prec M(x)$.
            \end{enumerate}
        \end{assumption}

        \begin{remark}
            We discuss these assumptions in a specific setting in Section~\ref{ss:example_ass}.
        \end{remark}

        For convenience we define a matrix-valued function $N(\cdot) : \Re^d \to \Re^{d\times d}$:
        \begin{align*}
        N(x) := \sqrt{M(x)^2 - \cm^2 I}.
        \numberthis
        \label{d:N}
        \end{align*}
        Under Assumption \ref{ass:U_properties}, we can prove that $N(x)$ and $M(x)$ are bounded and Lipschitz (see Lemma \ref{l:M_is_regular} and \ref{l:N_is_regular} in Appendix \ref{ss:mnregularity}). These properties will be crucial in ensuring convergence.
        
        Given an arbitrary sample space $\Omega$ and any two distribution $p\in \Pspace\lrp{\Omega}$ and $q\in \Pspace\lrp{\Omega}$, a joint distribution $\zeta \in \Pspace\lrp{\Omega \times \Omega}$ is a \textit{coupling} between $p$ and $q$ if its marginals are equal to $p$ and $q$ respectively.
        
        For a matrix, we use $\lrn{G}_2$ to denote the operator norm:
        $
        \lrn{G}_2 = \sup_{v\in \Re^d, \|v\|_2 =1} \lrn{Gv}_2.
        $.


            Finally, we define a few useful constants which will be used throughout the paper:
            \begin{align*}
            &\LN := \frac{4\beta L_\xi}{\cm}, \ \  \aq:=\frac{\LR+L_N^2}{2\cm^2},  \\ &\Rq:=\max\lrbb{R,{\frac{16\beta^2 L_N}{m\cdot \cm}}}  \\
            &\lambda :=\min\lrbb{\frac{m}{2}, \frac{2\cm^2}{32\Rq^2}}\exp\lrp{-\frac{7}{3}\aq\Rq^2}.
            \numberthis \label{d:constants}
            \end{align*}

            $\LN$ is the smoothness parameter of the matrix $N(x)$, and we show in Lemma \ref{l:N_is_regular} that $\tr\lrp{\lrp{N(x) - N(y)}^2} \leq \LN^2\lrn{x-y}_2^2$. The constants $\aq$ and $\Rq$ are used to define a Lyapunov function $q$ in Appendix \ref{ss:defining-q}. A key step in our proof uses the fact that, under the dynamics \eqref{e:discrete-langevin}, $q$ contracts at a rate of $e^{-\lambda}$, plus discretization error.

\end{section}

\begin{section}{Main Results}\label{s:main_results}
In this section, we present our main convergence results beginning with convergence under Gaussian noise and proceeding to the non-Gaussian case.

    \begin{theorem}
        \label{t:main_gaussian}
        Let $x_t$ and $y_t$ have dynamics as defined in \eqref{e:exact-sde} and \eqref{e:discrete-langevin} respectively, and suppose that the initial conditions satisfy $\E{\lrn{x_0}_2^2}\leq R^2 + \beta^2/m$ and $\E{\lrn{y_0}_2^2}\leq R^2 + \beta^2/m$. Let $\hat{\epsilon}$ be a target accuracy satisfying $\hat{\epsilon} \leq  \lrp{\frac{16\lrp{L + \LN^2}}{\lambda}} \cdot \exp\lrp{7\aq\Rq/3} \cdot \frac{\Rq}{\aq\Rq^2 + 1}$. Let $\delta$ be a step size satisfying
        \begin{align*}
        \delta \leq \min\twocase{\frac{\lambda^2 \hat{\epsilon}^2}{512 \beta^2\lrp{L^2 + \LN^4}\exp\lrp{\frac{14\aq\Rq^2}{3}}}}{\frac{2\lambda \hat{\epsilon}}{(L^2+\LN^4)\exp\lrp{\frac{7\aq\Rq^2}{3}}\sqrt{R^2 + \beta^2/m}}}.
        \end{align*}

        If we assume that $x_0 = y_0$, then there exists a coupling between $x_t$ and $y_t$ such that for any $k$,
        \begin{align*}
        \E{\lrn{x_{k\delta} - y_{k\delta}}_2} \leq \hat{\epsilon}
        \end{align*}

        Alternatively, if we assume $n \geq \frac{ 3\aq\Rq^2}{\delta} \log \frac{R^2 + \beta^2/m}{\hat{\epsilon}}$, then
        \begin{align*}
        W_1\lrp{p^*, p^y_{n\delta}} \leq 2\hat{\epsilon}
        \end{align*}
        where $p^y_t := \Law(y_t)$.
    \end{theorem}

    \begin{remark}
Note that $m,L,R$ are from Assumption \ref{ass:U_properties}, $L_N$ is from \eqref{d:constants}, $\cm, \beta,L_\xi$ are from Assumption \ref{ass:xi_properties}).
\end{remark}
\begin{remark}
Finding a suitable $y_0$ can be done very quickly using gradient descent wrt $U(\cdot)$. The convergence rate to the ball of radius $R$ is very fast, due to Assumption \ref{ass:U_properties}.3.
\end{remark}

After some algebraic simplifications, we see that for a sufficiently small $\hat{\epsilon}$, achieving $W_1(p^y_{n\delta}, p^*) \leq \hat{\epsilon}$ requires number of steps
\begin{align*}
n = \tilde{O}\left(\frac{\beta^2}{\hat{\epsilon}^2} \cdot \exp\left(\frac{14}{3} \cdot \lrp{\frac{\LR}{\cm^2} + \frac{16\beta^2 L_\xi^2}{\cm^4}} \right. \right.\\
\left.\left. \cdot\max\lrbb{R^2, \frac{2^{12} \beta^6 L_\xi^2}{m^2 \cm^4}}\right)\right).
\end{align*}

\begin{remark}
The convergence rate contains a term $e^{R^2}$; this term is also present in all of the work cited in the previous section under Remark 1. Given our assumptions, in particular  \ref{e:convexity_outside_ball}, this dependence is unavoidable as it describes the time to transit between two modes of the invariant distribution.  It can be verified to be tight by considering a simple double-well potential.
\end{remark}
\begin{remark}
As illustrated in Section \ref{ss:example_ass}, the $m$ from Assumption \ref{ass:xi_properties}.3 should be thought of as a regularization term which can be set arbitrarily large. In the following discussion, we will assume that $ \max\lrbb{R^2, \frac{\beta^6 L_\xi^2}{m^2 \cm^4}}$ is dominated by the $R^2$ term.
\end{remark}

To gain intuition about this term, let's consider what it looks like under a sequence of increasingly weaker assumptions:


\textbf{a. Strongly convex, constant noise}: $U(x)$ $m$-strongly convex, $L$-smooth, $\xi(x,\eta)\sim \N(0,I)$ for all $x$. (In reality we need to consider a truncated Gaussian so as not to violate Assumption \ref{ass:xi_properties}.2, but this is a minor issue). In this case, $L_\xi=0$, $\cm = 1$, $R=0$, $\beta = \tilde{O}(\sqrt{d})$, so $k=O(\frac{d}{\hat{\epsilon}^2})$.  This is the same rate as obtained by \citet{durmus2016high}. We remark that \citet{durmus2016high} obtain a $W_2$ bound which is stronger than our $W_1$ bound.

\textbf{b. Non-convex, constant noise}: $U(x)$ not strongly convex but satisfies Assumption \ref{ass:U_properties}, and $\xi(x,\eta) \sim \N(0,I)$. In this case, $L_\xi=0$, $\cm=1$, $\beta = \tilde{O}(\sqrt{d})$
This is the setting studied by \citet{cheng2018sharp} and \citet{ma2018sampling}. The rate we recover is $k=\tilde{O}\lrp{\frac{d}{\hat{\epsilon^2}} \cdot \exp\lrp{\frac{14}{3}LR^2}}$, which is in line with \citet{cheng2018sharp}, and is the best $W_1$ rate obtainable from \citet{ma2018sampling}.

\textbf{c. Non-convex, state-dependent noise}: $U(x)$ satisfies Assumption \ref{ass:U_properties}, and $\xi$ satisfies Assumption \ref{ass:xi_properties}. To simplify matters, suppose the problem is rescaled so that $\cm =1$. Then the main additional term compared to setting b. above is $\exp\lrp{\frac{64\beta^2 L_\xi^2 R^2}{\cm^4}}$. This suggests that the effect of a $L_\xi$-Lipschitz noise can play a similar role in hindering mixing as a $\LR$-Lipschitz nonconvex drift.


When the dimension is high, computing $M(y_k)$ can be difficult, but if for each $x$, one has access to samples whose covariance is $M(x)$, then one can approximate $M(y_k) \theta_k$ via the central limit theorem by drawing a sufficiently large number of samples. The proof of Theorem \ref{t:main_gaussian} can be readily modified to accommodate this (see Appendix \ref{ss:simlutating_discrete_sde}).

We now turn to the non-Gaussian case.
 \begin{theorem}
            \label{t:main_nongaussian}
            Let $x_t$ and $w_t$ have dynamics as defined in \eqref{e:exact-sde} and \eqref{e:discrete-clt} respectively, and suppose that the initial conditions satisfy $\E{\lrn{x_0}_2^2}\leq R^2 + \beta^2/m$ and $\E{\lrn{w_0}_2^2}\leq R^2 + \beta^2/m$. Let $\hat{\epsilon}$ be a target accuracy satisfying $\hat{\epsilon} \leq  \lrp{\frac{16\lrp{L + \LN^2}}{\lambda}} \cdot \exp\lrp{7\aq\Rq/3} \cdot \frac{\Rq}{\aq\Rq^2 + 1}$. Let $\epsilon:= \frac{\lambda}{16 (L+\LN^2)} \exp\lrp{-\frac{7\aq\Rq^2}{3}} \hat{\epsilon}$. Let $T:= \min\lrbb{\frac{1}{16L}, \frac{\beta^2}{8L^2\lrp{R^2 + \beta^2/m}}, \frac{\epsilon}{32\sqrt{L} \beta}, \frac{\epsilon^2}{128\beta^2}, \frac{\epsilon^4 \LN^2}{2^{14}\beta^2 \cm^2}}$ and let $\delta$ be a step size satisfying
            \begin{align*}
            \delta \leq \min\lrbb{\frac{T\epsilon^2L}{36 d\beta^2\log \lrp{ \frac{36 d\beta^2}{\epsilon^2L}}},  \frac{T\epsilon^4L^2} {2^{14} d\beta^4\log\lrp{\frac{2^{14} d\beta^4}{\epsilon^4L^2}}}}.
            \end{align*}

            If we assume that $x_0 = w_0$, then there exists a coupling between $x_t$ and $w_t$ such that for any $k$,
            \begin{align*}
            \E{\lrn{x_{k\delta} - w_{k\delta}}_2} \leq \hat{\epsilon}.
            \end{align*}

            Alternatively, if we assume that
            $n \geq \frac{3\aq\Rq^2}{\delta} \cdot  \log \frac{R^2 + \beta^2/m}{\hat{\epsilon}}$, then
            \begin{align*}
            W_1\lrp{p^*, p^w_{n\delta}} \leq 2\hat{\epsilon},
            \end{align*}
            where $p^w_t := \Law(w_t)$.
        \end{theorem}
        \begin{remark}
        To achieve $W_1(p^*, p^w_{n\delta}) \leq \hat{\epsilon}$, the number of steps needed is of order $n= \tilde{O}\lrp{\frac{1}{\hat{\epsilon}^8} \cdot e^{29 \aq \Rq^2}}$. The $\hat{\epsilon}$ dependency is considerably worse than in Theorem \ref{t:main_gaussian}. This is because we need to take many steps of \eqref{e:discrete-clt} in order to approximate a single step of \eqref{e:discrete-langevin}. For details, see the coupling construction in equations~\eqref{e:coupled_4_processes_x}--\eqref{e:coupled_4_processes_w} of Appendix \ref{s:nongaussianproof}.
        \end{remark}

\end{section}

\begin{section}{Application to Stochastic Gradient Descent}\label{s:sgd}

In this section, we will cast SGD in the form of \eqref{e:discrete-clt}. We consider an objective of the form
\begin{align*}
\numberthis
\label{e:example_U}
U(w) = \frac{1}{n} \sum_{i=1}^n U_i(w).
\end{align*}
We reserve the letter $\eta$ to denote a random minibatch from $\lrbb{1, \ldots, n}$, sampled with replacement, and define $\zeta(w,\eta)$ as follows:
\begin{equation}\label{e:def_zeta}
\zeta(w,\eta) := \nabla U(w) - \frac{1}{\lrabs{\eta}} \sum_{i\in \eta} \nabla U_i(w)
\end{equation}
For a sample of size one, i.e. $\lrabs{\eta}=1$, we define
\begin{align*}
H(w) := \E{\zeta(w,\eta) \zeta(w,\eta)^T }
\numberthis\label{e:m0}
\end{align*}
as the covariance matrix of the difference between the true gradient and a single sampled gradient at $w$. A standard run of SGD, with minibatch size $b := \lrabs{\eta_k} $, then has the following form:
\begin{align*}
w_{k+1} &= w_k  - \delta \frac{1}{b} \sum_{i\in\eta_k} \nabla U_{i}(w_k)\\
&= w_k  - \delta \nabla U(w_k) + \sqrt{\delta}\lrp{\sqrt{\delta} \zeta (w_k, \eta_k)}.
\numberthis \label{e:sgd}
\end{align*}
We refer to an SGD algorithm with step size $\delta$ and minibatch size $b$ a $(\delta, b)$-SGD. Notice that $\eqref{e:sgd}$ is in the form of \eqref{e:discrete-clt}, with $\xi(w,\eta) = \sqrt{\delta} \zeta (w,\eta)$. The covariance matrix of the noise term is
\begin{equation}\label{e:noise_cov_sgd}
    \E{\xi (w,\eta)\xi (w,\eta)^T} = \frac{\delta}{b} H(w).
\end{equation}
Because the magnitude of the noise covariance scales with $\sqrt{\delta}$, it follows that as $\delta \to 0$, \eqref{e:sgd} converges to deterministic gradient flow. However, the loss of randomness as $\delta\to 0$ is not desirable as it has been observed that as SGD approaches GD, through either small step size or large batch size, the generalization error goes up \cite{jastrzkebski2017three,he2019control,keskar2016large,hoffer2017train}; this is also consistent with our experimental observations in Section~\ref{ss:acc_rel_var}.

Therefore, a more meaningful way to take the limit of SGD is to hold the noise term constant in \eqref{e:sgd}. More specifically, we define the \textbf{\textit{constant-noise limit }} of \eqref{e:sgd} as
\begin{equation*}
d x_t = - \nabla U(x_t) dt +  M(x_t) dB_t,
\numberthis
\label{e:sgd-limit}
\end{equation*}
where $M(x) := \sqrt{\frac{\delta}{b}H(x)}$. Note that this is in the form of \eqref{e:exact-sde}, with noise covariance $M(x_t)^2$ matching that of SGD in~\eqref{e:sgd}. Using Theorem~\ref{t:main_nongaussian}, we can bound the $W_1$ distance between the SGD iterates $w_{k}$ from ~\eqref{e:sgd}, and the continuous-time SDE $x_t$ from ~\eqref{e:sgd-limit}.

\subsection{Importance of Noise Covariance}\label{ss:noise_covariance}
We highlight the fact that the limiting SDE of a discrete process,
\begin{equation}\label{e:disc-mcmc-new}
    w_{k+1} = w_k - s\nabla U(w_k) + \sqrt{s} \xi(w_k, \eta_k),
\end{equation}
depends only on the covariance matrix of $\xi$. More specifically, as long as $\xi$ satisfies $\sqrt{\E{\xi(w, \eta)\xi(w, \eta)^T}} = M(w)$, \eqref{e:disc-mcmc-new} will have \eqref{e:sgd-limit} as its limiting SDE, \textit{regardless of higher moments of $\xi$}. This fact, combined with Theorem \ref{t:main_nongaussian}, means that in the limit of $\delta \to 0$ and $k\to \infty$, the distribution of $w_k$ will be determined by the covariance of $\xi$ alone. An immediate consequence is the following: \emph{at convergence, the test performance of any Langevin MCMC-like algorithm is almost entirely determined by the covariance of its noise term.}

Returning to the case of SGD algorithms, since the noise covariance is $M(x)^2 = \frac{\delta}{b} H(x)$ (see \eqref{e:noise_cov_sgd}), we know that the ratio of step size $\delta$ to batch size $b$ is an important quantity which can dictate the test error of the algorithm; this observation has been made many times in prior work~\cite{jastrzkebski2017three, he2019control}, and our results in this paper are in line with these observations. Here, we move one step further, and provide experimental evidence to show that more fundamentally, it is the noise covariance in the constant-noise limit that controls the test error.

To verify this empirically, we propose the following algorithm called \emph{large-noise SGD}.

\begin{definition}\label{def:large_noise_sgd}
An $(s, \sigma, b_1, b_2)$-large-noise SGD is an algorithm that aims to minimize~\eqref{e:example_U} using the following updates:
\begin{align*}\numberthis\label{e:sgd-noisy}
    & w_{k+1} = w_k - \frac{s}{b_1} \sum_{i \in \eta_k} \nabla U_i(w_k) \\
    & \qquad + \frac{\sigma \sqrt{s}}{b_2} \left( \sum_{i\in \eta_k'}\nabla U_i(w_k) - \sum_{i\in\eta_k''}\nabla U_i(w_k) \right),
\end{align*}
where $\eta_k$, $\eta_k'$, and $\eta_k''$ are minibatches of sizes $b_1$, $b_2$, and $b_2$, sampled uniformly at random from $\{1, \ldots, n\}$ with replacement. The three minibatches are sampled independently and are also independent of other iterations.
\end{definition}

Intuitively, an $(s, \sigma, b_1, b_2)$-large-noise SGD should be considered as an SGD algorithm with step size $s$ and minibatch size $b_1$ and an additional noise term. The noise term computes the difference of two independent and unbiased estimates of the full gradient $\nabla U(w_k)$, each using a batch of $b_2$ data points. Using the definition of $\zeta$ in~\eqref{e:def_zeta}, we can verify that the update~\eqref{e:sgd-noisy} is equivalent to
\begin{align*}\numberthis\label{e:sgd-noisy-2}
    & w_{k+1} = w_k - s\nabla U(w_k) + s\zeta(w_k, \eta_k) \\
    &\qquad + \sigma \sqrt{s}(\zeta(w_k, \eta_k'') - \zeta(w_k, \eta_k')),
\end{align*}
which is in the form of \eqref{e:discrete-clt}, with
\begin{align*}
\xi(w, \tilde{\eta}) = \sqrt{s} \zeta(w, \eta) + \sigma \lrp{\zeta(w, \eta'')-\zeta(w, \eta')},
\numberthis
\label{e:sgd-noisy-xi}
\end{align*}
where $\tilde{\eta} = (\eta, \eta', \eta'')$, and $|\eta| = b_1$, $|\eta'|=|\eta''|=b_2$. Further, the noise covariance matrix is
\begin{equation}\label{e:noise_cov_large_noise_sgd}
    \E{\xi(w, \tilde{\eta}) \xi(w, \tilde{\eta})^T} = (\frac{s}{b_1} + \frac{2\sigma^2}{b_2}) H(w).
\end{equation}
Therefore, if we have
\begin{equation}\label{e:var_matching}
    \frac{s}{b_1} + \frac{2\sigma^2}{b_2} = \frac{\delta}{b},
\end{equation}
then an $(s, \sigma, b_1, b_2)$-large-noise SGD should have the same noise covariance as a $(\delta, b)$-SGD \textit{(but very different higher noise moments due to the injected noise)}, and based on our theory, the large-noise SGD should have similar test error to that of the SGD algorithm, even if the step size and batch size are different. In Section~\ref{subsection:experiments}, we verify this experimentally. 
We stress that we are not proposing the large-noise SGD as a practical algorithm. The reason that this algorithm is interesting is that it gives us a family of $\lrp{w_k}_{k=1,2,\ldots}$ which converges to \eqref{e:sgd-limit}, and is implementable in practice. Thus this algorithm helps us uncover the importance of noise covariance (and the unimportance of higher noise moments) in Langevin MCMC-like algorithms. We also remark that \citet{hoffer2017train} proposed a different way of injecting noise, multiplying the sampled gradient with a suitably scaled Gaussian noise.

\begin{subsection}{Satisfying the Assumptions}\label{ss:example_ass}

Before presenting the experimental results, we remark on a particular way that a function $U(w)$ defined in \eqref{e:example_U}, along with the stochastic sequence $w_k$ defined in \eqref{e:sgd-noisy}, can satisfy the assumptions in Section \ref{ss:ass}.

Suppose first that we shift the coordinate system so that $\nabla U(0)=0$. Let us additionally assume that for each $i$, $U_i(w)$ has the form
\begin{align*}
U_i(w) = U_i'(w) + V(w),
\end{align*}
where $V(w):= m \lrp{\|x\|_2-R/2}^2$ is a $m$-strongly convex regularizer outside a ball of radius $R$, and each $U_i'(w)$ has $\LR$-Lipschitz gradients. Suppose further that $m \geq 4 \cdot \LR$. These additional assumptions make sense when we are only interested in $U(w)$ over $B_R(0)$, so $V(w)$ plays the role of a barrier function that keeps us within $B_R(0)$. Then, it can immediately be verified that $U(w)$ satisfies Assumption \ref{ass:U_properties} with $L=m+\LR$.

The noise term $\xi$ in \eqref{e:sgd-noisy-xi} satisfies Assumption \ref{ass:xi_properties}.1 by definition, and satisfies Assumption \ref{ass:xi_properties}.3 with $L_\xi = \lrp{\sqrt{s} + 2\sigma} L$. Assumption \ref{ass:xi_properties}.2 is satisfied if $\zeta(w,\eta)$ is bounded for all $w$, i.e. the sampled gradient does not deviate from the true gradient by more than a constant. We will need to assume directly Assumption \ref{ass:xi_properties}.4, as it is a property of the distribution of $\nabla U_i(w)$ for $i=1, \ldots, n$. 

\end{subsection}

\begin{figure}[t]
\centering
\includegraphics[width=1\linewidth]{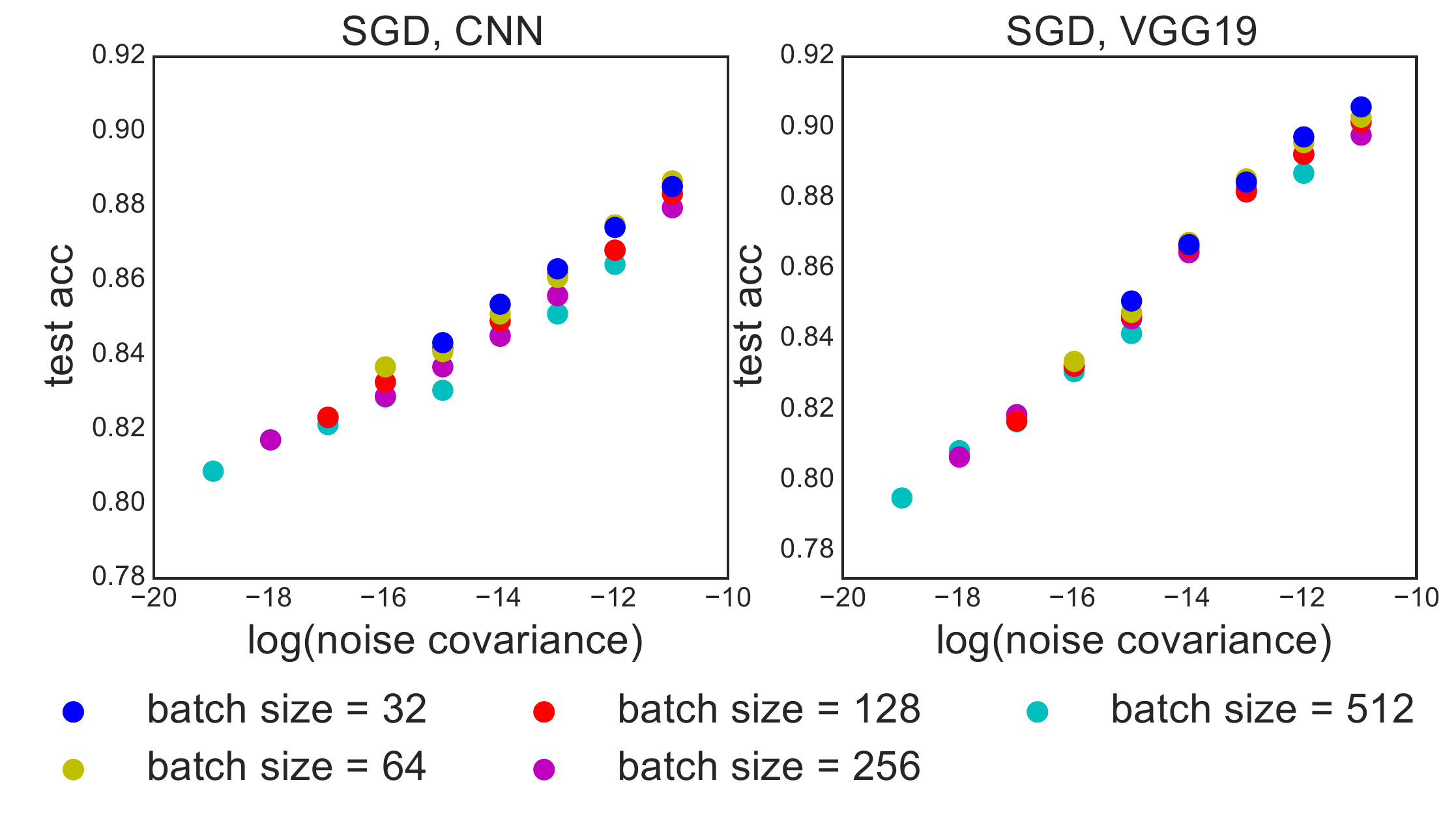}
\caption{Relationship between test accuracy and the noise covariance of SGD algorithm. In each plot, the dots with the same color correspond to SGD runs with the same batch size but different step sizes.}
\label{fig:const_lr_acc_vs_var}
\end{figure}

\begin{subsection}{Experiments}\label{subsection:experiments}

\begin{figure*}[t]
\centering
\includegraphics[width=1.0\linewidth]{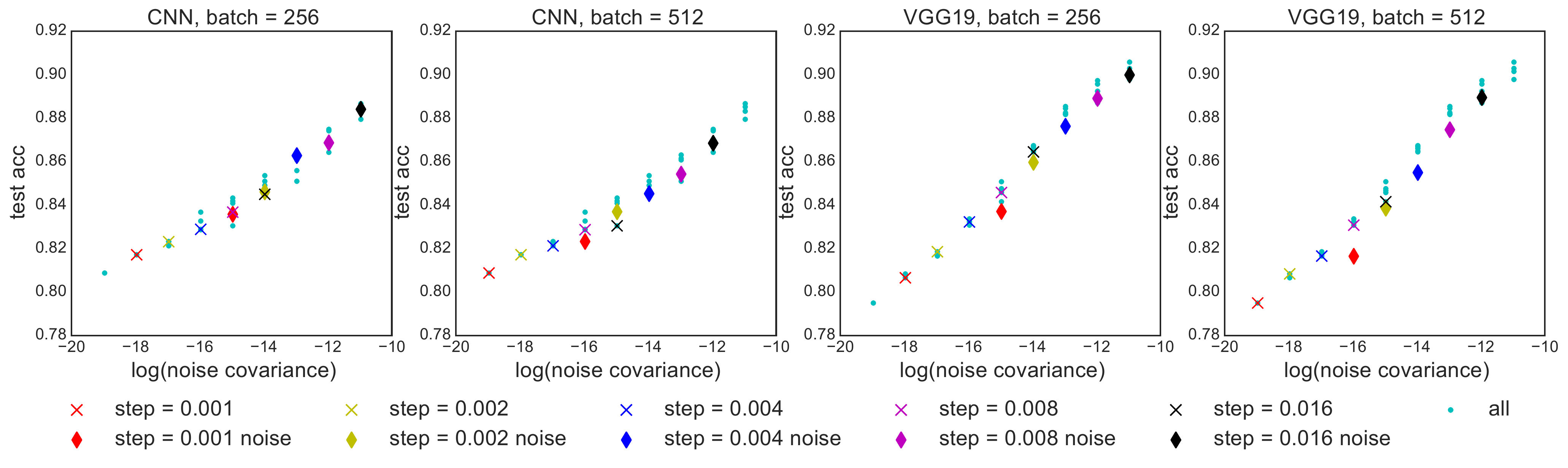}
\caption{
Large-noise SGD. Small dots correspond to all the baseline SGD runs in Figure~\ref{fig:const_lr_acc_vs_var}. Each $\times$ corresponds to a baseline SGD run whose step size is specified in the legend and batch size is specified in the title. Each $\diamond$ corresponds to a large-noise SGD run whose noise covariance is $8$ times of that of the $\times$ with the same color. As we can see, injecting noise improves test accuracy, and the large-noise SGD runs fall close to the linear trend.
}\label{fig:lr_matching}
\end{figure*}

\begin{figure*}[t]
\centering
\includegraphics[width=1.0\linewidth]{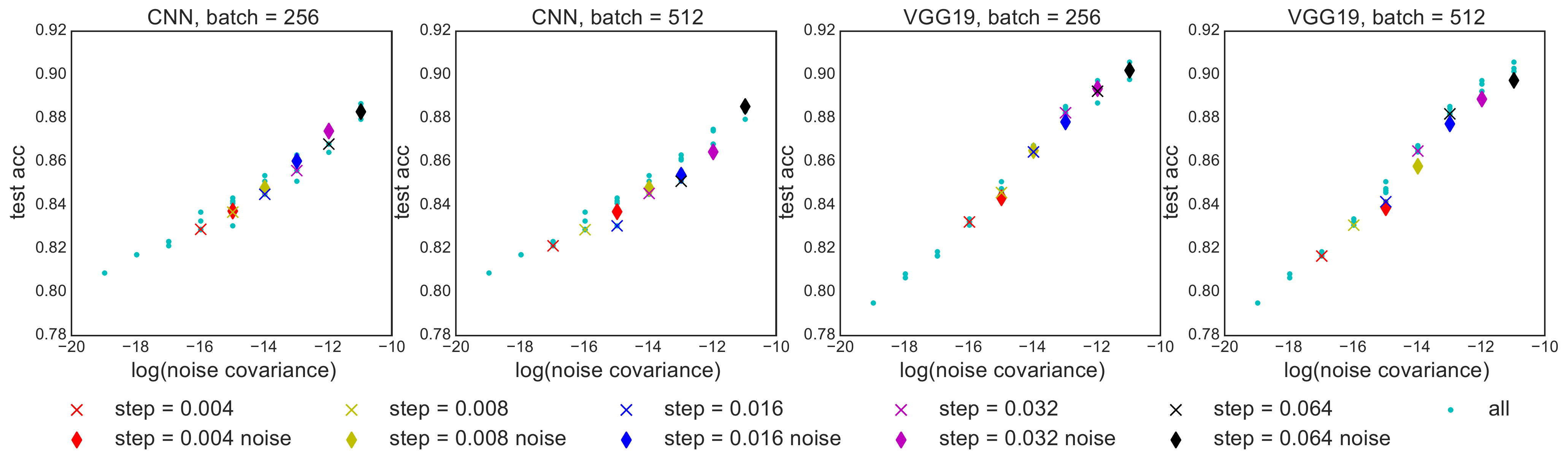}
\caption{Large-noise SGD. Batch size in the titles represents the batch size of $\times$ runs. Each $\diamond$ corresponds to a large-noise SGD run whose noise covariance matches that of a baseline SGD run whose step size is the same as the $\times$ run with the same color and batch size is $128$. Again, large-noise SGD falls close to the linear trend.}\label{fig:batch_matching}
\end{figure*}

In this section, we present experimental results that validate the importance of noise covariance in predicting the test error of Langevin MCMC-like algorithms. Our experiment code can be found at \url{https://github.com/dongyin92/noise_covariance}.

In all experiments, we use two different neural network architectures on the CIFAR-10 dataset~\cite{krizhevsky2009learning} with the standard test-train split. The first architecture is a simple convolutional neural network, which we call CNN in the following,\footnote{We provide details of this CNN architecture in Appendix~\ref{apx:experiments}.} and the other is the VGG19 network~\cite{simonyan2014very}. To make our experiments consistent with the setting of SGD, we do not use batch normalization or dropout, and use constant step size. In all of our experiments, we run SGD algorithm $2000$ epochs such that the algorithm converges sufficiently. Since in most of our experiments, the accuracies on the training dataset are almost $100\%$, we use the test accuracy to measure the generalization performance.

Recall that according to~\eqref{e:noise_cov_sgd} and~\eqref{e:noise_cov_large_noise_sgd}, for both SGD and large-noise SGD, the noise covariance is a scalar multiple of $H(w)$. For simplicity, in the following, we will slightly abuse our terminology and call this scalar the \emph{noise covariance}; more specifically, for $(\delta, b)$-SGD, the noise covariance is $\delta / b$, and for an $(s, \sigma, b_1, b_2)$-large-noise SGD, the noise covariance is $\frac{s}{b_1} + \frac{2\sigma^2}{b_2}$.

\begin{subsubsection}{Accuracy vs Noise Covariance}
\label{ss:acc_rel_var}

In our first experiment, we focus on the SGD algorithm, and show that there is a positive correlation between the noise covariance and the final test accuracy of the trained model. One major purpose of this experiment is to establish baselines for our experiments on large-noise SGD.

We choose constant step size $\delta$ from
\begin{equation*}
\{0.001, 0.002, 0.004, 0.008, 0.016, 0.032, 0.064, 0.128\}
\end{equation*}
and minibatch size $b$ from $\{32, 64, 128, 256, 512\}$.
For each (step size, batch size) pair, we plot its final test accuracy against its noise covariance in Figure~\ref{fig:const_lr_acc_vs_var}. From the plot, we can see that higher noise covariance leads to better final test accuracy, and there is a linear trend between the test accuracy and the logarithm. We also highlight the fact that conditioned on the noise covariance, the test accuracy is not significantly correlated with either the step size or the minibatch size. In other words, similar to the observations in prior work~\cite{jastrzkebski2017three, he2019control}, there is a strong correlation between relative variance of an SGD sequence and its test accuracy, regardless of the combination of minibatch size and step size.

\end{subsubsection}

\begin{subsubsection}{Large-Noise SGD}
\label{ss:inject}

In this section, we implement and examine the performance of the large-noise SGD algorithm proposed in \eqref{e:sgd-noisy}. We select a subset of SGD runs with relatively small noise covariance in the experiment in the previous section (we call them \emph{baseline SGD runs}), and implement large-noise SGD by injecting noise. Our goal is to see, for a particular noise covariance, whether large-noise SGD has test accuracy that is similar to SGD, \textit{in spite of significant differences in third-and-higher moments of the noise in large-noise SGD compared to standard SGD}.

Our first experiment is to add noise with the same minibatch size to the $(\delta, b)$ baseline SGD run such that the new noise covariance matches that of an $(8\delta, b)$-SGD (an SGD run with larger step size). In other words, we implement $(\delta, \sqrt{7\delta / 2}, b, b)$-large-noise SGD, whose noise covariance is $8$ times of that of the baseline. Our results are shown in Figure~\ref{fig:lr_matching}. Our second experiment is similar: we add noise with minibatch size $128$ to the $(\delta, b)$ baseline SGD run with $b \in \{256, 512\}$ such that the new noise covariance matches that of a $(\delta, 128)$-SGD (an SGD run with smaller batch size). More specifically, we implement $(\delta, \sqrt{\frac{1}{2} (1-\frac{128}{b})\delta }, b, 128)$-large-noise SGD runs. The results are shown in Figure~\ref{fig:batch_matching}. In these figures, each $\times$ denotes a baseline SGD run, with step size specified in the legend and minibatch size specified by plot title. For each baseline SGD run, we have a corresponding large-noise SGD run, denoted by $\diamond$ with the same color. As mentioned, these $\diamond$ runs are designed to match the noise covariance of SGD with larger step size or smaller batch size. In addition to $\times$ and $\diamond$, we also plot using a small teal marker all the other runs from Section \ref{ss:acc_rel_var}. This helps highlight the linear trend between the logarithm of noise covariance and test accuracy that we observed in Section \ref{ss:acc_rel_var}.

As can be seen, the (noise variance, test accuracy) values for the $\diamond$ runs fall close to the linear trend. More specifically, a run of large-noise SGD produces similar test accuracy to vanilla SGD runs with the same noise variance. We highlight two potential implications: First, just like in Section \ref{ss:acc_rel_var}, we observe that the test accuracy strongly correlates with relative variance, even for noise of the form \eqref{e:sgd-noisy-xi}, which can have rather different higher moments than $\zeta$ (standard SGD noise); Second, since the $\diamond$ points fall close to the linear trend, we hypothesize that the constant-noise limit SDE \eqref{e:sgd-limit} should also have similar test error. If true, then this implies that we only need to study the potential $U(x)$ and noise covariance $M(x)$ to explain the generalization properties of SGD.
\end{subsubsection}
\end{subsection}
\end{section}

\begin{section}{Acknowledgements}
We wish to acknowledge support by the Army Research Office (ARO) under contract W911NF-17-1-0304 
under the Multidisciplinary University Research Initiative (MURI).
\end{section}

\bibliography{references}
\bibliographystyle{icml2020}

\onecolumn
\appendix
\newpage
\section*{Appendix}
\begin{section}{Proofs for Convergence under Gaussian Noise (Theorem \ref{t:main_gaussian})}
    \begin{subsection}{Proof Overview}
    The main proof of Theorem \ref{t:main_gaussian} is contained in Appendix \ref{ss:proof:t:main_gaussian}. 
    
    Here, we outline the steps of our proof:
    \begin{enumerate}
    \item In Appendix \ref{ss:coupling_construction}, we construct a coupling between $\eqref{e:exact-sde}$ and $\eqref{e:discrete-langevin}$ over a single step (i.e. for $t\in[k\delta, (k+1)\delta]$, for some $k$ and $\delta$).
    \item Appendix \ref{ss:step_gaussian}, we prove Lemma \ref{l:gaussian_contraction}, which shows that under the coupling constructed in Step 1, a Lyapunov function $f(x_T - y_T)$ contracts exponentially with rate $\lambda$, plus a discretization error term. The function $f$ is defined in Appendix \ref{s:defining-q}, and sandwiches $\lrn{x_T - y_T}_2$. In Corollary \ref{c:main_gaussian:1}, we apply the results of Lemma \ref{l:gaussian_contraction} recursively over multiple steps to give a bound on $f(x_{k\delta}-y_{k\delta})$ for all $k$, and for sufficiently small $\delta$.
    \item Finally, in Appendix \ref{ss:proof:t:main_gaussian}, we prove Theorem \ref{t:main_gaussian} by applying the results of Corollary \ref{c:main_gaussian:1}, together with the fact that $f(z)$ upper bounds $\lrn{z}_2$ up to a constant factor.
    \end{enumerate}
    \end{subsection}
    
    \begin{subsection}{A coupling construction}
        \label{ss:coupling_construction}
        In this subsection, we will study the evolution of \eqref{e:exact-sde} and \eqref{e:discrete-langevin} over a small time interval. Specifically, we will study
        \begin{align*}
        \numberthis
        \label{e:marginal_x}
        d x_t =& - \nabla U(x_t) dt + M(x_t) dB_t\\
        \numberthis
        \label{e:marginal_y}
        d y_t =& - \nabla U(y_0) dt + M(y_0) dB_t
        \end{align*}
        One can verify that \eqref{e:marginal_x} is equivalent to \eqref{e:exact-sde}, and \eqref{e:marginal_y} is equivalent to a single step of \eqref{e:discrete-langevin} (i.e. over an interval $t\leq \delta$).

        We first give the explicit coupling between \eqref{e:marginal_x} and \eqref{e:marginal_y}: ( A similar coupling in the continuous-time setting is first seen in \cite{gorham2016measuring} in their proof of contraction of \eqref{e:exact-sde}.)

        Given arbirary $(x_0,y_0)$, define $(x_t,y_t)$ using the following coupled SDE:
        \begin{align*}
        \numberthis \label{e:coupled_2_processes}
        x_t =& x_0 + \int_0^t -\nabla U(x_s) ds + \int_0^t \cm dV_s + \int_0^t N(x_s) dW_s\\
        y_t =& y_0 + \int_0^t -\nabla U(y_0) dt + \int_0^t \cm  \lrp{I - 2\gamma_s \gamma_s^T} dV_s + \int_0^t N(y_0) dW_s
        \end{align*}

        Where $dV_t$ and $dW_t$ are two independent standard Brownian motion, and
        \begin{align*}
        \gamma_t := \frac{x_t-y_t}{\|x_t-y\|_2} \cdot \ind{\|x_t-y_t\|_2 \in [2\epsilon, \Rq)}
        \numberthis \label{d:gammat}
        \end{align*}

        By Lemma \ref{l:marginal_of_coupling}, we show that \eqref{e:marginal_x} has the same distribution as $x_t$ in \eqref{e:coupled_2_processes}, and \eqref{e:marginal_y} has the same distribution as $y_t$ in \eqref{e:coupled_2_processes}. Thus, for any $t$, the process $(x_t,y_t)$ defined by \eqref{e:coupled_2_processes} is a valid coupling for \eqref{e:marginal_x} and \eqref{e:marginal_y}.
    \end{subsection}

    \begin{subsection}{One step contraction}
    \label{ss:step_gaussian}
    \begin{lemma}
        \label{l:gaussian_contraction}
        Let $f$ be as defined in Lemma \ref{l:fproperties} with parameters $\epsilon$ satisfying $\epsilon \leq \frac{\Rq}{\aq\Rq^2 + 1}$. Let $x_t$ and $y_t$ be as defined in \eqref{e:coupled_2_processes}.
        If we assume that $\E{\lrn{y_0}_2^2} \leq 8\lrp{R^2 + \beta^2/m}$ and $T\leq \min\lrbb{\frac{\epsilon^2}{\beta^2}, \frac{\epsilon}{6 L\sqrt{R^2 + \beta^2/m}}}$, then
        \begin{align*}
        \E{f(x_T - y_T)}
        \leq e^{-\lambda T} \E{f(x_0 - y_0)} + 3T (L+\LN^2) \epsilon
        \end{align*}
    \end{lemma}

    \begin{remark}
        For ease of reference: $m,L,\LR, R$ are from Assumption \ref{ass:U_properties}, $\cm, \beta$ are from Assumption \ref{ass:xi_properties}, $\aq, \Rq, \LN, \lambda$ are defined in \eqref{d:constants}.
    \end{remark}

    \begin{proof}[Proof of Lemma \ref{l:gaussian_contraction}]
        For notational convenience, for the rest of this proof, let us define $z_t := x_t - y_t$ and $\nabla_t := \nabla U(x_t) - \nabla U(y_t)$, $\Delta_t := \nabla U(y_0) - \nabla U(y_t)$ $N_t := N(x_t) - N(y_t)$.

        It follows from \eqref{e:coupled_2_processes} that
        \begin{align*}
        d z_t = - \nabla_t dt + \Delta_t dt + 2 \cm \gamma_t \gamma_t^T dV_t + \lrp{N_t + N(y_t) - N(y_0)} dW_t
        \numberthis \label{e:coupled_difference_sde}
        \end{align*}

        Using Ito's Lemma, the dynamics of $f(z_t)$ is given by
        \begin{align*}
        &d f(z_t)\\
        =& {\lin{\nabla f(z_t), dz_t}}
        + {2\cm^2\tr\lrp{\nabla^2 f(z_t) \lrp{\gamma_t \gamma_t^T}}} dt
        + {\frac{1}{2}\tr\lrp{\nabla^2 f(z_t) \lrp{N_t+ N(y_t) - N(y_0)}^2}} dt\\
        =& \underbrace{-\lin{\nabla f(z_t), \nabla_t}}_{\circled{1}} dt + \underbrace{\lin{\nabla f(z_t), \Delta_t}}_{\circled{2}} dt  +  \underbrace{\lin{\nabla f(z_t), 2 \cm \gamma_t \gamma_t^T dV_t + \lrp{N_t + N(y_t) - N(y_0)} dW_t }}_{\circled{3}}\\
        &\quad + \underbrace{2\cm^2\tr\lrp{\nabla^2 f(z_t) \lrp{\gamma_t \gamma_t^T}}}_{\circled{4}} dt
        + \underbrace{\frac{1}{2}\tr\lrp{\nabla^2 f(z_t) \lrp{N_t+ N(y_t) - N(y_0)}^2}}_{\circled{5}} dt
        \numberthis \label{e:t:asldka}
        \end{align*}

        $\circled{3}$ goes to $0$ when we take expectation, so we will focus on $\circled{1}, \circled{2}, \circled{4}, \circled{5}$. We will consider 3 cases

        \textbf{Case 1: $\|z_t\|_2 \leq 2\epsilon$}\\
        From item 1(c) of Lemma \ref{l:fproperties}, $\lrn{\nabla f(z)}_2 \leq 1$. Using Assumption \ref{ass:U_properties}.1, $\lrn{\nabla_t} \leq L\lrn{z_t}_2$, so that
        \begin{align*}
        \circled{1} \leq \lrn{\nabla_t}_2 \leq L \lrn{z_t}_2 \leq 2 L \epsilon
        \end{align*}

        Also by Cauchy Schwarz,
        \begin{align*}
        \circled{2}
        =& \lin{\nabla f(z_t), \Delta_t}
        \leq \lrn{\Delta_t}_2 \leq L \lrn{y_t - y_0}_2
        \end{align*}

        Since $\gamma_t = 0$ in this case by definition in \eqref{d:gammat}, $\circled{4}=0$.

        Using Lemma \ref{l:fproperties}.2.c. $ \lrn{\nabla^2 f(z_t)}_2 \leq \frac{2}{\epsilon}$, so that
        \begin{align*}
        \circled{5}
        \leq& \frac{1}{\epsilon} \lrp{\tr\lrp{N_t^2 + N(y_t) - N(y_0)}^2}\\
        \leq& \frac{2}{\epsilon} \lrp{\tr\lrp{N_t^2} + \tr\lrp{\lrp{N(y_t) - N(y_0)}^2}}\\
        \leq& \frac{2L_N^2}{\epsilon} \lrp{\lrn{z_t}_2^2 + \lrn{y_t - y_0}_2^2}\\
        \leq& 4L_N^2 \epsilon + \frac{2L_N^2}{\epsilon} \|y_t-y_0\|_2^2
        \end{align*}

        Where the second inequality is by Young's inequality, the third inequality is by item 2 of Lemma \ref{l:N_is_regular}, the fourth inequality is by our assumption that $\lrn{z_t}_2 \leq 2\epsilon$.

        Summing these,
        \begin{align*}
        \circled{1} + \circled{2} + \circled{4} + \circled{5}
        \leq 4 \lrp{L + L_N^2}\epsilon + L \lrn{y_t - y_0}_2 + \frac{2L_N^2}{\epsilon}\|y_t-y_0\|_2^2
        \end{align*}

        \textbf{Case 2: $\|z_t\|_2\in (2\epsilon, \Rq)$}\\
        In this case, $\gamma_t = \frac{z_t}{\|z_t\|_2}$. Let $q$ be as defined in \eqref{d:f} and $g$ be as defined in Lemma \ref{l:gproperties}.
        By items 1(b) and 2(b) of Lemma \ref{l:fproperties} and items 1(b) and 2(b) of Lemma \ref{l:gproperties},
        \begin{align*}
        \nabla f(z_t)
        =& q'(g(z_t)) \nabla g(z_t) \\
        =& q'(g(z_t)) \frac{z_t}{\lrn{z_t}_2}\\
        \nabla^2 f(z_t)
        =& q''(g(z_t))\nabla g(z_t) \nabla g(z_t)^T + q'(g(z_t))\nabla^2 g(z_t)\\
        =& q''(g(z_t)) \frac{z_t z_t^T}{\|z_t\|_2^2} + q'(g(z_t)) \frac{1}{\|z_t\|_2} \lrp{I - \frac{z_tz_t^T}{\|z_t\|_2^2}}
        \end{align*}

        Once again, by Assumption \ref{ass:U_properties}.3,
        \begin{align*}
        \circled{1} \leq q'(g(z_t))\lrn{\nabla_t}_2 \leq q'(g(z_t)) \cdot \LR \cdot \|z_t\|_2 \leq& L \cdot q'(g(z_t)) g(z_t) + 2 L \epsilon
        \end{align*}
        Where the last inequality uses Lemma \ref{l:gproperties}.4.
        We can also verify that
        \begin{align*}
        \circled{2} \leq L \lrn{y_t - y_0}_2
        \end{align*}

        Using the expression for $\nabla^2 f(z_t)$,
        \begin{align*}
        \circled{4}
        = 2\cm^2 \tr\lrp{\nabla^2 f(z_t) \gamma_t \gamma_t^T} = 2\cm^2 \cdot q''(g(z_t))
        \end{align*}

        Finally,
        \begin{align*}
        \circled{5}
        =& \frac{1}{2}\tr\lrp{\nabla^2 f(z_t) \lrp{N_t+ N(y_t) - N(y_0)}^2}\\
        =& \frac{1}{2}\tr\lrp{\lrp{q''(g(z_t)) \frac{z_t z_t^T}{\|z_t\|_2^2} + q'(g(z_t)) \frac{1}{\|z_t\|_2} \lrp{I - \frac{z_tz_t^T}{\|z_t\|_2^2}}} \lrp{N_t+ N(y_t) - N(y_0)}^2}\\
        \leq& \frac{1}{2}\tr\lrp{\lrp{ q'(g(z_t)) \frac{1}{\|z_t\|_2} \lrp{I - \frac{z_tz_t^T}{\|z_t\|_2^2}}} \lrp{N_t+ N(y_t) - N(y_0)}^2}\\
        \leq& \frac{q'(g(z_t))}{\|z_t\|_2}\cdot \lrp{\tr\lrp{N_t^2} + \tr\lrp{\lrp{N(y_t) - N(y_0)}^2}}\\
        \leq& q'(g(z_t)) \cdot L_N^2 \|z_t\|_2 + \frac{L_N^2\|y_t-y_0\|_2^2}{2\epsilon}\\
        \leq& q'(g(z_t)) \cdot L_N^2 g(z_t) + \frac{L_N^2\|y_t-y_0\|_2^2}{2\epsilon} + 2 L_N^2 \epsilon
        \end{align*}

        The above uses multiples times the fact that  $0 \leq q' \leq 1$ and $q'' \leq 0$ (proven in items 3 and 4 of Lemma \ref{l:qproperties}). The second inequality is by Young's inequality, the third inequality is by item 2 of Lemma \ref{l:N_is_regular}, the fourth inequality uses item 4 of Lemma \ref{l:gproperties}.

        Summing these,
        \begin{align*}
        \circled{1} + \circled{2} + \circled{4} + \circled{5}
        \leq& \lrp{\LR + L_N^2}  q'(g(z_t)) g(z_t) + 2\cm^2 q''(g(z_t)) + \frac{L_N^2\|y_t-y_0\|_2^2}{2\epsilon} + 2 \lrp{L + \LN^2}\epsilon\\
        \leq&  - \frac{2\cm^2\exp\lrp{-\frac{7\aq\Rq^2}{3}}}{32\Rq^2} q(g(z_t)) + \frac{L_N^2\|y_t-y_0\|_2^2}{2\epsilon} + 2 \lrp{L + \LN^2}\epsilon\\
        \leq& - \lambda q(g(z_t)) + \frac{L_N^2\|y_t-y_0\|_2^2}{2\epsilon} + 2 (L+L_N^2) \epsilon\\
        =& - \lambda f(z_t) + \frac{L_N^2\|y_t-y_0\|_2^2}{2\epsilon} + 2 (L+L_N^2) \epsilon + L \lrn{y_t - y_0}_2
        \end{align*}

        Where the last inequality follows from Lemma \ref{l:qproperties}.\ref{f:contraction} and the definition of $\lambda$ in \eqref{d:constants}.

        \textbf{Case 3: $\|z_t\|_2 \geq \Rq$}\\
        In this case, $\gamma_t=0$. Similar to case 2,
        \begin{align*}
        \nabla f(z_t) = q'(g(z_t)) \frac{z_t}{\|z_t\|_2}
        \end{align*}

        Thus by Assumption \ref{ass:U_properties}.3,
        \begin{align*}
        \circled{1}
        =& \lin{q'(g(z_t)) \frac{z_t}{\|z_t\|_2}, - \nabla_t}\\
        \leq& -m q'(g(z_t)) \lrn{z_t}_2
        \end{align*}
        Where the inequality is by Assumption \ref{ass:U_properties}.3.

        For identical reasons as in Case 1, $\circled{2}\leq \LR \lrn{y_t-y_0}_2$, and $\circled{4} = 0$.
        Finally,
        \begin{align*}
        \circled{5}
        =& \frac{1}{2}\tr\lrp{\nabla^2 f(z_t) \lrp{N_t+ N(y_t) - N(y_0)}^2}\\
        =& \frac{1}{2}\tr\lrp{\lrp{q''(g(z_t)) \frac{z_t z_t^T}{\|z_t\|_2^2} + q'(g(z_t)) \frac{1}{\|z_t\|_2} \lrp{I - \frac{z_tz_t^T}{\|z_t\|_2^2}}} \lrp{N_t+ N(y_t) - N(y_0)}^2}\\
        \leq& \frac{1}{2}\tr\lrp{\lrp{ q'(g(z_t)) \frac{1}{\|z_t\|_2} \lrp{I - \frac{z_tz_t^T}{\|z_t\|_2^2}}} \lrp{N_t+ N(y_t) - N(y_0)}^2}\\
        \leq& \frac{q'(g(z_t))}{\|z_t\|_2}\cdot \lrp{\tr\lrp{N_t^2} + \tr\lrp{\lrp{N(y_t) - N(y_0)}^2}}
        \end{align*}

        Where the first inequality is because $q''\leq 0$ from item 4 of Lemma \ref{l:qproperties}, the second inequality is by Young's inequality. (These steps are identical to Case 2). Continuing from above, and using item 2 and 3 of Lemma \ref{l:N_is_regular},
        \begin{align*}
        \circled{5}
        \leq& q'(g(z_t)) \cdot \lrp{\frac{8\beta^2 \LN}{\cm} + \frac{\LN^2\|y_t-y_0\|_2^2}{\epsilon}}\\
        \leq& q'(g(z_t)) \cdot \lrp{\frac{m}{2} \|z_t\|_2}
        + q'(g(z_t)) \cdot \lrp{ \frac{\LN^2\|y_t-y_0\|_2^2}{\epsilon}}
        \end{align*}
        Where the second inequality is by our definition of $\Rq$ in the Lemma statement, which ensures that $\frac{8\beta^2 \LN}{\cm} \leq \frac{m}{2} \Rq \leq \frac{m}{2} \|z_t\|_2 $.

        Thus
        \begin{align*}
        & \circled{1} + \circled{2} + \circled{4} + \circled{5}\\
        \leq& -m q'(g(z_t)) \|z_t\|_2  + \LR \lrn{y_t - y_0}_2 + \frac{m}{2} q'(g(z_t)) \|z_t\|_2 + q'(g(z_t)) \cdot \lrp{ \frac{\LN^2\|y_t-y_0\|_2^2}{\epsilon}}\\
        \leq& -\frac{m}{2} q'(g(z_t)) \|z_t\|_2 + \frac{\LN^2}{\epsilon} \lrn{y_t-y_0}_2^2 + L \lrn{y_t - y_0}_2\\
        \leq& - \lambda f(z_t) + \frac{\LN^2}{\epsilon} \lrn{y_t-y_0}_2^2 + L \lrn{y_t - y_0}_2
        \end{align*}

        where the second inequality uses $q'\leq 1$ from item 3 of Lemma \ref{l:qproperties}, the third inequality uses our definition of $\lambda$ in \eqref{d:constants}.
        
        Combining the three cases, \eqref{e:t:asldka} can be upper bounded with probability 1:
        \begin{align*}
            &d f(z_t)
            \leq - \lambda f(z_t) + \frac{\LN^2}{\epsilon} \lrn{y_t-y_0}_2^2 + L \lrn{y_t - y_0}_2 + \lin{\nabla f(z_t), 2 \cm \gamma_t \gamma_t^T dV_t + \lrp{N_t + N(y_t) - N(y_0)} dW_t }
        \end{align*}
        
        To simplify notation, let us define $G_t \in \Re^{1\times 2d}$ as $G_t:= \lrb{\nabla f(z_t)^T 2\cm \gamma_t \gamma_t^T, \nabla f(z_t)^T \lrp{N_t  + N(y_t) - N(y_0)}}$, and let $A_t$ be a $2d$-dimensional Brownian motion from concatenating $A_t = \cvec{V_t}{W_t}$. Thus
        \begin{align*}
            d f(z_t) \leq -\lambda f(z_t) dt + \lrp{\frac{L_N^2}{\epsilon}\lrn{y_t - y_0}_2^2 + L \lrn{y_t - y_0}_2} + G_t dA_t.
        \end{align*}
        We will study the Lyapunov function
        \begin{align*}
            \mathcal{L}_t:= f(z_t) - \int_0^t e^{-\lambda(t-s)}\lrp{\frac{L_N^2}{\epsilon}\lrn{y_s - y_0}_2^2 + L \lrn{y_s - y_0}_2} ds - \int_0^t e^{-\lambda(t-s)} G_s dA_s.
        \end{align*}
        By taking derivatives, we see that
        \begin{align*}
            d\mathcal{L}_t 
            \leq& -\lambda f(z_t) dt + \lrp{\frac{L_N^2}{\epsilon}\lrn{y_t - y_0}_2^2 + L \lrn{y_t - y_0}_2}dt + G_t dA_t\\
            &\qquad + \lambda \lrp{\int_0^t e^{-\lambda(t-s)}\lrp{\frac{L_N^2}{\epsilon}\lrn{y_s - y_0}_2^2 + L \lrn{y_s - y_0}_2} ds} dt - \lrp{\frac{L_N^2}{\epsilon}\lrn{y_t - y_0}_2^2 + L \lrn{y_t - y_0}_2} dt\\
            &\qquad + \lambda \lrp{\int_0^t e^{-\lambda(t-s)} G_s dA_s} dt - G_t dA_t\\
            =& -\lambda \mathcal{L}_t dt
        \end{align*}
        
        We can then apply Gronwall's Lemma to $\mathcal{L}_t$, so that
        \begin{align*}
            \mathcal{L}_T \leq e^{-\lambda T} \mathcal{L}_0,
        \end{align*}
        which is equivalent to
        \begin{align*}
            f(z_T) - \int_0^T e^{-\lambda(T-s)}\lrp{\frac{L_N^2}{\epsilon}\lrn{y_s - y_0}_2^2 + L \lrn{y_s - y_0}_2} ds - \int_0^T e^{-\lambda(t-s)} G_s dA_s \leq e^{-\lambda T} f(z_0).
        \end{align*}
        Observe that $G_s$ is measurable wrt the natural filtration generated by $A_s$, so that $\int_0^T e^{-\lambda (T-s)} G_s dA_s$ is a martingale. Thus taking expectations,
        \begin{align*}
            \E{f(z_T)} \leq e^{-\lambda T} \E{f(z_0)} + \int_0^T \frac{L_N^2}{\epsilon} \E{\lrn{y_s - y_0}_2^2} + L \E{\lrn{y_s - y_0}_2} ds
        \end{align*}
        By Lemma \ref{l:divergence_yt}, $\E{\lrn{y_t - y_0}_2^2} \leq t^2 L^2 \E{\lrn{y_0}_2^2} + t\beta^2$, so that 
        \begin{align*}
        & \int_0^T \frac{L_N^2}{\epsilon} \E{\lrn{y_s - y_0}_2^2} ds \leq \frac{T^3 L_N^2 L^2}{\epsilon}\E{\lrn{y_0}_2^2} +  \frac{T^2L_N^2}{\epsilon}\beta^2\\
        & L \E{\lrn{y_s - y_0}_2} \leq T^2 L^2 \sqrt{\E{\lrn{y_0}_2^2}} + T^{3/2} L \beta
        \end{align*}
        Furthermore, using our assumption in the Lemma statement that $T \leq \min\lrbb{\frac{\epsilon^2}{\beta^2}, \frac{\epsilon}{{6L \sqrt{R^2 + \beta^2/m}}}}$ and $\E{\lrn{y_0}_2^2} \leq 8 \lrp{R^2 + \beta^2/m}$, we can verify that
        \begin{align*}
            & \int_0^T \frac{L_N^2}{\epsilon} \E{\lrn{y_s - y_0}_2^2} ds \leq \frac{1}{4} TL_N^2 \epsilon + TL_N^2 \epsilon\\
            & L \E{\lrn{y_s - y_0}_2} \leq \frac{1}{2} TL\epsilon + TL\epsilon
        \end{align*}
        Combining the above gives 
        \begin{align*}
            \E{f(z_T)} \leq e^{-\lambda T} \E{f(z_0)} + 3 T\lrp{L+L_N^2} \epsilon
        \end{align*}
    \end{proof}
    
    \begin{corollary}
        \label{c:main_gaussian:1}
        Let $f$ be as defined in Lemma \ref{l:fproperties} with parameter $\epsilon$ satisfying $\epsilon \leq \frac{\Rq}{\aq\Rq^2 + 1}$.

        Let $\delta\leq \min\lrbb{\frac{\epsilon^2}{\beta^2}, \frac{\epsilon}{8 L\sqrt{R^2 + \beta^2/m}}}$, and let $\bx_t$ and $\by_t$ have dynamics as defined in \eqref{e:exact-sde} and \eqref{e:discrete-langevin} respectively, and suppose that the initial conditions satisfy $\E{\lrn{\bx_0}_2^2}\leq R^2 + \beta^2/m$ and $\E{\lrn{\by_0}_2^2}\leq R^2 + \beta^2/m$. Then there exists a coupling between $\bx_t$ and $\by_t$ such that
        \begin{align*}
        \E{f(\bx_{i\delta} - \by_{i\delta})} \leq e^{-\lambda i\delta}  \E{f(\bx_{0} - \by_{0})} + \frac{6}{\lambda}  \lrp{L + \LN^2} \epsilon
        \end{align*}
    \end{corollary}

    \begin{proof}[Proof of Corollary \ref{c:main_gaussian:1}]
        From Lemma \ref{l:energy_x} and \ref{l:energy_y}, our initial conditions imply that for all $t$, $\E{\|\bx_t\|_2^2} \leq 6\lrp{R^2 + \frac{\beta^2}{m}}$ and $\E{\|\by_{k\delta}\|_2^2} \leq 8 \lrp{R^2 + \frac{\beta^2}{m}}$.

        Consider an arbitrary $k$, and for $t\in[k\delta,(k+1)\delta)$, define
        \begin{align*}
        x_t := \bx_{k\delta+t} \quad \text{and} \quad
        y_t := \by_{k\delta+t}
        \end{align*}
        Under this definition, $x_t$ and $y_t$ have dynamics described in \eqref{e:marginal_x} and \eqref{e:marginal_y}. Thus the coupling in \eqref{e:coupled_2_processes}, which describes a coupling between $x_t$ and $y_t$, equivalently describes a coupling between $\bx_t$ and $\by_t$ over $t\in[k\delta, (k+1)\delta)$.

        We now apply Lemma \ref{l:gaussian_contraction}. Given our assumed bound on $\delta$ and our proven bounds on $\E{\lrn{\bx_t}_2^2}$ and $\E{\lrn{\by_t}_2^2}$,
        \begin{align*}
        & \E{f(\bx_{(k+1)\delta} - \by_{(k+1)\delta})}\\
        =& \E{f(x_{\delta} - y_{\delta})}\\
        \leq& e^{-\lambda \delta} \E{f(x_0 - y_0)} + 6\delta (L+\LN^2) \epsilon\\
        =& e^{-\lambda \delta} \E{f(\bx_{k\delta} - \by_{k\delta})} + 6\delta (L+\LN^2) \epsilon
        \end{align*}
        Applying the above recursively gives, for any $i$
        \begin{align*}
        \E{f(\bx_{i\delta} - \by_{i\delta})} \leq e^{-\lambda i\delta}  \E{f(\bx_{0} - \by_{0})} + \frac{6}{\lambda}  \lrp{L + \LN^2} \epsilon
        \end{align*}
    \end{proof}
    
    \end{subsection}

    \begin{subsection}{Proof of Theorem \ref{t:main_gaussian}}\label{ss:proof:t:main_gaussian}
    
    For ease of reference, we re-state Theorem \ref{t:main_gaussian} below as Theorem \ref{t:main_gaussian:restated} below. We make a minor notational change: using the letters $\bx_t$ and $\by_t$ in Theorem \ref{t:main_gaussian:restated}, instead of the letters $x_t$ and $y_t$ in Theorem \ref{t:main_gaussian}. This is to avoid some notation conflicts in the proof.

    \begin{theorem} [Equivalent to Theorem \ref{t:main_gaussian}]
        \label{t:main_gaussian:restated}
        Let $\bx_t$ and $\by_t$ have dynamics as defined in \eqref{e:exact-sde} and \eqref{e:discrete-langevin} respectively, and suppose that the initial conditions satisfy $\E{\lrn{\bx_0}_2^2}\leq R^2 + \beta^2/m$ and $\E{\lrn{\by_0}_2^2}\leq R^2 + \beta^2/m$. Let $\hat{\epsilon}$ be a target accuracy satisfying $\hat{\epsilon} \leq  \lrp{\frac{16\lrp{L + \LN^2}}{\lambda}} \cdot \exp\lrp{7\aq\Rq/3} \cdot \frac{\Rq}{\aq\Rq^2 + 1}$. Let $\delta$ be a step size satisfying
        \begin{align*}
        \delta \leq \min\twocase{\frac{\lambda^2 \hat{\epsilon}^2}{512 \beta^2\lrp{L^2 + \LN^4}\exp\lrp{\frac{14\aq\Rq^2}{3}}}}{\frac{2\lambda \hat{\epsilon}}{(L^2+\LN^4)\exp\lrp{\frac{7\aq\Rq^2}{3}}\sqrt{R^2 + \beta^2/m}}}.
        \end{align*}

        If we assume that $\bx_0 = \by_0$, then there exists a coupling between $\bx_t$ and $\by_t$ such that for any $k$,
        \begin{align*}
        \E{\lrn{\bx_{k\delta} - \by_{k\delta}}_2} \leq \hat{\epsilon}
        \end{align*}

        Alternatively, if we assume $k \geq \frac{ 3\aq\Rq^2}{ \delta} \log \frac{R^2 + \beta^2/m}{\hat{\epsilon}}$, then
        \begin{align*}
        W_1\lrp{p^*, p^y_{k\delta}} \leq 2\hat{\epsilon}
        \end{align*}
        where $p^y_t := \Law(\by_t)$.
    \end{theorem}

    \begin{proof}[Proof of Theorem \ref{t:main_gaussian:restated}]

        Let $\epsilon := \frac{\lambda}{16 (L+\LN^2)} \exp\lrp{-\frac{7\aq\Rq^2}{3}} \hat{\epsilon}$. Let $f$ be defined as in Lemma \ref{l:fproperties} with the parameter $\epsilon$.

        \begin{align*}
        & \E{\lrn{\bx_{i\delta} - \by_{i\delta}}_2}\\
        \leq& 2\exp\lrp{\frac{7\aq\Rq^2}{3}}\E{f(\bx_{i\delta} - \by_{i\delta})} + 2\exp\lrp{\frac{7\aq\Rq^2}{3}}\epsilon\\
        \leq& 2\exp\lrp{\frac{7\aq\Rq^2}{3}}\lrp{e^{-\lambda i\delta}  \E{f(\bx_{0} - \by_{0})} + \frac{6}{\lambda}  \lrp{L + \LN^2} \epsilon} + 2\exp\lrp{\frac{7\aq\Rq^2}{3}}\epsilon\\
        \leq& 2\exp\lrp{\frac{7\aq\Rq^2}{3}}e^{-\lambda i\delta}  \E{f(\bx_{0} - \by_{0})} + \frac{16\lrp{L+\LN^2}}{\lambda}\exp\lrp{\frac{7\aq\Rq^2}{3}} \cdot \epsilon
        \numberthis \label{e:t:asdkja:1}\\
        =& 2\exp\lrp{\frac{7\aq\Rq^2}{3}}e^{-\lambda i\delta}  \E{f(\bx_{0} - \by_{0})} + \hat{\epsilon}
        \end{align*}
        where the first inequality is by item 4 of Lemma \ref{l:fproperties}, the second inequality is by Corollary \ref{c:main_gaussian:1} (notice that $\delta$ satisfies the requirement on $T$ in Theorem \ref{t:main_gaussian}, for the given $\epsilon$). The third inequality uses the fact that $1\leq L/m \leq \frac{\lrp{L+\LN^2}}{\lambda}$.

        The first claim follows from substituting $\bx_0 = \by_0$ into \eqref{e:t:asdkja:1}, so that the first term is $0$, and using the definition of $\epsilon$, so that the second term is $0$.

        For the second claim, let $\bx_0 \sim p^*$, the invariant distribution of \eqref{e:exact-sde}. From Lemma \ref{l:energy_x}, we know that $\bx_0$ satisfies the required initial conditions in this Lemma. Continuing from \eqref{e:t:asdkja:1},
        \begin{align*}
        & \E{\lrn{\bx_{i\delta} - \by_{i\delta}}_2}\\
        \leq& 2\exp\lrp{\frac{7\aq\Rq^2}{3}}\lrp{2e^{-\lambda i\delta}  \E{\lrn{\bx_0}_2^2 + \lrn{\by_0}_2^2} + \frac{6}{\lambda}  \lrp{L + \LN^2} \epsilon} + \epsilon\\
        \leq& 2\exp\lrp{\frac{7\aq\Rq^2}{3}}\lrp{2e^{-\lambda i\delta}  \lrp{R^2 + \beta^2/m}}  + \frac{16}{\lambda}  \exp\lrp{2\frac{7\aq\Rq^2}{3}}\lrp{L + \LN^2} \epsilon\\
        =& 4\exp\lrp{\frac{7\aq\Rq^2}{3}}\lrp{e^{-\lambda i\delta}  \lrp{R^2 + \beta^2/m}} + \hat{\epsilon}
        \end{align*}
        By our assumption that $i\geq \frac{1}{\delta} \cdot 3\aq\Rq^2 \log \frac{R^2 + \beta^2/m}{\hat{\epsilon}}$, the first term is also bounded by $\hat{\epsilon}$, and this proves our second claim.
    \end{proof}

    \end{subsection}

    \begin{subsection}{Simulating the SDE}
    \label{ss:simlutating_discrete_sde}
    One can verify that the SDE in \eqref{e:discrete-langevin} can be simulated (at discrete time intervals) as follows:
    \begin{align*}
    y_{(k+1) \delta} = y_{k\delta} - \delta \nabla U(y_{k\delta}) + \sqrt{\delta} M(y_{k\delta}) \theta_k
    \end{align*}
    Where $\theta_k \sim \N(0,I)$. This however requires access to $M(y_{k,\delta})$, which may be difficult to compute.

    If for any $y$, one is able to draw samples from some distribution $p_y$ such that
    \begin{enumerate}
    \item $\Ep{\xi \sim p_y}{\xi}=0$
    \item $\Ep{\xi\sim p_y}{\xi \xi^T}=M(y)$
    \item $\lrn{\xi}_2 \leq \beta$ almost surely, for some $\beta$.
    \end{enumerate}

    then one might sample a noise that is $\delta$ close to $M(y_{k\delta}) \theta_k$ through Theorem \ref{t:zhai}.

    Specifically, if one draws $n$ samples $\xi_1...\xi_n\overset{iid}{\sim} p_y$, and let $S_n := \frac{1}{\sqrt{n}}\sum_{i=1}^n \xi_i$, Theorem \ref{t:zhai} guarantees that \\$W_2 \lrp{S_n, M(y) \theta} \leq \frac{6\sqrt{d}\beta\sqrt{\log n}}{\sqrt{n}}$. We remark that the proof of Theorem \ref{t:main_gaussian} can be modified to accommodate for this sampling error. The number of samples needed to achieve $\epsilon$ accuracy will be on the order of $n \approxeq O(\delta \epsilon)^{-2} = O(\epsilon^{-6})$.
    \end{subsection}

\end{section}

\begin{section}{Proofs for Convergence under Non-Gaussian Noise (Theorem \ref{t:main_nongaussian})}
\label{s:nongaussianproof}
    \begin{subsection}{Proof Overview}
    The main proof of Theorem \ref{t:main_nongaussian} is contained in Appendix \ref{ss:proof:t:main_nongaussian}. 
    
    Here, we outline the steps of our proof:
    \begin{enumerate}
    \item In Appendix \ref{ss:4_coupling}, we construct a coupling between $\eqref{e:exact-sde}$ and $\eqref{e:discrete-clt}$ over an epoch which consists of an interval $[k\delta , (k+n)\delta)$ for some $k$. The coupling in \eqref{ss:4_coupling} consists of four processes $(x_t,y_t,v_t,w_t)$, where $y_t$ and $v_t$ are auxiliary processes used in defining the coupling. Notably, the process $(x_t,y_t)$ has the same distribution over the epoch as \eqref{e:coupled_2_processes}.
    \item In Appendix \ref{ss:epoch_nongaussian}, we prove Lemma \ref{l:non_gaussian_contraction_stationary} and Lemma \ref{l:non_gaussian_contraction_anisotropic}, which, combined with Lemma \ref{l:gaussian_contraction} from Appendix \ref{ss:step_gaussian}, show that under the coupling constructed in Step 1, a Lyapunov function $f(x_T - w_T)$ contracts exponentially with rate $\lambda$, plus a discretization error term. In Corollary \ref{c:main_nongaussian:1}, we apply the results of Lemma \ref{l:gaussian_contraction}, Lemma \ref{l:non_gaussian_contraction_stationary} and Lemma \ref{l:non_gaussian_contraction_anisotropic} recursively over multiple steps to give a bound on $f(x_{k\delta}-w_{k\delta})$ for all $k$, and for sufficiently small $\delta$.
    \item Finally, in Appendix \ref{ss:proof:t:main_nongaussian}, we prove Theorem \ref{t:main_nongaussian} by applying the results of Corollary \ref{c:main_nongaussian:1}, together with the fact that $f(z)$ upper bounds $\lrn{z}_2$ up to a constant.
    \end{enumerate}
    \end{subsection}

    \begin{subsection}{Constructing a Coupling}
        \label{ss:4_coupling}
        In this subsection, we construct a coupling between \eqref{e:discrete-clt} and \eqref{e:exact-sde}, given arbitrary initialization $(x_0,w_0)$. We will consider a finite time $T=n\delta$, which we will refer to as an \textit{epoch}.

        \begin{enumerate}
            \item Let $V_t$ and $W_t$ be two independent Brownian motion.
            
            \item Using $V_t$ and $W_t$, define
            \begin{align*}
            \numberthis \label{e:coupled_4_processes_x}
            x_t =& x_0 + \int_0^t -\nabla U(x_s) ds  + \int_0^t \cm dV_s + \int_0^t N(w_0) dW_s
            \end{align*}
            
            \item Using the same $V_t$ and $W_t$ in \eqref{e:coupled_4_processes_x}, we will define $y_t$ as
            \begin{align*}
            \numberthis \label{e:coupled_4_processes_y}
            y_t =& w_0 + \int_0^t -\nabla U(w_0) ds + \int_0^t \cm  \lrp{I - 2\gamma_s \gamma_s^t} dV_s + \int_0^T N(x_s) dW_s
            \end{align*}
            
            Where $\gamma_t := \frac{x_t - y_t}{\|x_t-y_t\|_2} \cdot \ind{\|x_t-y_t\|_2 \in [2\epsilon, \Rq)}$.
            The coupling $(x_t,y_t)$ defined in \eqref{e:coupled_4_processes_x} and \eqref{e:coupled_4_processes_y} is identical to the coupling in \eqref{e:coupled_2_processes} (with $y_0 = w_0$).

            \item We now define a process $v_{k\delta}$ for $k=0...n$:
            \begin{align*}
            \numberthis \label{e:coupled_4_processes_v}
            v_{k\delta} =& w_0 + \sum_{i=0}^{k-1} - \delta \nabla U(w_0) + {\sqrt{\delta}} \sum_{i=0}^{k-1} \xi(w_0,\eta_i)
            \end{align*}
            where marginally, the variables $\lrp{\eta_0...\eta_{n-1}}$ are drawn $i.i.d$ from the same distribution as in \eqref{e:discrete-clt}.
            
            Notice that $y_T - w_0 - T \nabla U(w_0) = \int_0^T \cm dB_t + \int_0^T N(w_0) dW_t$, so that $\Law(y_T - w_0 - T \nabla U(w_0)) = \N(0, T M(w_0)^2)$. Notice also that $v_T - w_0 - T\nabla U(w_0) = \sqrt{\delta} \sum_{i=0}^{n-1} \xi(w_0, \eta_i)$. By Corollary \ref{c:clt_sum},  $W_2(y_T - w_0 - T \nabla U(w_0), v_T - w_0 - T \nabla U(w_0)) = 6 \sqrt{d\delta}\beta \sqrt{\log n}$. Let the joint distribution between \eqref{e:coupled_4_processes_v} and \eqref{e:coupled_4_processes_y} be the one induced by the optimal coupling between $y_T - w_0 - T \nabla U(w_0)$ and $v_T - w_0 - T \nabla U(w_0)$, so that
            \begin{align*}
            & \sqrt{\E{\lrn{y_T - v_T}_2^2}} \\
            =& \sqrt{\E{\lrn{y_T - T \nabla U(w_0) - v_T + T \nabla U(w_0)}_2^2}} \\
            =& W_2(y_T - w_0 - T\nabla U(w_0), v_T - w_0 - T\nabla U(w_0)) \\
            \leq& 6 \sqrt{d\delta}\beta \sqrt{\log n}
            \numberthis \label{e:yt-vt}
            \end{align*}
            where the last inequality is by Corollary \ref{c:clt_sum}.

            \item Given the sequence $\lrp{\eta_0...\eta_{n-1}}$ from \eqref{e:coupled_4_processes_v}, we can define
            \begin{align*}
            \numberthis \label{e:coupled_4_processes_w}
            w_{k\delta} =& w_0 + \sum_{i=0}^{k-1} -\delta\nabla U(w_{i\delta}) + {\sqrt{\delta}} \sum_{i=0}^{k-1} \xi(w_{i\delta},\eta_i)
            \end{align*}
            specifically, $(w_0 ... w_{n\delta})$ in \eqref{e:coupled_4_processes_w} and $(v_0...v_{n\delta})$ in \eqref{e:coupled_4_processes_v} are coupled through the shared $(\eta_0...\eta_{n-1})$ variables.
        \end{enumerate}
        
        For convenience, we will let $v_t := v_{i \delta}$ and $w_t := w_{i\delta}$, where $i$ is the unique integer satisfying $t\in[i\delta, (i+1)\delta)$.
        
        We can verify that, marginally, the process $x_t$ in \eqref{e:coupled_4_processes_x} has the same distribution as \eqref{e:exact-sde}, using the proof as Lemma \ref{l:marginal_of_coupling}. It is also straightforward to verify that $w_{k\delta}$, as defined in \eqref{e:coupled_4_processes_w}, has the same marginal distribution as \eqref{e:discrete-clt}, due to the definition of $\eta_i$ in \eqref{e:coupled_4_processes_v}.
    \end{subsection}

    \begin{subsection}{One Epoch Contraction}
        \label{ss:epoch_nongaussian}
        In Lemma \ref{l:non_gaussian_contraction_stationary}, we prove a discretization error bound between $f(x_T - y_T)$ and $f(x_T - v_T)$, for the coupling defined in \eqref{e:coupled_4_processes_x}, \eqref{e:coupled_4_processes_y} and \eqref{e:coupled_4_processes_v}.
        
        In Lemma \ref{l:non_gaussian_contraction_anisotropic}, we prove a discretization error bound between $f(x_T - v_T)$ and $f(x_T - w_T)$, for the coupling defined in \eqref{e:coupled_4_processes_x}, \eqref{e:coupled_4_processes_v} and \eqref{e:coupled_4_processes_w}.

        \begin{lemma}
            \label{l:non_gaussian_contraction_stationary}
            Let $f$ be as defined in Lemma \ref{l:fproperties} with parameter $\epsilon$ satisfying $\epsilon \leq  \frac{\Rq}{\aq\Rq^2 + 1}$. Let $x_t$, $y_t$ and $v_t$ be as defined in \eqref{e:coupled_4_processes_x}, \eqref{e:coupled_4_processes_y}, \eqref{e:coupled_4_processes_v}. Let $n$ be any integer and $\delta$ be any step size, and let $T:= n\delta$.

            If $\E{\lrn{x_0}_2^2} \leq  8 \lrp{R^2 + \beta^2/m}$, $\E{\lrn{y_0}_2^2} \leq  8 \lrp{R^2 + \beta^2/m}$ and $T\leq \min\lrbb{\frac{1}{16L}, \frac{\beta^2}{8L^2\lrp{R^2 + \beta^2/m}}}$ and
            \begin{align*}
            \delta \leq \min\lrbb{\frac{T\epsilon^2L}{36 d\beta^2\log \lrp{ \frac{36 d\beta^2}{\epsilon^2L}}},  \frac{T\epsilon^4L^2} {2^{14} d\beta^4\log\lrp{\frac{2^{14} d\beta^4}{\epsilon^4L^2}}}}
            \end{align*}
            Then
            \begin{align*}
            & \E{f(x_T - v_T)} - \E{f(x_T - y_T)} \leq 4TL\epsilon
            \end{align*}

        \end{lemma}

        \begin{proof}
            By Taylor's Theorem,
            \begin{align*}
            & \E{f(x_T - v_T)}\\
            =& \E{f(x_T - y_T)+ \lin{\nabla f(x_T - y_T), y_T - v_T} + \int_0^1\int_0^s \lin{\nabla^2 f(x_T - y_T + s(y_T-v_T)), (y_T - v_T) (y_T - v_T)^T} ds dt}\\
            =& \E{f(x_T - y_T)+ \underbrace{\lin{\nabla f(x_0 - y_0), y_T - v_T}}_{\circled{1}} + \underbrace{\lin{\nabla f(x_T - y_T) - \nabla f(x_0 - y_0), y_T - v_T}}_{\circled{2}} }\\
            &\quad + \E{\underbrace{\int_0^1\int_0^s \lin{\nabla^2 f(x_T - y_T + s(y_T-v_T)), (y_T - v_T) (y_T - v_T)^T} ds dt}_{\circled{3}}}
            \end{align*}

            We will bound each of the terms above separately.
            \begin{align*}
            & \E{\circled{1}}\\
            =& \E{\lin{\nabla f(x_0 - y_0), y_T - v_T}}\\
            =& \E{\lin{\nabla f(x_0 - y_0), n \delta \nabla U(y_0) - n \delta \nabla U(v_0) + \int_0^T -\nabla U(w_0) dt + \int_0^T \cm dV_t + \int_0^T N(w_0) dW_t + \sum_{i=0}^{n-1} \sqrt{\delta} \xi(v_0,\eta_i)}}\\
            =& \E{\lin{\nabla f(x_0 - y_0), n \delta \nabla U(y_0) - n \delta \nabla U(v_0) }}\\
            =& 0
            \end{align*}
            where the third equality is because $\int_0^T dB_t$, $\int_0^T dW_t$ and $\sum_{k=1}^T \xi(v_0,\eta_i)$ have zero mean conditioned on the information at time $0$, and the fourth equality is because $y_0 = v_0$ by definition in \eqref{e:coupled_4_processes_y} and \eqref{e:coupled_4_processes_v}.
            \begin{align*}
            & \E{\circled{2}}\\
            =& \E{\lin{\nabla f(x_T - y_T) - \nabla f(x_0 - y_0), y_T - v_T}}\\
            \leq& \sqrt{\E{\lrn{\nabla f(x_T - y_T) - \nabla f(x_0 - y_0)}_2^2} }\sqrt{\E{\lrn{y_T - v_T}_2^2}}\\
            \leq& \frac{2}{\epsilon}\sqrt{2\E{\lrn{x_T - x_0}_2^2 + \lrn{y_T - y_0}_2^2}} \sqrt{\E{\lrn{y_T - v_T}_2^2}}\\
            \leq& \frac{2}{\epsilon}\sqrt{\lrp{32T\beta^2 + 4T\beta^2}}  \cdot \lrp{6 \sqrt{d\delta}\beta {\log n}}\\
            \leq& \frac{128}{\epsilon} \sqrt{T}\beta^2 \cdot \lrp{ \sqrt{d\delta} {\log n}}
            \end{align*}

            Where the second inequality is by $\lrn{\nabla^2 f}_2 \leq \frac{2}{\epsilon}$ from item 2(c) of Lemma \ref{l:fproperties} and Young's inequality. The third inequality is by Lemma \ref{l:divergence_xt} and Lemma \ref{l:divergence_yt} and \eqref{e:yt-vt}.

            Finally, we can bound
            \begin{align*}
            & \E{\circled{3}}\\
            \leq& \int_0^1\int_0^s \E{\lrn{\nabla^2 f(x_T - y_T + s(y_T-v_T))}_2 \lrn{y_T - v_T}_2^2} ds dt\\
            \leq& \frac{2}{\epsilon} \E{\lrn{y_T - v_T}_2^2}\\
            \leq& \frac{72 d \delta \beta^2 \log^2 n}{\epsilon}
            \end{align*}
            Where the second inequality is by $\lrn{\nabla^2 f}_2 \leq \frac{2}{\epsilon}$ from item 2(c) of Lemma \ref{l:fproperties}, the third inequality is by \eqref{e:yt-vt}.

            Summing these 3 terms,
            \begin{align*}
            &\E{f(x_T - v_T) - f(x_T - y_T)} \\
            \leq& \frac{128}{\epsilon} \sqrt{T}\beta^2 \cdot \lrp{ \sqrt{d\delta} \sqrt{\log n}} +  \frac{36 d \delta \beta^2 \log n}{\epsilon}\\
            =& \frac{128}{\epsilon} \sqrt{T}\beta^2 \cdot \lrp{ \sqrt{d\delta}\sqrt{\log \frac{T}{\delta}}} +  \frac{36 d \delta \beta^2 \log \frac{T}{\delta}}{\epsilon }
            \end{align*}

            Let us bound the first term. We apply Lemma \ref{l:xlogxbound} (with $x = \frac{T}{\delta}$ and $c = \frac{\epsilon^4}{2^{14} d\beta^4}$), which shows that
            \begin{align*}
            \frac{T}{\delta} \geq \frac{2^{14} d\beta^4}{\epsilon^4} \log\lrp{\frac{2^{14} d\beta^4}{\epsilon^4L^2}}
            \quad \Rightarrow \quad
            \frac{T}{\delta} \frac{1}{\log \frac{T}{\delta}} \geq \frac{2^{14} d\beta^4}{\epsilon^4L^2}
            \quad \Leftrightarrow \quad
            \frac{128}{\epsilon} \sqrt{T}\beta^2 \cdot \lrp{ \sqrt{d\delta}{\log \frac{T}{\delta}}} \leq TL\epsilon
            \end{align*}

            For the second term, we can again apply Lemma \ref{l:xlogxbound} ($x = \frac{T}{\delta}$ and $c = \frac{\epsilon^2L}{36 d\beta^2}$), which shows that
            \begin{align*}
            \frac{T}{\delta} \geq \frac{36 d\beta^2}{\epsilon^2L}  \log \lrp{ \frac{36 d\beta^2}{\epsilon^2L} }
            \quad \Rightarrow \quad
            \frac{T}{\delta} \frac{1}{\log \frac{T}{\delta}} \geq \frac{36 d\beta^2}{\epsilon^2L}
            \quad \Rightarrow \quad
            \frac{36 d \delta \beta^2 \log \frac{T}{\delta}}{\epsilon } \leq TL\epsilon
            \end{align*}
            The above imply that
            \begin{align*}
            \E{f(x_T - v_T) - f(x_T - y_T)}  \leq 2TL\epsilon
            \end{align*}

        \end{proof}

    \begin{lemma}
        \label{l:non_gaussian_contraction_anisotropic}
        Let $f$ be as defined in Lemma \ref{l:fproperties} with parameter $\epsilon$ satisfying $\epsilon\leq \frac{\Rq}{\aq\Rq^2 + 1}$. Let $x_t$, $v_t$ and $w_t$ be as defined in \eqref{e:coupled_4_processes_x}, \eqref{e:coupled_4_processes_v}, \eqref{e:coupled_4_processes_w}. Let $n$ be an integer and $\delta$ be a step size, and let $T:= n\delta$.

        If we assume that $\E{\lrn{x_0}_2^2}$, $\E{\lrn{v_0}_2^2}$, and $\E{\lrn{w_0}_2^2}$ are each upper bounded by $8 \lrp{R^2 + \beta^2/m}$ and that $T \leq \min\lrbb{ \frac{1}{16L},  \frac{\epsilon}{32\sqrt{L} \beta}, \frac{\epsilon^2}{128\beta^2}, \frac{\epsilon^4 \LN^2}{2^{14}\beta^2 \cm^2}}$, then
        \begin{align*}
        & \E{f(x_T - w_T)} - \E{f(x_T - v_T)} \leq 4T(L+\LN^2)\epsilon
        \end{align*}
    \end{lemma}
    \begin{remark}
        For sufficiently small $\epsilon$, our assumption on $T$ boils down to $T = o(\epsilon^4)$
    \end{remark}

    \begin{proof}
        First, we can verify using Taylor's theorem that for any $x,y$,
        \begin{align*}
        f(y) =& f(x) + \lin{\nabla f(x), y-x} + \int_0^1\int_0^s \lin{\nabla^2 f(x + s(y-x)), (y-x) (y-x)^T} ds dt
        \numberthis \label{e:taylor1}\\
        \nabla f(y) =& \nabla f(x) + \lin{\nabla^2 f(x), y-x} + \int_0^1\int_0^s \lin{\nabla^3 f(x + s(y-x)), (y-x) (y-x)^T} ds dt
        \numberthis \label{e:taylor2}
        \end{align*}

        Thus
        \begin{align*}
        & \E{f(x_T - w_T)}\\
        =& \E{f(x_T - v_T)+ \lin{\nabla f(x_T - v_T), v_T - w_T} + \int_0^1\int_0^s \lin{\nabla^2 f(x_T - v_T + s(v_T-w_T)), (v_T - w_T) (v_T - w_T)^T} ds dt}\\
        =& \E{f(x_T - v_T)+ \underbrace{\lin{\nabla f(x_0 - v_0), v_T - w_T}}_{\circled{1}} + \underbrace{\lin{\nabla f(x_T - v_T) - \nabla f(x_0 - v_0), v_T - w_T}}_{\circled{2}}}\\
        &\quad + \E{\underbrace{\int_0^1\int_0^s \lin{\nabla^2 f(x_T - v_T + s(v_T-w_T)), (v_T - w_T) (v_T - w_T)^T} ds dt}_{\circled{3}}}
        \end{align*}
        Recall from \eqref{e:coupled_4_processes_v} and \eqref{e:coupled_4_processes_w} that
        \begin{align*}
        v_{n\delta} =& w_0 + \sum_{i=0}^{n-1} \delta \nabla U(w_0) + {\sqrt{\delta}} \sum_{i=0}^{n-1} \xi(w_0,\eta_i)\\
        w_{n\delta} =& w_0 + \sum_{i=0}^{n-1} \delta \nabla U(w_{i\delta}) + {\sqrt{\delta}} \sum_{i=0}^{n-1} \xi(w_{i\delta},\eta_i)
        \end{align*}
        Note that conditioned on the randomness up to time $0$, $\E{\sum_{i=0}^{n-1} \xi(w_0,\eta_i)} = \E{\sum_{i=0}^{n-1} \xi(w_{i\delta},\eta_i)} = 0$, so that
        \begin{align*}
        & \E{\circled{1}}\\
        =& \E{\lin{\nabla f(x_0 - v_0), v_T - w_T}}\\
        =& \delta \E{\lin{\nabla f(x_0 - v_0), \sum_{i=0}^{n-1} \nabla U(w_{0}) - \nabla U(w_{i\delta})}} + \sqrt{\delta} \E{\lin{\nabla f(x_0 - v_0), \sum_{i=0}^{n-1} \xi(w_{0},\eta_i) - \sum_{i=0}^{n-1} \xi(w_{i\delta},\eta_i)}}\\
        =& \delta \E{\lin{\nabla f(x_0 - v_0), \sum_{i=0}^{n-1} \nabla U(w_{0}) - \nabla U(w_{i\delta})}}\\
        \leq& \delta \sum_{i=0}^{n-1} L \E{\lrn{w_0 - w_{i\delta}}_2}\\
        \leq& T L \sqrt{32 T\beta^2} \leq 8 T^{3/2} L \beta
        \end{align*}

        where the third equality is becayse $\xi(\cdot, \eta_i)$ has $0$ mean conditioned on the randomness at time $0$, and the second inequality is by Lemma \ref{l:divergence_wt}.

        Next,
        \begin{align*}
        & \E{\circled{2}}\\
        =& \E{\lin{\nabla f(x_T - v_T) - \nabla f(x_0 - v_0), v_T - w_T}}\\
        \leq& \E{\lrn{\nabla f(x_T - v_T) - \nabla f(x_0 - v_0)}_2 \lrn{v_T - w_T}}\\
        \leq& \frac{4}{\epsilon} \sqrt{\E{\lrn{x_T - x_0}_2^2 + \lrn{v_T - v_0}_2^2}} \cdot \sqrt{\E{\lrn{v_T - w_T}_2^2}}\\
        \leq& \frac{4}{\epsilon} \sqrt{16T\beta^2 + 2T\beta^2} \cdot \sqrt{32 \lrp{T^2 L^2 + TL_\xi^2} T\beta^2}\\
        \leq& \frac{128}{\epsilon} T \beta^2 \lrp{\sqrt{T} L_\xi + TL}
        \end{align*}
        where the second inequality is because $\lrn{\nabla^2 f}_2 \leq \frac{2}{\epsilon}$ from item 2(c) of Lemma \ref{l:fproperties} and by Young's inequality. The third inequality is by Lemma \ref{l:divergence_xt}, Lemma \ref{l:divergence_vt} and Lemma \ref{l:vt-wt}.

        Finally,
        \begin{align*}
        & \E{\circled{3}}\\
        =& \E{\int_0^1\int_0^s \lin{\nabla^2 f(x_T - v_T + s(v_T-w_T)), (v_T - w_T) (v_T - w_T)^T} ds dt}\\
        \leq& \int_0^1 \int_0^s \E{\lrn{\nabla^2 f(x_T - v_T + s(v_T-w_T))}_2\lrn{v_T - w_T}_2^2} ds \\
        \leq& \frac{1}{\epsilon} \E{\lrn{v_T - w_T}_2^2}\\
        \leq& \frac{32}{\epsilon} \lrp{T^2 L^2 + TL_\xi^2} T\beta^2
        \end{align*}
        wehere the second inequality is because $\lrn{\nabla^2 f}_2 \leq \frac{2}{\epsilon}$ from item 2(c) of Lemma \ref{l:fproperties} and by Young's inequality. The third inequality is by Lemma \ref{l:vt-wt}.

        Summing the above,
        \begin{align*}
        &\E{f(x_T - w_T) - f(x_T - v_T)} \\
        \leq& 8T^{3/2} L\beta + \frac{128}{\epsilon} T \beta^2 \lrp{\sqrt{T} L_\xi + TL} + \frac{32}{\epsilon} \lrp{T^2 L^2 + TL_\xi^2} T\beta^2\\
        \leq& T^{3/2} \epsilon
        \end{align*}
        where the last inequality is by our assumption on $T$, specifically,
        \begin{align*}
        &T \leq \frac{\epsilon^2}{128\beta^2}
        \Rightarrow T^{3/2} L\beta \leq TL\epsilon\\
        &T \leq \frac{\epsilon^2}{128\beta^2} \Rightarrow \frac{128}{\epsilon} T^2L \beta^2 \leq TL \epsilon\\
        &T \leq \frac{\epsilon}{32\sqrt{L} \beta} \Rightarrow \frac{32}{\epsilon}(T^3 L^2 \beta^2) \leq TL\epsilon\\
        & T \leq \frac{\epsilon^4 \LN^2}{2^{14}\beta^2 \cm^2} \Rightarrow \frac{128}{\epsilon} T^{3/2}\beta^2 L_\xi \leq T\LN^2 \epsilon\\
        & T \leq \frac{\epsilon^2}{128\beta^2} \Rightarrow T \leq \frac{\epsilon^2}{128\cm^2} \Rightarrow \frac{32}{\epsilon} T^2L_\xi^2\beta^2 \leq T\LN^2\epsilon
        \end{align*}
        where the last line uses the fact that $\beta \geq \cm^2$.

    \end{proof}
    
    \begin{corollary}
            \label{c:main_nongaussian:1}
            Let $f$ be as defined in Lemma \ref{l:fproperties} with parameter $\epsilon$ satisfying $\epsilon \leq \frac{\Rq}{\aq\Rq^2 + 1}$. \\
            Let $T= \min\lrbb{\frac{1}{16L}, \frac{\beta^2}{8L^2\lrp{R^2 + \beta^2/m}}, \frac{\epsilon}{32\sqrt{L} \beta}, \frac{\epsilon^2}{128\beta^2}, \frac{\epsilon^4 \LN^2}{2^{14}\beta^2 \cm^2}}$ and let $\delta \leq \min\lrbb{\frac{T\epsilon^2L}{36 d\beta^2\log \lrp{ \frac{36 d\beta^2}{\epsilon^2L}}},  \frac{T\epsilon^4L^2} {2^{14} d\beta^4\log\lrp{\frac{2^{14} d\beta^4}{\epsilon^4L^2}}}}$, assume additionally that $n=T/\delta$ is an integer.\\
            Let $\bx_t$ and $\bw_t$ have dynamics as defined in \eqref{e:exact-sde} and \eqref{e:discrete-langevin} respectively, and suppose that the initial conditions satisfy $\E{\lrn{\bx_0}_2^2}\leq R^2 + \beta^2/m$ and $\E{\lrn{\bw_0}_2^2}\leq R^2 + \beta^2/m$. Then there exists a coupling between $\bx_t$ and $\bw_t$ such that
            \begin{align*}
            \E{f(\bx_{i\delta} - \bw_{i\delta})} \leq e^{-\lambda i\delta}  \E{f(\bx_{0} - \bw_{0})} + \frac{6}{\lambda}  \lrp{L + \LN^2} \epsilon
            \end{align*}
        \end{corollary}

        \begin{proof}
            From Lemma \ref{l:energy_x} and \ref{l:energy_w}, our initial conditions imply that for all $t$, $\E{\|\bx_t\|_2^2} \leq 6\lrp{R^2 + \frac{\beta^2}{m}}$ and $\E{\|\bw_{k\delta}\|_2^2} \leq 8 \lrp{R^2 + \frac{\beta^2}{m}}$.

            Consider an arbitrary $k$, and for $t\in[0,T)$, define
            \begin{align*}
            x_t := \bx_{kT+t} \quad \text{and} \quad
            w_t := \bw_{kT+t} \numberthis \label{e:t:kasjnd}
            \end{align*}
            Notice that as described above, $x_t$ and $w_t$ have dynamics described in \eqref{e:exact-sde} and \eqref{e:discrete-clt}. Let $x_t,w_t$ have joint distribution as described in \eqref{e:coupled_4_processes_x} and \eqref{e:coupled_4_processes_w}, and let $(y_t,v_t)$ be the processes defined in \eqref{e:coupled_4_processes_y} and \eqref{e:coupled_4_processes_v}. Notice that the joint distribution between $x_t$ and $w_t$ equivalently describes a coupling between $\bx_t$ and $\bw_t$ over $t\in[kT, (k+1)T)$.

            First, notice that the processes $\eqref{e:coupled_4_processes_x}$ and $\eqref{e:coupled_4_processes_y}$ have the same distribution as \eqref{e:coupled_2_processes}. We can thus apply Lemma \ref{l:gaussian_contraction}:
            \begin{align*}
            \E{f(x_{T} - y_{T})}
            \leq& e^{-\lambda T} \E{f(x_0 - y_0)} + 6T (L+\LN^2) \epsilon
            \end{align*}
            
            By Lemma \ref{l:non_gaussian_contraction_stationary},
            \begin{align*}
            \E{f(x_T - v_T)} - \E{f(x_T - y_T)}
            \leq 4TL\epsilon
            \end{align*}
            By Lemma \ref{l:non_gaussian_contraction_anisotropic},
            \begin{align*}
            \E{f(x_T - w_T)} - \E{f(x_T - v_T)}
            \leq 4T(L+\LN^2)\epsilon
            \end{align*}

            Summing the above three equations,
            \begin{align*}
            \E{f(x_T - w_T)}
            \leq e^{-\lambda \delta} \E{f(x_0 - w_0)} + 14T (L+\LN^2)
            \end{align*}
            Where we use the fact that $y_0 = w_0$  by construction in \eqref{e:coupled_4_processes_y}.
            
            Recalling \eqref{e:t:kasjnd}, this is equivalent to
            \begin{align*}
            \E{f(\bx_{(k+1)T} - \bw_{(k+1)T})}
            \leq e^{-\lambda \delta} \E{f(\bx_{kT} - \bw_{kT})} + 14T (L+\LN^2)
            \end{align*}

            Applying the above recursively gives, for any $i$
            \begin{align*}
            \E{f(\bx_{iT} - \bw_{iT})} \leq e^{-\lambda iT}  \E{f(\bx_{0} - \bw_{0})} + \frac{14}{\lambda}  \lrp{L + \LN^2} \epsilon
            \end{align*}
        \end{proof}

    \end{subsection}

        \begin{subsection} {Proof of Theorem \ref{t:main_nongaussian}}
        \label{ss:proof:t:main_nongaussian}
    For ease of reference, we re-state Theorem \ref{t:main_nongaussian} below as Theorem \ref{t:main_nongaussian:restated} below. We make a minor notational change: using the letters $\bx_t$ and $\by_t$ in Theorem \ref{t:main_nongaussian:restated}, instead of the letters $x_t$ and $y_t$ in Theorem \ref{t:main_nongaussian}. This is to avoid some notation conflicts in the proof.
    
    \begin{theorem} [Equivalent to Theorem \ref{t:main_nongaussian}]
    \label{t:main_nongaussian:restated}
            Let $\bx_t$ and $w_t$ have dynamics as defined in \eqref{e:exact-sde} and \eqref{e:discrete-clt} respectively, and suppose that the initial conditions satisfy $\E{\lrn{\bx_0}_2^2}\leq R^2 + \beta^2/m$ and $\E{\lrn{\bw_0}_2^2}\leq R^2 + \beta^2/m$. Let $\hat{\epsilon}$ be a target accuracy satisfying $\hat{\epsilon} \leq  \lrp{\frac{16\lrp{L + \LN^2}}{\lambda}} \cdot \exp\lrp{7\aq\Rq/3} \cdot \frac{\Rq}{\aq\Rq^2 + 1}$. Let $\epsilon:= \frac{\lambda}{16 (L+\LN^2)} \exp\lrp{-\frac{7\aq\Rq^2}{3}} \hat{\epsilon}$. Let $T:= \min\lrbb{\frac{1}{16L}, \frac{\beta^2}{8L^2\lrp{R^2 + \beta^2/m}}, \frac{\epsilon}{32\sqrt{L} \beta}, \frac{\epsilon^2}{128\beta^2}, \frac{\epsilon^4 \LN^2}{2^{14}\beta^2 \cm^2}}$ and let $\delta$ be a step size satisfying
            \begin{align*}
            \delta \leq \min\lrbb{\frac{T\epsilon^2L}{36 d\beta^2\log \lrp{ \frac{36 d\beta^2}{\epsilon^2L}}},  \frac{T\epsilon^4L^2} {2^{14} d\beta^4\log\lrp{\frac{2^{14} d\beta^4}{\epsilon^4L^2}}}}.
            \end{align*}

            If we assume that $\bx_0 = \bw_0$, then there exists a coupling between $\bx_t$ and $\bw_t$ such that for any $k$,
            \begin{align*}
            \E{\lrn{\bx_{k\delta} - \bw_{k\delta}}_2} \leq \hat{\epsilon}.
            \end{align*}

            Alternatively, if we assume that
            $k \geq \frac{3\aq\Rq^2}{ \delta} \cdot  \log \frac{R^2 + \beta^2/m}{\hat{\epsilon}}$, then
            \begin{align*}
            W_1\lrp{p^*, p^w_{k\delta}} \leq 2\hat{\epsilon},
            \end{align*}
            where $p^w_t := \Law(\bw_t)$.
    \end{theorem}

       \begin{proof}[Proof of Theorem \ref{t:main_nongaussian:restated}]
            Let $f$ be defined as in Lemma \ref{l:fproperties} with parameter $\epsilon$.
            \begin{align*}
            & \E{\lrn{\bx_{i\delta} - \bw_{i\delta}}_2}\\
            \leq& 2\exp\lrp{\frac{7\aq\Rq^2}{3}}\E{f(\bx_{i\delta} - \bw_{i\delta})} + 2\exp\lrp{\frac{7\aq\Rq^2}{3}}\epsilon\\
            \leq& 2\exp\lrp{\frac{7\aq\Rq^2}{3}}\lrp{e^{-\lambda i\delta}  \E{f(\bx_{0} - \bw_{0})} + \frac{6}{\lambda}  \lrp{L + \LN^2} \epsilon} + 2\exp\lrp{\frac{7\aq\Rq^2}{3}}\epsilon\\
            \leq& 2\exp\lrp{\frac{7\aq\Rq^2}{3}}e^{-\lambda i\delta}  \E{f(\bx_{0} - \bw_{0})} + \frac{16\lrp{L+\LN^2}}{\lambda}\exp\lrp{\frac{7\aq\Rq^2}{3}} \cdot \epsilon
            \numberthis \label{e:t:asdkjas:1}\\
            =& 2\exp\lrp{\frac{7\aq\Rq^2}{3}}e^{-\lambda i\delta}  \E{f(\bx_{0} - \bw_{0})} + \hat{\epsilon}
            \end{align*}
            where the first inequality is by item 4 of Lemma \ref{l:fproperties}, the second inequality is by Corollary \ref{c:main_nongaussian:1} (notice that $\delta$ satisfies the requirement on $T$ in Theorem \ref{t:main_gaussian}, for the given $\epsilon$). The third inequality uses the fact that $1\leq L/m \leq \frac{\lrp{L+\LN^2}}{\lambda}$.

            The first claim follows from substituting $\bx_0 = \bw_0$ into \eqref{e:t:asdkjas:1}, so that the first term is $0$, and using the definition of $\epsilon$, so that the second term is $0$.

            For the second claim, let $\bx_0 \sim p^*$, the invariant distribution of \eqref{e:exact-sde}. From Lemma \ref{l:energy_x}, we know that $\bx_0$ satisfies the required initial conditions in this Lemma. Continuing from \eqref{e:t:asdkjas:1},
            \begin{align*}
            & \E{\lrn{\bx_{i\delta} - \bw_{i\delta}}_2}\\
            \leq& 2\exp\lrp{\frac{7\aq\Rq^2}{3}}\lrp{2e^{-\lambda i\delta}  \E{\lrn{\bx_0}_2^2 + \lrn{\bw_0}_2^2} + \frac{6}{\lambda}  \lrp{L + \LN^2} \epsilon} + \epsilon\\
            \leq& 2\exp\lrp{\frac{7\aq\Rq^2}{3}}\lrp{2e^{-\lambda i\delta}  \lrp{R^2 + \beta^2/m}}  + \frac{16}{\lambda}  \exp\lrp{2\frac{7\aq\Rq^2}{3}}\lrp{L + \LN^2} \epsilon\\
            =& 4\exp\lrp{\frac{7\aq\Rq^2}{3}}\lrp{e^{-\lambda i\delta}  \lrp{R^2 + \beta^2/m}} + \hat{\epsilon}
            \end{align*}
            By our assumption that $i\geq \frac{1}{ \delta} \cdot 3\aq\Rq^2 \log \frac{R^2 + \beta^2/m}{\hat{\epsilon}}$, the first term is also bounded by $\hat{\epsilon}$, and this proves our second claim.
        \end{proof}

    \end{subsection}

\end{section}

\begin{section}{Coupling Properties}
    \label{s:coupling_properties}
    \begin{lemma}
        \label{l:marginal_of_coupling}
        Consider the coupled $(x_t,y_t)$ in \eqref{e:coupled_2_processes}. Let $p_t$ denote the distribution of $x_t$, and $q_t$ denote the distribution of $y_t$. Let $p_t'$ and $q_t'$ denote the distributions of \eqref{e:marginal_x} and \eqref{e:marginal_y}.

        If $p_0 = p_0'$ and $q_0 = q_0'$, then $p_t = p_t'$ and $q_t=q_t'$ for all $t$.

    \end{lemma}
    
    \begin{proof}
    Consider the coupling in \eqref{e:coupled_2_processes}, reproduced below for ease of reference:
    \begin{align*}
        x_t =& x_0 + \int_0^t -\nabla U(x_s) ds + \int_0^t \cm dV_s + \int_0^t N(x_s) dW_s\\
        y_t =& y_0 + \int_0^t -\nabla U(y_0) dt + \int_0^t \cm  \lrp{I - 2\gamma_s \gamma_s^T} dV_s + \int_0^t N(y_0) dW_s
    \end{align*}
    
    Let us define the stochastic process $A_t := \int_0^t M(x_s)^{-1} \cm dV_s + \int_0^t M(x_s)^{-1} N(x_s) dW_s$. We can verify using Levy's characterization that $A_t$ is a standard Brownian motion: first, since $V_t$ and $W_t$ are Brownian motions, and $N(x)$ is differentiable with bounded derivatives, we know that $A_t$ has continuous sample paths. We now verify that $A_t^i A_t^j - \ind{i=j} t$ is a martingale.
    
	Notice that $d A_t = \cm dV_t + M(x_s)^{-1} N(x_s) dW_s$.
	Then
	\begin{align*}
	d A_t^i A_t^j
	=& d A_t^T \lrp{e_i e_j^T} A_t\\
	=& A_t \lrp{e_i e_j^T} \lrp{\cm dV_t + M(x_s)^{-1} N(x_s) dW_s}^T + \lrp{\cm dV_t + M(x_s)^{-1} N(x_s) dW_s} \lrp{e_j e_i^T} a_t^T\\
	&\qquad + \frac{1}{2} \tr\lrp{ \lrp{e_i e_j^T + e_j e_i^T}\lrp{c_m^2 M(x_s)^{-2} + M(x_s)^{-1} N(x_s)^2 M(x_s)^{-1}}} dt
	\end{align*}
	where the second inequality is by Ito's Lemma applied to $f(A_t) = A_t^T e_j e_j^T A_t$. Taking expectations,
	\begin{align*}
	d \E{A^i_t A^j_t} 
	=& \E{\frac{1}{2} \tr\lrp{ \lrp{e_i e_j^T + e_j e_i^T}\lrp{c_m^2 M(x_s)^{-2} + M(x_s)^{-1} N(x_s) N(x_s)^T \lrp{M(x_s)^{-1}}^T}}}dt\\
	=& \E{\frac{1}{2} \tr\lrp{ \lrp{e_i e_j^T + e_j e_i^T}\lrp{M(x_s)^{-1} \lrp{c_m^2 I + N(x_s)^2} M(x_s)^{-1}}}}dt\\
	=& \E{\frac{1}{2} \tr\lrp{ \lrp{e_i e_j^T + e_j e_i^T}\lrp{M(x_s)^{-1} \lrp{M(x_s)^2} M(x_s)^{-1}}}}dt\\
	=& \E{\frac{1}{2} \tr\lrp{ \lrp{e_i e_j^T + e_j e_i^T}}}dt\\
	=& \ind{i=j} dt
	\end{align*}
	
	This verifies that $A_t^i A_t^j - \ind{i=j} t$ is a martingale, and hence by Levy's characterization, $A_t$ is a standard Brownian motion. In turn, we verify that by definition of $A_t$,
	\begin{align*}
	x_t 
	=& x_0 + \int_0^t -\nabla U(x_s) ds + \int_0^t \cm dV_s + \int_0^t N(x_s) dW_s \\
	=& x_0 + \int_0^t -\nabla U(x_s) ds + \int_0^t M(x_s) \lrp{M(x_s)^{-1}\lrp{\cm dV_s + N(x_s) dW_s}}\\
	=& x_0 + \int_0^t -\nabla U(x_s) ds + \int_0^t M(x_s) dA_s
	\end{align*}
	Since we showed that $A_t$ is a standard Brownian motion, we verify that $x_t$ as defined in \eqref{e:coupled_2_processes} has the same distribution as \eqref{e:exact-sde}.
	
	On the other hand, we can verify that $A'_t := \int_0^T (I - 2\gamma_s \gamma_s^T) V_s$ is a standard Brownian motion by the reflection principle. Thus
	\begin{align*}
	\int_0^t \cm  \lrp{I - 2\gamma_s \gamma_s^T} dV_s + \int_0^t N(y_0) dW_s \sim \N(0, \lrp{c_m^2 I + N(y_0)^2}) = \N(0, M(y_0)^2)
	\end{align*}
	where the equality is by definition of $N$ in \eqref{d:N}.
	
	It follows immediately that $y_t$ in \eqref{e:coupled_2_processes} has the same distribution as $y_t$ in \eqref{e:discrete-langevin}.
    
    \end{proof}

    \begin{subsection}{Energy Bounds}
    \label{ss:energy_bounds}

    \begin{lemma}
        \label{l:energy_x}
        Consider $x_t$ as defined in \eqref{e:exact-sde}. If $x_0$ satisfies $\E{\|x_0\|_2^2} \leq R^2 + \frac{\beta^2}{m}$, then
        Then for all $t$,
        \begin{align*}
        \E{\|x_t\|_2^2} \leq 6\lrp{R^2 + \frac{\beta^2}{m}}\\
        \end{align*}
        We can also show that
        \begin{align*}
        \Ep{p^*}{\lrn{x}_2^2} \leq 4\lrp{R^2 + \frac{\beta^2}{m}}
        \end{align*}
    \end{lemma}
    \begin{proof}
        We consider the potential function $a(x) = \lrp{\|x\|_2 - R}_+^2$
        We verify that
        \begin{align*}
        \nabla a(x) =& (\|x\|_2 - R)_+ \frac{x}{\|x\|_2}\\
        \nabla^2 a(x) =& \ind{\|x\|_2 \geq R} \frac{xx^T}{\|x\|_2^2} +  \frac{(\|x\|_2 - R)_+}{\|x\|_2}\lrp{I - \frac{xx^T}{\|x\|_2^2}}
        \end{align*}
        Observe that
        \begin{enumerate}
            \item $\lrn{\nabla^2 a(x)}_2 \leq 2 \ind{\|x\|_2 \geq R} \leq 2$
            \item $\lin{\nabla a(x), - \nabla U(x)} \leq -m a(x)$. This can be verified by considering 2 cases. If $\|x\|_2 \leq R$, then $\nabla a(x) = 0$ and $a(x) = 0$. If $\|x\|_2 \geq R$, then by Assumption \ref{ass:U_properties},
            \begin{align*}
            \lin{\nabla a(x), -\nabla U(x)}
            \leq - m \lrp{\|x\|_2 - R}_+ \|w\|_2
            \leq - m \lrp{\|x\|_2 - R}_+^2 = -m \cdot a(x)
            \end{align*}
            \item $a(x) \geq \frac{1}{2}\|x\|_2^2 - 2R^2$. One can first verify that $a(x) \geq (\lrn{x}_2 - R)^2 - R^2$. Next, by Young's inequality,
            $(\lrn{x}_2 - R)^2 = \lrn{x}_2^2 + R^2 - 2\lrn{x}_2 R \geq \lrn{x}_2^2 + R^2 - \frac{1}{2} \lrn{x}_2^2 - 2R^2 = \frac{1}{2} \lrn{x}_2^2 - R^2$.
        \end{enumerate}

        Therefore,
        \begin{align*}
        & \ddt \E{a(x_t)}
        = \E{\lin{\nabla a(x_t), -\nabla U(x_t) dt}} + \frac{1}{2}\E{\tr\lrp{M(x_t)^2 \nabla^2 a(x)}}
        \leq -m\E{a(x_t)} + \beta^2\\
        \Rightarrow \qquad &
        \ddt \lrp{\E{a(x_t)} - \frac{\beta^2}{m}} \leq - m \lrp{\E{a(x_t)} - \frac{\beta^2}{m}}\\
        \Rightarrow \qquad &
        \ddt \lrp{\E{a(x_t)} - R^2 - \frac{\beta^2}{m}} \leq - m \lrp{\E{a(x_t)} - R^2 - \frac{\beta^2}{m}}
        \end{align*}

        Thus if $\E{\|x_0\|_2^2} \leq R^2 + \frac{\beta^2}{m}$, then  $\E{a(x_0)}\leq R^2 - \frac{\beta^2}{m}$, then $ \lrp{\E{a(x_0)} - R^2 - \frac{\beta^2}{m}} \leq 0$, and $\lrp{\E{a(x_t)} - R^2 + \frac{\beta^2}{m}} \leq e^{-mt} \cdot 0 \leq 0$ for all $t$. This implies that, for all $t$,
        \begin{align*}
        \E{\|x_t\|_2^2} \leq \E{2 a(x_t) + 4R^2} \leq 6\lrp{R^2 + \frac{\beta^2}{m}}
        \end{align*}

        For our second claim that $\Ep{p^*}{\lrn{x}_2^2} \leq R^2 + \frac{\beta^2}{m}$, we can use the fact that if $x_0 \sim p^*$, then $\E{a(x_t)}$ does not change as $p^*$ is invariant, so that
        \begin{align*}
        0 = \ddt \E{a(x_t)} \leq -m\E{a(x_t)} + \beta^2
        \end{align*}
        Thus
        \begin{align*}
        \E{a(x_t)} \leq \frac{\beta^2}{m}
        \end{align*}
        Again,
        \begin{align*}
        \Ep{p^*}{\lrn{x}_2^2} = \E{\lrn{x_t}_2^2} \leq 2 \E{a(x_t)} + 4 R^2 \leq 4 \lrp{R^2 + \frac{\beta^2}{m}}
        \end{align*}
    \end{proof}

    \begin{lemma}
        \label{l:energy_y}
        Let the sequence $y_{k\delta}$ be as defined in \eqref{e:discrete-clt}. Assuming that $\delta \leq m/(16L^2)$ and $\E{\|y_0\|_2^2} \leq 2 \lrp{R^2 + \frac{\beta^2}{m}}$
        Then for all $k$,
        \begin{align*}
        \E{\|y_{k\delta}\|_2^2} \leq 8 \lrp{R^2 + \frac{\beta^2}{m}}
        \end{align*}
    \end{lemma}
    \begin{proof}
        Let $a(w) := \lrp{\|w\|_2 - R}_+^2$. We can verify that
        \begin{align*}
        \nabla a(w) =& \lrp{\|w\|_2 - R}_+\frac{w}{\|w\|_2}\\
        \nabla^2 a(w) =& \ind{\|w\|_2 \geq R}\frac{ww^T}{\|w\|_2^2} + \lrp{\|w\|_2 - R}_+ \frac{1}{\|w\|_2} \lrp{I - \frac{ww^T}{\|w\|_2^2}}
        \end{align*}
        Observe that
        \begin{enumerate}
            \item $\lrn{\nabla^2 a(w)}_2 \leq 2 \ind{\|w\|_2 \geq R} \leq 2$
            \item $\lin{\nabla a(w), - \nabla U(w)} \leq -m a(w)$.
            \item $a(w) \geq \frac{1}{2}\|w\|_2^2 - 2R^2$.
        \end{enumerate}
        The proofs are identical to the proof at the start of Lemma \ref{l:energy_w}, so we omit them here.

        Using Taylor's Theorem, and taking expectation of $y_{(k+1)\delta}$ conditioned on $y_{k\delta}$,
        \begin{align*}
        &\E{a(y_{(k+1)\delta})}\\
        =& \E{a(y_{k\delta})} + \E{\lin{\nabla a(y_{k\delta}), y_{(k+1)\delta} - y_{k\delta}}}\\
        &\quad + \E{\int_0^1 \int_0^t \lin{\nabla^2 a(y_{k\delta} + s(y_{(k+1)\delta} - y_{k\delta}), (y_{(k+1)\delta} - y_{k\delta})(y_{(k+1)\delta} - y_{k\delta})^T} dt ds}\\
        \leq& \E{a(y_{k\delta})} + \E{\lin{\nabla a(y_{k\delta}), y_{(k+1)\delta} - y_{k\delta}}} + \E{\lrn{(y_{(k+1)\delta} - y_{k\delta})}_2^2 ds}\\
        \leq& \E{a(y_{k\delta})} + \E{\lin{\nabla a(y_{k\delta}), - \delta \nabla U(y_{k\delta})}} + 2\delta^2 \lrn{\nabla U(y_{k\delta})}_2^2 + 2\delta \E{\tr\lrp{M(y_{k\delta})^2}}\\
        \leq& \E{a(y_{k\delta})} -m\delta \E{a(y_{k\delta})} + 2\delta^2 \E{\lrn{\nabla U(y_{k\delta})}_2^2} + 2\delta \E{\tr\lrp{M(y_{k\delta})^2}}\\
        \leq& \E{a(y_{k\delta})} -m\delta \E{a(y_{k\delta})} + 2\delta^2L^2 \E{\lrn{y_{k\delta}}_2^2} + 2\delta \beta^2\\
        \leq& \E{a(y_{k\delta})} -m\delta \E{a(y_{k\delta})} + 4\delta^2L^2 \E{a(y_{k\delta})} + 8\delta^2 L^2 R^2 + 2\delta \beta^2\\
        \leq& (1-m\delta/2) \E{a(y_{k\delta})} + {m\delta} R^2 + 2\delta \beta^2
        \end{align*}
        Where the first inequality uses the upper bound on $\lrn{\nabla^2 a(y)}_2$ above, the second inequality uses the fact that $y_{(k+1)\delta} \sim \N\lrp{y_{k\delta} - \delta \nabla U(y_{k\delta}), \delta M(y_{k\delta})^2}$, the third inequality uses claim 2. at the start of this proof, the fourth inequality uses item 2 of Assumption \ref{ass:xi_properties}. The fifth inequality uses claim 3. above, the sixth inequality uses our assumption that $\delta \leq \frac{m}{16L^2}$.

        Taking expectation wrt $y_{k\delta}$,
        \begin{align*}
        & \E{a(y_{(k+1)\delta})} \leq \E{a(y_{k})}-m\delta \lrp{\E{a(y_{k\delta})} - 2R^2 + 2\beta^2/m}\\
        \Rightarrow \qquad &
        \E{a(y_{(k+1)\delta})} - (2R^2/2 + 2\beta^2/m) \leq (1-m\delta) \lrp{\E{a(y_{k\delta})} - (2R^2 + 2\beta^2/m}
        \end{align*}

        Thus, if $\E{\|y_0\|_2^2} \leq 2R^2 + 2\beta^2/m$, then $\E{a(y_0)} - \lrp{2R^2 + 2\beta^2/m} \leq 0 $, then $\E{a(y_{k\delta})}  - \lrp{2R^2 + 2\beta^2/m}\leq 0$ for all $k$, which implies that
        \begin{align*}
        \E{\lrn{y_{k\delta}}_2^2} \leq 2\E{a(y_{k\delta})} + 4 R^2 \leq  8\lrp{R^2 + \beta^2/m}
        \end{align*}
        for all $k$.
    \end{proof}

    \begin{lemma}
        \label{l:energy_w}
        Let the sequence $w_{k\delta}$ be as defined in \eqref{e:discrete-clt}. Assuming that $\delta \leq m/(16L^2)$ and $\E{\|w_0\|_2^2} \leq 2 \lrp{R^2 + \frac{\beta^2}{m}}$
        Then for all $k$,
        \begin{align*}
        \E{\|w_{k\delta}\|_2^2} \leq 8 \lrp{R^2 + \frac{\beta^2}{m}}
        \end{align*}
    \end{lemma}
    \begin{proof}
        The proof is almost identical to that of Lemma \ref{l:energy_y}.
        Let $a(w) := \lrp{\|w\|_2 - R}_+^2$. We can verify that
        \begin{align*}
        \nabla a(w) =& \lrp{\|w\|_2 - R}_+\frac{w}{\|w\|_2}\\
        \nabla^2 a(y) =& \ind{\|w\|_2 \geq R}\frac{ww^T}{\|w\|_2^2} + \lrp{\|w\|_2 - R}_+ \frac{1}{\|w\|_2} \lrp{I - \frac{ww^T}{\|w\|_2^2}}
        \end{align*}
        Observe that
        \begin{enumerate}
            \item $\lrn{\nabla^2 a(w)}_2 \leq 2 \ind{\|w\|_2 \geq R} \leq 2$
            \item $\lin{\nabla a(w), - \nabla U(w)} \leq -m a(w)$.
            \item $a(w) \geq \frac{1}{2}\|w\|_2^2 - 2R^2$.
        \end{enumerate}
        The proofs are identical to the proof at the start of Lemma \ref{l:energy_w}, so we omit them here.

        Using Taylor's Theorem, and taking expectation of $w_{(k+1)\delta}$ conditioned on $w_{k\delta}$,
        \begin{align*}
        &\E{a(w_{(k+1)\delta})}\\
        =& \E{a(w_{k\delta})} + \E{\lin{\nabla a(w_{k\delta}), w_{(k+1)\delta} - w_{k\delta}}}\\
        &\quad + \E{\int_0^1 \int_0^t \lin{\nabla^2 a(w_{k\delta} + s(w_{(k+1)\delta} - w_{k\delta}), (w_{(k+1)\delta} - w_{k\delta})(w_{(k+1)\delta} - w_{k\delta})^T} dt ds}\\
        \leq& \E{a(w_{k\delta})} + \E{\lin{\nabla a(w_{k\delta}), w_{(k+1)\delta} - w_{k\delta}}} + \E{\lrn{(w_{(k+1)\delta} - w_{k\delta})}_2^2 ds}\\
        \leq& \E{a(w_{k\delta})} + \E{\lin{\nabla a(w_{k\delta}), - \delta \nabla U(w_{k\delta})}} + 2\delta^2 \lrn{\nabla U(w_{k\delta})}_2^2 + 2\delta \E{\lrn{\xi(w_{k\delta},\eta_k)}_2^2}\\
        \leq& \E{a(w_{k\delta})} -m\delta \E{a(w_{k\delta})} + 2\delta^2 \E{\lrn{\nabla U(w_{k\delta})}_2^2} + 2\delta \E{\lrn{\xi(w_{k\delta},\eta_k)}_2^2}\\
        \leq& \E{a(w_{k\delta})} -m\delta \E{a(w_{k\delta})} + 2\delta^2L^2 \E{\lrn{w_{k\delta}}_2^2} + 2\delta \beta^2\\
        \leq& \E{a(w_{k\delta})} -m\delta \E{a(w_{k\delta})} + 2\delta^2L^2 a(w_{k\delta}) + 2\delta^2 L^2 R^2 + 2\delta \beta^2\\
        \leq& (1-m\delta/2)a(w_{k\delta}) + {m\delta} R^2 + 2\delta \beta^2
        \end{align*}

        Where the first inequality uses the upper bound on $\lrn{\nabla^2 a(y)}_2$ above, the second inequality uses the fact that $w_{(k+1)\delta} = \lrp{y_{k\delta} - \delta \nabla U(y_{k\delta}) = \xi(w_{k\delta}, \eta_k)}$, and $\E{ \xi(w_{k\delta}, \eta_k) | w_{k\delta}} = 0$, the third inequality uses claim 2. at the start of this proof, the fourth inequality uses item 2 of Assumption \ref{ass:xi_properties}. The fifth inequality uses claim 3. above, the sixth inequality uses our assumption that $\delta \leq \frac{m}{16L^2}$.

        Taking expectation wrt $w_{k\delta}$,
        \begin{align*}
        & \E{a(w_{(k+1)\delta})} \leq \E{a(w_{k})}-m\delta \lrp{\E{a(w_{k\delta})} - 2R^2 + 2\beta^2/m}\\
        \Rightarrow \qquad &
        \E{a(w_{(k+1)\delta})} - (2R^2/2 + 2\beta^2/m) \leq (1-m\delta) \lrp{\E{a(w_{k\delta})} - (2R^2 + 2\beta^2/m}
        \end{align*}

        Thus, if $\E{\|w_0\|_2^2} \leq 2R^2 + 2\beta^2/m$, then $\E{a(w_0)} - \lrp{2R^2 + 2\beta^2/m} \leq 0 $, then $\E{a(w_{k\delta})}  - \lrp{2R^2 + 2\beta^2/m}\leq 0$ for all $k$, which implies that
        \begin{align*}
        \E{\lrn{w_{k\delta}}_2^2} \leq 2\E{a(w_{k\delta})} + 4 R^2 \leq  8\lrp{R^2 + \beta^2/m}
        \end{align*}
        for all $k$.
    \end{proof}
    \end{subsection}

    \begin{subsection}{Divergence Bounds}

        \begin{lemma}
            \label{l:divergence_xt}
            Let $x_t$ be as defined in \eqref{e:marginal_x} (or equivalently \eqref{e:coupled_2_processes} or \eqref{e:coupled_4_processes_x}), initialized at $x_0$. Then for any $T\leq \frac{1}{16L}$,
            \begin{align*}
            \E{\lrn{x_T - x_0}_2^2} \leq 8 \lrp{T\beta^2 + T^2 L^2 \E{\|x_0\|_2^2}}
            \end{align*}
            If we additionally assume that $\E{\lrn{x_0}_2^2} \leq 8 \lrp{R^2 + \beta^2/m}$ and $T \leq \frac{\beta^2}{8L^2\lrp{R^2 + \beta^2/m}}$, then
            \begin{align*}
            \E{\lrn{x_T - x_0}_2^2} \leq 16 T\beta^2
            \end{align*}
        \end{lemma}
        \begin{proof}
            By Ito's Lemma,
            \begin{align*}
            & \ddt \E{\lrn{x_t}_2^2} \\
            =& 2\E{\lin{\nabla U(x_t), x_t - x_0}} + \E{ \tr\lrp{M(x_t)^2}}\\
            \leq& 2L \E{\lrn{x_t}_2 \lrn{x_t - x_0}_2} + \beta^2\\
            \leq& 2L \E{\lrn{x_t - x_0}_2^2} + 2L\E{\lrn{x_0}_2\lrn{x_t - x_0}_2} + \beta^2\\
            \leq& 2L \E{\lrn{x_t - x_0}_2^2} + L^2 T\E{\lrn{x_0}_2^2} + \frac{1}{T} \E{\lrn{x_t - x_0}_2^2} + \beta^2\\
            \leq& \frac{2}{T} \E{\lrn{x_t - x_0}_2^2} + \lrp{L^2 T \E{\lrn{x_0}_2^2} + \beta^2}
            \end{align*}
            where the first inequality is by item 1 of Assumption \ref{ass:U_properties} and item 2 of Assumption \ref{ass:xi_properties}, the second inequality is by triangle inequality, the third inequality is by Young's inequality, the last inequality is by our assumption on $T$.

            Applying Gronwall's inequality for $t\in[0,T]$,
            \begin{align*}
            & \lrp{\E{\lrn{x_t - x_0}_2^2} + L^2 T^2 \E{\lrn{x_0}_2^2} + T \beta^2} \\
            \leq& e^{2}\lrp{\E{\lrn{x_0 - x_0}} + L^2 T^2 \E{\lrn{x_0}_2^2} + T \beta^2}\\
            \leq& 8 L^2 T^2 \E{\lrn{x_0}_2^2} + T \beta^2
            \end{align*}
            This concludes our proof.
        \end{proof}

    \begin{lemma}
        \label{l:divergence_yt}
        Let $y_t$ be as defined in \eqref{e:marginal_y} (or equivalently \eqref{e:coupled_2_processes} or \eqref{e:coupled_4_processes_x}), initialized at $y_0$. Then for any $T$,
        \begin{align*}
        \E{\lrn{y_T - y_0}_2^2} \leq T^2 L^2 \E{\lrn{y_0}_2^2} + T\beta^2
        \end{align*}
        If we additionally assume that $\E{\lrn{y_0}_2^2} \leq 8 \lrp{R^2 + \beta^2/m}$ and $T \leq \frac{\beta^2}{8L^2\lrp{R^2 + \beta^2/m}}$, then
        \begin{align*}
        \E{\lrn{y_T - y_0}_2^2} \leq 2 T\beta^2
        \end{align*}
    \end{lemma}
    \begin{proof}
        Notice from the definition in \eqref{e:marginal_y} that $y_T - y_0\sim \N \lrp{-T \nabla U(y_0), T M(y_0)^2}$, the conclusion immediately follows from where the inequality is by item 1 of Assumption \ref{ass:U_properties} and item 2 of Assumption \ref{ass:xi_properties}, and the fact that
        \begin{align*}
        \tr\lrp{M(x)^2}  = \tr\lrp{\E{\xi(x,\eta) \xi(x,\eta)^T}} = \E{\lrn{\xi(x,\eta)}_2^2}
        \end{align*}
    \end{proof}

    \begin{lemma}
        \label{l:divergence_vt}
        Let $v_t$ be as defined in \eqref{e:coupled_4_processes_v}, initialized at $v_0$. Then for any $T=n\delta$,
        \begin{align*}
        \E{\lrn{v_T - v_0}_2^2} \leq T^2 L^2 \E{\lrn{v_0}_2^2} + T\beta^2
        \end{align*}
        If we additionally assume that $\E{\lrn{v_0}_2^2} \leq 8 \lrp{R^2 + \beta^2/m}$ and $T \leq \frac{\beta^2}{8L^2\lrp{R^2 + \beta^2/m}}$, then
        \begin{align*}
        \E{\lrn{v_T - v_0}_2^2} \leq 2 T\beta^2
        \end{align*}
    \end{lemma}
    \begin{proof}
    From \eqref{e:coupled_4_processes_v},
    \begin{align*}
    v_T - v_0 = - T \nabla U(v_0) + \sqrt{\delta} \sum_{i=0}^{n-1} \xi(v_0, \eta_i)
    \end{align*}
    Conditioned on the randomness up to time $i$, $\E{\xi(v_0,\eta_{i+1})}=0$. Thus
    \begin{align*}
    & \E{\lrn{v_T - v_0}_2^2}\\
    =& T^2\E{\lrn{\nabla U(v_0)}_2^2} + \delta \sum_{i=0}^{n-1} \E{\lrn{\xi(v_0,\eta_i)}_2^2}\\
    \leq& T^2 L^2 \E{\lrn{v_0}_2^2} + T \beta^2
    \end{align*}
    where the inequality is by item 1 of Assumption \ref{ass:U_properties} and item 2 of Assumption \ref{ass:xi_properties}.
    \end{proof}

    \begin{lemma}
        \label{l:divergence_wt}
        Let $w_t$ be as defined in \eqref{e:coupled_4_processes_w}, initialized at $w_0$. Then for any $T=n\delta$ such that $T \leq \frac{1}{2L}$,
        \begin{align*}
        \E{\lrn{w_T - w_0}_2^2} \leq 16 \lrp{T^2 L^2 \E{\lrn{w_0}_2^2} + T \beta^2}
        \end{align*}
        If we additionally assume that $\E{\lrn{w_0}_2^2} \leq 8 \lrp{R^2 + \beta^2/m}$ and $T \leq \frac{\beta^2}{8L^2\lrp{R^2 + \beta^2/m}}$, then
        \begin{align*}
        \E{\lrn{w_T - w_0}_2^2} \leq 32 T\beta^2
        \end{align*}
    \end{lemma}
    \begin{proof}
        \begin{align*}
        & \E{\lrn{w_{(k+1)\delta} - w_0}_2^2}\\
        =& \E{\lrn{w_{k\delta} - \delta \nabla U(w_{k\delta}) + \sqrt{\delta} \xi\lrp{w_{k\delta}, \eta_k} - w_0}_2^2}\\
        =& \E{\lrn{w_{k\delta} - \delta \nabla U(w_{k\delta}) - w_0}_2^2}  + \delta \E{\lrn{\xi\lrp{w_{k\delta}, \eta_k}}_2^2} \numberthis \label{e:t:sfddd}
        \end{align*}
        We can bound $\delta \E{\lrn{\xi\lrp{w_{k\delta}, \eta_k}}_2^2} \leq \delta \beta^2$ by item 2 of Assumption \ref{ass:xi_properties}.
        \begin{align*}
        & \E{\lrn{w_{k\delta} - \delta \nabla U(w_{k\delta}) - w_0}_2^2}\\
        \leq& \E{\lrp{\lrn{w_{k\delta} - w_0 - \delta \lrp{\nabla U(w_{k\delta}) - \nabla U(w_0)}}_2 + \delta \lrn{\nabla U(w_0)}_2}^2}\\
        \leq& \lrp{1 + \frac{1}{n}}\E{\lrn{w_{k\delta} - w_0 - \delta \lrp{\nabla U(w_{k\delta}) - \nabla U(w_0)}}_2^2 }\\
        &\quad + (1+n)\delta^2 \E{\lrn{\nabla U(w_0)}_2^2}\\
        \leq& \lrp{1 + \frac{1}{n}} \lrp{1 + \delta L}^2 \E{\lrn{w_{k\delta} - w_0}_2^2} + 2n\delta^2 L^2 \E{\lrn{w_0}_2^2}\\
        \leq& e^{1/n + 2\delta L }\E{\lrn{w_{k\delta} - w_0}_2^2} + 2n\delta^2 L^2 \E{\lrn{w_0}_2^2}
        \end{align*}
        where the first inequality is by triangle inequality, the second inequality is by Young's inequality, the third inequality is by item 1 of Assumption \ref{ass:U_properties}.

        Inserting the above into \eqref{e:t:sfddd} gives
        \begin{align*}
        \E{\lrn{w_{(k+1)\delta} - w_0}_2^2} \leq e^{1/n + 2\delta L }\E{\lrn{w_{k\delta} - w_0}_2^2} + 2n\delta^2 L^2 \E{\lrn{w_0}_2^2} + \delta \beta^2
        \end{align*}

        Applying the above recursively for $k=1...n$, we see that
        \begin{align*}
        & \E{\lrn{w_{n\delta} - w_0}_2^2}\\
        \leq& \sum_{k=0}^{n-1} e^{(n-k) \cdot (1/n + 2\delta L)} \cdot \lrp{2n\delta^2 L^2 \E{\lrn{w_0}_2^2} + \delta \beta^2}\\
        \leq& 16 \lrp{n^2 \delta^2 L^2 \E{\lrn{w_0}_2^2} + n\delta \beta^2}\\
        =& 16 \lrp{T^2 L^2 \E{\lrn{w_0}_2^2} + T \beta^2}
        \end{align*}

    \end{proof}

    \end{subsection}
    \begin{subsection}{Discretization Bounds}
        \label{ss:discretization_bonuds}
        \begin{lemma}
            \label{l:vt-wt}
            Let $v_{k\delta}$ and $w_{k\delta}$ be as defined in \eqref{e:coupled_4_processes_v} and \eqref{e:coupled_4_processes_w}. Then for any $\delta,n$, such that $T:= n\delta \leq \frac{1}{16L}$,
            \begin{align*}
            \E{\lrn{v_{T} - w_{T}}_2^2} \leq 8 \lrp{2T^2 L^2 \lrp{T^2L^2 \E{\lrn{v_0}_2^2} + T\beta^2} + T L_\xi^2 \lrp{16 \lrp{T^2 L^2 \E{\lrn{w_0}_2^2} + T \beta^2}}}
            \end{align*}
            If we additionally assume that $\E{\lrn{v_0}_2^2} \leq 8 \lrp{R^2 + \beta^2/m}$, $\E{\lrn{w_0}_2^2} \leq 8 \lrp{R^2 + \beta^2/m}$ and $T \leq \frac{\beta^2}{8L^2\lrp{R^2 + \beta^2/m}}$, then
            \begin{align*}
            \E{\lrn{v_T - w_T}_2^2} \leq 32 \lrp{T^2 L^2 + TL_\xi^2} T\beta^2
            \end{align*}

        \end{lemma}
        \begin{proof}
            Using the fact that conditioned on the randomness up to step $k$, $\E{\xi(v_0,\eta_{k+1}) - \xi(w_{k\delta}, \eta_{k+1})}=0$, we can show that for any $k\leq n$,
            \begin{align*}
            & \E{\lrn{v_{(k+1)\delta} - w_{(k+1)\delta}}_2^2}\\
            =& \E{\lrn{v_{k\delta} - \delta \nabla U(v_0) - w_{k\delta} +\delta \nabla U(w_{k\delta}) + \sqrt{\delta} \xi(w_0, \eta_k) - \sqrt{\delta} \xi(w_{k\delta}, \eta_k)}_2^2}\\
            =& \E{\lrn{v_{k\delta} - \delta \nabla U(v_0) - w_{k\delta} +\delta \nabla U(w_{k\delta})}_2^2} + \delta \E{\lrn{\xi(w_0, \eta_k) - \xi(w_{k\delta}, \eta_k)}_2^2}
            \numberthis \label{e:t:mkqwm}
            \end{align*}
            where the first inequality is by (Assumption on smoothness of U and xi).

            Using (smoothness of xi), and Lemma \ref{l:divergence_vt}, we can bound
            \begin{align*}
            & \delta \E{\lrn{\xi(w_0, \eta_k) - \xi(w_{k\delta}, \eta_k)}_2^2}\\
            \leq& \delta L_\xi^2 \E{\lrn{w_{k\delta} - w_0}_2^2}\\
            \leq& \delta L_\xi^2 \lrp{16 \lrp{T^2 L^2 \E{\lrn{w_0}_2^2} + T \beta^2}}
            \end{align*}

            We can also bound
            \begin{align*}
            & \E{\lrn{v_{k\delta} - \delta \nabla U(v_0) - w_{k\delta} +\delta \nabla U(w_{k\delta})}_2^2}\\
            \leq& \lrp{1 + \frac{1}{n}}\E{\lrn{v_{k\delta} - \delta \nabla U(v_{k\delta}) - w_{k\delta} +\delta \nabla U(w_{k\delta})}_2^2} + (1+n) \delta^2 \E{\lrn{\nabla U(v_{k\delta}) - \nabla U(v_0)}_2^2}\\
            \leq& \lrp{1+ \frac{1}{n}}\lrp{1+ \delta L}^2 \E{\lrn{v_{k\delta} - w_{k\delta}}_2^2} + 2n\delta^2 L^2 \E{\lrn{v_{k\delta} - v_0}_2^2}\\
            \leq& e^{1/n + 2\delta L}E{\lrn{v_{k\delta} - w_{k\delta}}_2^2} + 2n\delta^2 L^2 \E{\lrn{v_{k\delta} - v_0}_2^2}\\
            \leq& e^{1/n + 2\delta L}E{\lrn{v_{k\delta} - w_{k\delta}}_2^2} + 2n\delta^2 L^2 \lrp{T^2L^2 \E{\lrn{v_0}_2^2} + T\beta^2}
            \end{align*}
            where the first inequality is by Young's inequality and the second inequality is by item 1 of Assumption \ref{ass:U_properties}, the fourth inequality uses Lemma \ref{l:divergence_vt}.

            Substituting the above two equation blocks into \eqref{e:t:mkqwm}, and applying recursively for $k=0...n-1$ gives
            \begin{align*}
            & \E{\lrn{v_{T} - w_{T}}_2^2} \\
            =& \E{\lrn{v_{n\delta} - w_{n\delta}}_2^2} \\
            \leq& e^{1+2n\delta L} \lrp{2n^2\delta^2 L^2 \lrp{T^2L^2 \E{\lrn{v_0}_2^2} + T\beta^2} + n\delta L_\xi^2 \lrp{16 \lrp{T^2 L^2 \E{\lrn{w_0}_2^2} + T \beta^2}}}\\
            \leq& 8 \lrp{2T^2 L^2 \lrp{T^2L^2 \E{\lrn{v_0}_2^2} + T\beta^2} + T L_\xi^2 \lrp{16 \lrp{T^2 L^2 \E{\lrn{w_0}_2^2} + T \beta^2}}}
            \end{align*}

            the last inequality is by noting that $T = n\delta \leq \frac{1}{4L}$.
        \end{proof}
    \end{subsection}

    \begin{section}{Regularity of $M$ and $N$}\label{ss:mnregularity}
        \begin{lemma}{\label{l:M_is_regular}}
            \begin{align*}
            &1.\ \tr\lrp{M(x)^2} \leq \beta^2\\
            &2.\ \tr\lrp{(M(x)^2 - M(y)^2)^2} \leq 16 \beta^2 L_\xi^2  \|x-y\|_2^2\\
            &3.\ \tr\lrp{(M(x)^2 - M(y)^2)^2} \leq 32\beta^3 L_\xi \|x-y\|_2
            \end{align*}
        \end{lemma}
        \begin{proof}
            In this proof, we will use the fact that $\xi(\cdot,\eta)$ is $L_\xi$-Lipschitz from Assumption \ref{ass:xi_properties}.

            The first property is easy to see:
            \begin{align*}
            & \tr\lrp{M(x)^2}\\
            =& \tr\lrp{\Ep{\eta}{\xi(x,\eta) \xi(x,\eta)^T}}\\
            =& \Ep{\eta}{\tr\lrp{{\xi(x,\eta) \xi(x,\eta)^T}}}\\
            =& \Ep{\eta}{\lrn{\xi(x,\eta)}_2^2}\\
            \leq& \beta^2
            \end{align*}

            We now prove the second and third claims. Consider a fixed $x$ and fixed $y$, let $u_{\eta} := \xi(x,\eta)$, $v_{\eta} := \xi(y,\eta)$. Then

            \begin{align*}
            & \tr\lrp{\lrp{M(x)^2 - M(y)^2}^2 }\\
            =& \tr\lrp{\lrp{\Ep{\eta}{u_\eta u_\eta^T - v_\eta v_\eta^T}}^2}\\
            =& \tr\lrp{\Ep{\eta, \eta'}{\lrp{u_\eta u_\eta^T - v_\eta v_\eta^T} \lrp{u_{\eta'}u_{\eta'}^T - v_{\eta'} v_{\eta'}^T}}}\\
            =& \Ep{\eta,\eta'}{\tr\lrp{\lrp{u_\eta u_\eta^T - v_\eta v_\eta^T} \lrp{u_{\eta'}u_{\eta'}^T - v_{\eta'} v_{\eta'}^T}}}
            \end{align*}

            For any fixed $\eta$ and $\eta'$, let's further simplify notation by letting $u,u',v,v'$ denote $u_\eta, u_{\eta'}, v_\eta, v_{\eta'}$. Thus
            \begin{align*}
            &\tr\lrp{\lrp{uu^T - vv^T} \lrp{u'u'^T - v'v'^T}}\\
            =& \tr\lrp{ \lrp{(u-v)v^T + v(u-v)^T + (u-v)(u-v)^T}   \lrp{(u'-v')v'^T + v'(u'-v')^T + (u'-v')(u'-v')^T} }\\
            =& \tr\lrp{ (u-v) v^T (u'-v') v'^T} + \tr\lrp{ (u-v) v^T v'(u'-v')^T } + \tr\lrp{ (u-v) v^T  (u'-v')(u'-v')^T }\\
            &\quad + \tr\lrp{ v(u-v)^T (u'-v') v'^T} + \tr\lrp{ v(u-v)^T v'(u'-v')^T } + \tr\lrp{ v(u-v)^T (u'-v')(u'-v')^T }\\
            &\quad + \tr\lrp{ (u-v)(u-v)^T (u'-v') v'^T} + \tr\lrp{ (u-v)(u-v)^T v'(u'-v')^T }\\
            &\quad  + \tr\lrp{ (u-v)(u-v)^T (u'-v')(u'-v')^T }\\
            \leq& \min\lrbb{16 \beta^2 L_\xi^2 \lrn{x-y}_2^2, 32\beta^3 L_{\xi}\|x-y\|_2}
            \end{align*}

            Where the last inequality uses Assumption \ref{ass:xi_properties}.2 and \ref{ass:xi_properties}.3; in particular, $\lrn{v}_2\leq \beta$ and $\lrn{u-v}_2 \leq \min\lrbb{2\beta, L_\xi \|x-y\|_2}$. This proves 2. and 3. of the Lemma statement.
        \end{proof}

        \begin{lemma}\label{l:N_is_regular}
            Let $N(x)$ be as defined in \eqref{d:N} and $\LN$ be as defined in \eqref{d:constants}. Then
            \begin{align*}
            1.\ &\tr\lrp{N(x)^2} \leq \beta^2\\
            2.\ &\tr\lrp{\lrp{N(x) - N(y)}^2} \leq \LN^2\lrn{x-y}_2^2\\
            3.\ &\tr\lrp{\lrp{N(x) - N(y)}^2} \leq \frac{8 \beta^2}{\cm} \cdot \LN \lrn{x-y}_2.
            \end{align*}
        \end{lemma}
        \begin{proof}[Proof of Lemma \ref{l:N_is_regular}]
            The first inequality holds because $N(x)^2 := M(x)^2 - \cm^2 I$, and then applying Lemma \ref{l:M_is_regular}.1, and the fact that $\tr\lrp{M(x)^2 - \cm^2 I } \leq \tr\lrp{M(x)^2}$ by Assumption \ref{ass:xi_properties}.4.

            The second inequality is a immediate consequence of Lemma \ref{l:eldan-matrix}, Lemma \ref{l:M_is_regular}.2, and the fact that $\lambda_{min} \lrp{N(x)^2} = \lambda_{min} \lrp{M(x)^2 - \cm^2} \geq \cm^2$ by Assumption \ref{ass:xi_properties}.4.

            The proof for the third inequality is similar to the second inequality, and follows from Lemma \ref{l:M_is_regular} and Lemma \ref{l:eldan-matrix}.

        \end{proof}

        \begin{lemma}[Simplified version of Lemma 1 from \cite{eldan2018clt}]
            \label{l:eldan-matrix}
            Let $A$, $B$ be positive definite matrices. Then
            \begin{align*}
            \tr\lrp{\lrp{\sqrt{A} - \sqrt{B}}^2} \leq \tr\lrp{(A-B)^2 A^{-1}}
            \end{align*}
        \end{lemma}
    \end{section}


\end{section}
\begin{section}{Defining $f$ and related inequalities}
    \label{s:defining-q}
    
    In this section, we define the Lyapunov function $f$ which is central to the proof of our main results. Here, we give an overview of the various functions defined in this section:
    \begin{enumerate}
        \item $g(z): \Re^d \to \Re^+$: A smoothed version of $\lrn{z}_2$, with bounded derivatives up to third order.
        \item $q(r): \Re^+ \to \Re^+$: A concave potential function, similar to the one defined in \cite{eberle2016reflection}, which has bounded derivatives up to third order everywhere except at $r=0$.
        \item $f(z) = q(g(z)): \Re^d \to \Re^+$, a concave function which upper and lower bounds $\lrn{z}_2$ within a constant factor, has bounded derivatives up to third order everywhere.
    \end{enumerate}

    \begin{lemma}[Properties of $f$]
        \label{l:fproperties}
        Let $\epsilon$ satisfy $\epsilon \leq \frac{\Rq}{\aq\Rq^2 + 1}$. We define the function
        \begin{align*}
        f(z) := q(g(z))
        \end{align*}
        Where $q$ is as defined in \eqref{d:f} Appendix \ref{ss:defining-q}, and $g$ is as defined in Lemma \ref{l:gproperties} (with parameter $\epsilon$). Then
        \begin{enumerate}
            \item
            \begin{enumerate}
                \item $\nabla f(z) = q'(g(z)) \cdot \nabla g(z)$
                \item For $\lrn{z}_2 \geq 2\epsilon$, $\nabla f(z) = q'(g(z)) \frac{z}{\|z\|_2}$
                \item For all $z$, $\lrn{\nabla f(z)}_2 \leq 1$.
            \end{enumerate}

            \item
            \begin{enumerate}
                \item $\nabla^2 f(z) = q''(g(z)) \nabla g(z) \nabla g(z)^T + q'(g(z)) \nabla^2 g(z)$
                \item For $r\geq 2\epsilon$,
                $\nabla^2 f(z) = q''(g(z)) \frac{zz^T}{\|z\|_2^2} + q'(g(z)) \frac{1}{\|z\|_2} \lrp{I - \frac{zz^T}{\|z\|_2^2}}$
                \item For all $z$, $\lrn{\nabla^2 f(z)}_2 \leq \frac{2}{\epsilon}$
                \item For all $z,v$, $v^T \nabla^2 f(z) v \leq\frac{q'(g(z))}{\|z\|_2}$
            \end{enumerate}
            \item For any $z$, $\lrn{\nabla^3 f(z)}_2 \leq \frac{9}{\epsilon^2}$
            \item For any $z$, $f(z) \in \lrb{\frac{1}{2}\exp\lrp{-\frac{7\aq\Rq^2}{3}} g(\|z\|_2), g(\|z\|_2)} \in \lrb{\frac{1}{2}\exp\lrp{-\frac{7\aq\Rq^2}{3}} (\|z\|_2 - 2\epsilon), \|z\|_2}$
        \end{enumerate}
    \end{lemma}
    \begin{proof}[Proof of Lemma \ref{l:fproperties}]
        \begin{enumerate}
            \item
            \begin{enumerate}
                \item chain rule
                \item Use definition of $\nabla g(z)$ from Lemma \ref{l:gproperties}.
                \item By definition, $\nabla f(z) = q'(g(z)) \nabla g(z)$. From Lemma \ref{l:qproperties}, $\lrabs{q'(g(z))} \leq 1$. By definition, $\nabla g(z) = h'(\lrn{z}_2) \frac{z}{\lrn{z}_2}$. Our conclusion follows from $h' \leq 1$ using item 2 of Lemma \ref{l:hproperties}.
            \end{enumerate}
            \item
            \begin{enumerate}
                \item chain rule
                \item by item 2 b) of  Lemma \ref{l:gproperties}
                \item by item 1 c) and item 2 d) of Lemma \ref{l:gproperties}, and item 3 and item 4 of Lemma \ref{l:qproperties}, and our assumption that $\epsilon\leq \frac{\Rq}{\aq + \Rq^2 + 1}$.
                \item by item 4 of Lemma \ref{l:qproperties}), and items 2 c) and 2 d) of Lemma \ref{l:gproperties}, and our expression for $\nabla^2 f(z)$ established in item 2 a).
            \end{enumerate}
            \item It can be verified that
            \begin{align*}
            \nabla^3 f(z) =& q'''(g(z)) \cdot \nabla g(z)^{\bo 3} + q''(g(z)) \nabla g(z) \bo \nabla^2 g(z) + q''(g(z)) \nabla^2 g(z) \bo \nabla g(z) \\
            &\quad + q''(g(z)) \nabla g(z) \bo \nabla^2 g(z) + q'(g(z)) \nabla^3 g(z)
            \end{align*}

            Thus
            \begin{align*}
            \lrn{\nabla^3 f(z)}_2
            \leq& \lrabs{q'''(g(z))} \lrn{\nabla g(z)}_2^3 + 3 q''(g(z)) \lrn{\nabla g(z)}_2 \lrn{\nabla^2 g(z)}_2 + q'(g(z)) \lrn{\nabla^3 g(z)}\\
            \leq& 5\lrp{\aq + \frac{1}{\Rq^2}} \lrp{\aq\Rq^2 + 1} + 3\lrp{\frac{5\aq\Rq}{4} + \frac{4}{\Rq}}\cdot \frac{1}{\epsilon} + \frac{1}{\epsilon^2}\\
            \leq& \frac{9}{\epsilon^2}
            \end{align*}
            Where the first inequality uses Lemma \ref{l:qproperties} and Lemma \ref{l:gproperties}, and the second inequality assumes that $\epsilon \leq \frac{\Rq}{\aq\Rq^2 + 1}$
            \item
            \begin{align*}
            f(z) \in \lrb{\frac{1}{2}\exp\lrp{-\frac{7\aq\Rq^2}{3}} g(\|z\|_2), g(\|z\|_2)} \in \lrb{\frac{1}{2}\exp\lrp{-\frac{7\aq\Rq^2}{3}} (\|z\|_2 - 2\epsilon), \|z\|_2}
            \end{align*}
            The first containment is by Lemma \ref{l:qproperties}.\ref{f:q(r)_bounds}: $\frac{1}{2}\exp\lrp{-\frac{7\aq\Rq^2}{3}}\cdot g(z)  \leq q(g(z)) \leq g(z)$. THe second containment is by Lemma \ref{l:gproperties}.4: $g(\|z\|_2) \in [\|z\|_2-2\epsilon, \|z\|_2]$.

        \end{enumerate}

    \end{proof}

    \begin{lemma}[Properties of $h$]
        \label{l:hproperties}
        Given a parameter $\epsilon$, define
        \begin{align*}
        h(r) := \threecase
        {\frac{r^3}{6\epsilon^2}}{r\in [0,\epsilon]}
        {\frac{\epsilon}{6} + \frac{r-\epsilon}{2} + \frac{(r-\epsilon)^2}{2\epsilon} - \frac{(r-\epsilon)^3}{6\epsilon^2}}{r\in[\epsilon, 2\epsilon]}
        {r }{r\geq 2\epsilon}
        \end{align*}
        \begin{enumerate}
            \item The derivatives of $h$ are as follows:
            \begin{align*}
            h'(r) =& \threecase
            {\frac{r^2}{2\epsilon^2}}{r\in [0,\epsilon]}
            {\frac{1}{2} + \frac{r-\epsilon}{\epsilon} - \frac{(r-\epsilon)^2}{2\epsilon^2}}{r\in[\epsilon, 2\epsilon]}
            {1}{r\geq 2\epsilon}\\
            h''(r) =& \threecase
            {\frac{r}{\epsilon^2}}{r\in [0,\epsilon]}
            {\frac{1}{\epsilon} - \frac{r-\epsilon}{\epsilon^2}}{r\in[\epsilon, 2\epsilon]}
            {0}{r\geq 2\epsilon}\\
            h'''(r) =& \threecase
            {\frac{1}{\epsilon^2}}{r\in [0,\epsilon]}
            {-\frac{1}{\epsilon^2}}{r\in[\epsilon, 2\epsilon]}
            {0}{r\geq 2\epsilon}
            \end{align*}
            \item
            \begin{enumerate}
                \item $h'$ is positive, motonically increasing.
                \item $h'(0)=0$, $h'(r) =1$ for $r\geq \epsilon$
                \item $\frac{h'(r)}{r}\leq \min\lrbb{\frac{1}{\epsilon}, \frac{1}{r}}$ for all $r$
            \end{enumerate}
            \item
            \begin{enumerate}
                \item $h''(r)$ is positive
                \item $h''(r) = 0$ for $r=0$ and $r\geq 2\epsilon$
                \item $h''(r) \leq \frac{1}{\epsilon}$
                \item $\frac{h''(r)}{r} \leq \frac{1}{\epsilon^2}$
            \end{enumerate}
            \item $\lrabs{h'''(r)} \leq \frac{1}{\epsilon^2}$
            \item $r-2\epsilon\leq h(r) \leq r$
        \end{enumerate}

    \end{lemma}
    \begin{proof}[Proof of Lemma \ref{l:hproperties}]
        The claims can all be verified with simple algebra.
    \end{proof}

    \begin{lemma}[Properties of $g$]
        \label{l:gproperties}
        Given a parameter $\epsilon$, let us define
        \begin{align*}
        g(z) := h(\|z\|_2)
        \end{align*}
        Where $h$ is as defined in Lemma \ref{l:hproperties} (using parameter $\epsilon$). Then
        \begin{enumerate}
            \item
            \begin{enumerate}
                \item $\nabla g(z) = h'(\|z\|_2) \frac{z}{\|z\|_2}$
                \item For $\|z\|_2 \geq 2\epsilon$, $\nabla g(z) = \frac{z}{\|z\|_2}$.
                \item For any $\|z\|_2$, $\lrn{\nabla g(z)}_2\leq 1$
            \end{enumerate}

            \item
            \begin{enumerate}
                \item $\nabla^2 g(z) = h''(\|z\|_2)\frac{zz^T}{\|z\|_2^2} + h'(\|z\|_2) \frac{1}{\|z\|_2}\lrp{I - \frac{z z^T}{\|z\|_2^2}}$
                \item For $\lrn{z}_2\geq 2\epsilon$, $\nabla^2 g(z) =  \frac{1}{\|z\|_2}\lrp{I - \frac{z z^T}{\|z\|_2^2}}$.
                \item For $\lrn{z}_2\geq 2\epsilon$, $\lrn{\nabla^2 g(z)}_2 = \frac{1}{\|z\|_2}$
                \item For all $z$, $\lrn{\nabla^2 g(z)}_2 \leq \frac{1}{\epsilon}$
            \end{enumerate}
            \item $\lrn{\nabla^3 g(z)}_2 \leq \frac{5}{\epsilon^2}$
            \item $\|z\|_2 - 2\epsilon \leq g(z) \leq \|z\|_2$.
        \end{enumerate}
    \end{lemma}
    \begin{proof}[Proof of Lemma \ref{l:gproperties}]
        All the properties can be verified with algebra. We provide a proof for $3.$ since it is a bit involved.

        Let us define the functions $\kappa^1(z) = \nabla(\|z\|_2), \kappa^2(z) = \nabla^2(\|z\|_2), \kappa^3(z) = \nabla^3(\|z\|_2)$. Specifically,
        \begin{align*}
        \kappa^1(z) =& \frac{z}{\|z\|_2}\\
        \kappa^2(z) =& \frac{1}{\|z\|_2}\lrp{I - \frac{zz^T}{\|z\|_2^2}}\\
        \kappa^3(z) =& -\frac{1}{\|z\|_2^2} \frac{z}{\|z\|_2} \bo \lrp{I - \frac{zz^T}{\|z\|_2^2}} + \frac{1}{\|z\|_2} \lrp{\frac{z}{\|z\|_2} \bo \kappa^2(z) + \kappa^2(z) \bo \frac{z}{\|z\|_2}}
        \end{align*}

        It can be verified that
        \begin{align*}
        \lrn{\kappa^2(z)}_2 =& \frac{1}{\|z\|_2}\\
        \lrn{\kappa^3(z)}_2 =& \frac{1}{\|z\|_2^2}
        \end{align*}

        It can be verified that $\nabla^2 g(z)$ has the following form:
        \begin{align*}
        &\nabla^3 g(z) = h'''(\|z\|_2) \lrp{\kappa^1(z)}^{\bo 3} + h''(\|z\|_2) \kappa^1(z) \bo \kappa^2(z) + h''(\|z\|_2) \kappa^2(z) \bo \kappa^1(z) \\
        &\quad + h'(\|z\|_2) \kappa^3(z) + h''(\|z\|_2) \kappa^1(z) \bo \kappa^2(z)
        \end{align*}

        Thus
        \begin{align*}
        \lrn{\nabla^3 g(z)}_2 \leq \lrabs{h'''(\|z\|_2)} + 3\frac{h''(\|z\|_2)}{\|z\|_2} + \frac{h'(\|z\|_2)}{\|z\|_2^2} \leq \frac{5}{\epsilon^2}
        \end{align*}

        Where we use properties of $h$ from Lemma \ref{l:hproperties}.

        The last claim follows immediately from Lemma \ref{l:hproperties}.4.
    \end{proof}

    \begin{subsection}{Defining q}
        \label{ss:defining-q}

        In this section, we define the function $q$ that is used in Lemma \ref{l:fproperties}. Our construction is a slight modification to the original construction in \cite{eberle2011reflection}.

        Let $\aq$ and $\Rq$ be as defined in \eqref{d:constants}. We begin by defining auxiliary
        functions $\psi(r)$, $\Psi(r)$ and $\nu(r)$, all from $ \Re^+$ to $\Re$:
        \begin{align}
        \label{d:psietal}
        &\psi(r) := e^{- \aq \tau(r)}\,, \qquad
        &\Psi(r) := \int_0^r \psi(s) ds\,, \qquad \nu(r) := 1- \frac{1}{2}
        \frac{\int_0^{r}\frac{ \mu(s) \Psi(s)}{\psi(s)} ds}{\int_0^{4\Rq}\frac{ \mu(s) \Psi(s)}{\psi(s)}ds}\,,
        \end{align}

        Where $\tau(r)$ and $\mu(r)$ are as defined in Lemma \ref{l:tau} and Lemma \ref{l:mu} with $\R = \Rq$.

        Finally we define $q$ as
        \begin{equation}
        \label{d:f}
        q(r) := \int_0^r \psi(s) \nu(s) ds.
        \end{equation}
        We now state some useful properties of the distance function $q$.
        \begin{lemma} \label{l:qproperties} The function $q$ defined in \eqref{d:f} has the following properties.
            \begin{enumerate}[label={\arabic*}.]
                \item For all $r\leq \Rq $, $q''(r) + \aq q'(r) \cdot r \leq - \frac{\exp\lrp{-\frac{7\aq\Rq^2}{3}}}{32\Rq^2} q(r)$
                \label{f:contraction}
                \item For all $r$, $\frac{\exp\lrp{-\frac{7\aq\Rq^2}{3}}}{2}\cdot r  \leq q(r) \leq r$
                \label{f:q(r)_bounds}
                \item For all $r$, $\frac{\exp\lrp{-\frac{7\aq\Rq^2}{3}}}{2}\leq q'(r) \leq 1$
                \label{f:q'(r)_bounds}
                \item For all $r$, $q''(r) \leq 0$ and $\lrabs{q''(r)} \leq \lrp{\frac{5\aq\Rq}{4} + \frac{4}{\Rq}}$
                \label{f:q''(r)_bounds}
                \item For all $r$, $\lrabs{q'''(r)} \leq 5\aq + 2\aq\lrp{\aq\Rq^2 + 1} + \frac{2(\aq\Rq^2 + 1)}{\Rq^2}$
                \label{f:q'''(r)_bounds}
            \end{enumerate}
        \end{lemma}

        \begin{proof}[Proof of Lemma \ref{l:qproperties}]

            \textbf{Proof of \ref{f:contraction}}
            It can be verified that
            \begin{align*}
            \psi'(r)
            =& \psi(r) (-\aq \tau'(r))\\
            \psi''(r)
            =& \psi(r) \lrp{\lrp{\aq \tau'(r)}^2 + \aq \tau''(r)}\\
            \nu'(r)=& -\frac{1}{2} \frac{\frac{\mu(r)\Psi(r)}{\psi(r)}}{\int_0^{4\Rq} \frac{\mu(s)\Psi(s)}{\psi(s)}ds}
            \end{align*}

            For $r\in[0,\Rq]$, $\tau'(r) = r$, so that $\psi'(r) = \psi(r)(-\aq r)$. Thus
            \begin{align*}
            q'(r)
            =& \psi(r) \nu(r)\\
            q''(r)
            =& \psi'(r) \nu(r) + \psi(r) \nu'(r)\\
            =& \psi(r) \nu(r) (-\aq r) + \psi(r) \nu'(r)\\
            =& -\aq r \nu'(r) + \psi(r) \nu'(r)\\
            q''(r) + \aq r q'(r)
            =& \psi(r) \nu'(r)\\
            =& - \frac{1}{2}\frac{\mu(r)\Psi(r)}{\int_0^{4\Rq}\frac{\mu(s)\Psi(s)}{\psi(s)} ds}\\
            =& - \frac{1}{2}\frac{\Psi(r)}{\int_0^{4\Rq}\frac{\mu(s)\Psi(s)}{\psi(s)} ds}
            \end{align*}
            Where the last equality is by definition of $\mu(r)$ in Lemma \ref{l:mu} and the fact that $r\leq \Rq$.

            We can upper bound
            \begin{align*}
            \int_0^{4\Rq} \frac{\mu(s) \Psi(s)}{\psi(s)} ds
            \leq \int_0^{4\Rq} \frac{\Psi(s)}{\psi(s)} ds
            \leq \frac{\int_0^{4\Rq} s ds}{\psi(4\Rq)}
            = \frac{16\Rq^2}{\psi(4\Rq)}
            \leq& 16 \Rq^2 \cdot \exp\lrp{\frac{7\aq\Rq^2}{3}}
            \end{align*}
            Where the first inequality is by Lemma \ref{l:mu}, the second inequality is by the fact that $\psi(s)$ is monotonically decreasing, the third inequality is by Lemma \ref{l:tau}.

            Thus
            \begin{align*}
            q''(r) + \aq r q'(r)
            \leq& -\frac{1}{2} \lrp{\frac{\exp\lrp{-\frac{7\aq\Rq^2}{3}}}{16\Rq^2}} \Psi(r)\\
            \leq& -\frac{\exp\lrp{-\frac{7\aq\Rq^2}{3}}}{32\Rq^2} q(r)
            \end{align*}
            Where the last inequality is by $\Psi(r) \geq q(r)$.

            \textbf{Proof of \ref{f:q(r)_bounds}}
            Notice first that $\nu(r) \geq \frac{1}{2}$ for all $r$. Thus
            \begin{align*}
            q(r)
            :=& \int_0^r \psi(s) \nu(s) ds\\
            \geq& \frac{1}{2} \int_0^r \psi(s) ds\\
            \geq& \frac{\exp\lrp{-\frac{7\aq\Rq^2}{3}}}{2}\cdot r
            \end{align*}
            Where the last inequality is by Lemma \ref{l:tau}.

            \textbf{Proof of \ref{f:q'(r)_bounds}}
            By definition of $f$, $q'(r) = \psi(r) \nu(r) $, and
            \begin{align*}
            \frac{\exp\lrp{-\frac{7\aq\Rq^2}{3}}}{2}\leq \psi(r)\nu(r)\leq 1
            \end{align*}
            Where we use Lemma \ref{l:tau} and the fact that $\nu(r) \in [1/2,1]$

            \textbf{Proof of \ref{f:q''(r)_bounds}}
            Recall that
            \begin{align*}
            q''(r) = \psi'(r) \nu(r) + \psi(r) \nu'(r)
            \end{align*}

            That $q''\leq 0$ can immediately be verified from the definitions of $\psi$ and $\nu$.

            Thus
            \begin{align*}
            \lrabs{q''(r)}
            \leq& \lrabs{\psi'(r) \nu(r)} + \lrabs{\psi(r) \nu'(r)}\\
            \leq& \aq \tau'(r) + \lrabs{\psi(r) \nu'(r)}
            \end{align*}
            From Lemma \ref{l:tau}, we can upperbound $\tau'(r) \leq \frac{5\Rq}{4}$. In addition, $\Psi(r) = \int_0^r \psi(s) \geq r \psi(r)$, so that
            \begin{align*}
            \numberthis \label{e:l:psi(r)/psi(r)}
            \frac{\Psi(r)}{\psi(r)} \geq r
            \end{align*}
            (Recall again that $\psi(s)$ is monotonically decreasing).
            Thus $\Psi(r)/r \geq r$ for all $r$. In addition, using the fact that $\psi(r) \leq 1$,
            \begin{align*}
            \numberthis \label{e:l:psi(r)_upperbound}
            \Psi(r) = \int_0^r \psi(s)ds \leq r
            \end{align*}

            Combining the previous expressions,
            \begin{align*}
            \lrabs{\psi(r) \nu'(r)}
            =& \lrabs{ \frac{1}{2}\frac{\mu(r)\Psi(r)}{\int_0^{4\Rq}\frac{\mu(s)\Psi(s)}{\psi(s)} ds}}\\
            \leq& \lrabs{\frac{1}{2}\frac{\mu(r) r}{\int_0^{\Rq}\frac{\Psi(s)}{\psi(s)} ds}}\\
            \leq& \lrabs{\frac{1}{2} \frac{4\Rq}{\int_0^{\Rq} s ds}}\\
            \leq& \frac{4}{\Rq}
            \end{align*}
            Where the first inequality are by definition of $\mu(r)$ and \eqref{e:l:psi(r)_upperbound}, and the second inequality is by \eqref{e:l:psi(r)/psi(r)} and the fact that $\mu(r) = 0$ for $r\geq 4\Rq$. Combining with our bound on $\psi'(r) \nu(r)$ gives the desired bound.

            \textbf{Proof of \ref{f:q'''(r)_bounds}}
            \begin{align*}
            q'''(r) = \psi''(r) \nu(r) + 2\psi'(r) \nu'(r) + \psi(r) \nu''(r)
            \end{align*}
            We first bound the middle term:
            \begin{align*}
            \lrabs{\psi'(r) \nu'r)}
            =& \lrabs{\psi(r)(\aq \tau'(r)) \nu'r)}\\
            \leq& \aq \lrabs{\tau'(r)}\lrabs{\psi(r) \nu'r)}\\
            \leq& \frac{5\aq\Rq}{4} \cdot \frac{4}{\Rq}\\
            \leq& 5\aq
            \end{align*}
            Where the second last line follows form Lemma \ref{l:tau} and our proof of \ref{f:q''(r)_bounds}.

            Next,
            \begin{align*}
            \psi''(r) = \psi(r) \lrp{\aq^2 \tau'(r)^2 - \aq \tau''(r)}
            \end{align*}
            Thus applying Lemma \ref{l:tau}.1 and Lemma \ref{l:tau}.3,
            \begin{align*}
            \lrabs{\psi''(r) \nu(r)} \leq& 2\aq^2\Rq^2 + \aq
            \end{align*}

            Finally,
            \begin{align*}
            \nu''(r)
            =& \frac{1}{2\int_0^{4\Rq}\frac{\mu(s)\Psi(s)}{\psi(s)} ds}\cdot \frac{d}{dr} {\mu(r)\Psi(r)/\psi(r)}
            \end{align*}

            Expanding the numerator,
            \begin{align*}
            \frac{d}{dr} \frac{\mu(r) \Psi(r)}{\psi(r)}
            =& \mu'(r) \frac{\Psi(r)}{\psi(r)} + \mu(r) - \mu(r) \frac{\Psi(r) \psi'(r)}{\psi(r)^2}\\
            =& \mu'(r) \frac{\Psi(r)}{\psi(r)} + \mu(r) + \mu(r) \frac{\Psi(r) \psi(r)\aq \tau'(r)}{\psi(r)^2}
            \end{align*}

            Thus
            \begin{align*}
            \psi(r) \nu''(r) = \frac{1}{2\int_0^{4\Rq}\frac{\mu(s)\Psi(s)}{\psi(s)} ds}\cdot\lrp{\mu'(r) \Psi(r) + \mu(r) \psi(r) + \mu(r) \Psi(r) \aq \tau'(r)}
            \end{align*}
            Using the same argument as from the proof of \ref{f:q''(r)_bounds}, we can bound
            \begin{align*}
            \frac{1}{2\int_0^{4\Rq}\frac{\mu(s)\Psi(s)}{\psi(s)} ds}
            \leq& \frac{1}{2\int_0^\Rq s ds}\\
            \leq& \frac{1}{\Rq^2}
            \end{align*}
            Finally, from Lemma \ref{l:mu}, $\lrabs{\mu'(r)}\leq \frac{\pi}{6\Rq}$, so
            \begin{align*}
            \lrabs{\psi(r) \nu''(r)}\leq& \frac{\pi/6 + 1 + 5\aq \Rq^2/4}{\Rq^2}\\
            \leq& \frac{2(\aq\Rq^2 + 1)}{\Rq^2}
            \end{align*}
        \end{proof}

        \begin{lemma}\label{l:tau}
            Let $\tau(r): [0,\infty) \to \Re$ be defined as
            \begin{align*}
            \tau(r)=\fourcase
            {\frac{r^2}{2}}{r\leq \R}
            {\frac{\R^2}{2} + \R (r-\R) + \frac{(r-\R)^2}{2}- \frac{(r-\R)^3}{3\R}}{r\in[\R,2\R]}
            {\frac{5\R^2}{3} + \R(r-2\R) - \frac{(r-2\R)^2}{2} + \frac{(r-2\R)^3}{12\R}}{r\in [2\R,4\R]}
            {\frac{7\R^2}{3}}{r\geq 4\R]}
            \end{align*}
            Then
            \begin{enumerate}
                \item $\tau'(r) \in [0, \frac{5\R}{4}]$, with maxima at $r= \frac{3\R}{2}$. $\tau'(r) = 0$ for $r\in \lrbb{0}\bigcup [4\R, \infty)$
                \item As a consequence of 1, $\tau(r)$ is monotonically increasing
                \item $\tau''(r) \in [-1,1]$
            \end{enumerate}
        \end{lemma}
        \begin{proof}[Proof of Lemma \ref{l:tau}]
            We provide the derivatives of $\tau$ below. The claims in the Lemma can then be immediately verified.
            \begin{align*}
            \tau'(r) =&
            \fourcase
            {r}{r\leq \R}
            {\R + (r-\R) - \frac{(r-\R)^2}{\R}}{r\in[\R,2\R]}
            {\R - (r-2\R) + \frac{(r-2\R)^2}{4\R} }{r\in [2\R,4\R]}
            {0}{r\geq 4\R]}
            \end{align*}

            \begin{align*}
            \tau''(r) =&
            \fourcase
            {1}{r\leq \R}
            {1-\frac{2(r-\R)}{\R}}{r\in[\R,2\R]}
            {-1 + \frac{r-2\R}{2\R}}{r\in [2\R,4\R]}
            {0}{r\geq 4\R]}
            \end{align*}

        \end{proof}

        \begin{lemma}\label{l:mu}
            Let
            \begin{align*}
            \mu(r) := \threecase{1}{r \leq \R}{\frac{1}{2} + \frac{1}{2} \cos \lrp{\frac{\pi (r-\R)}{3\R}}}{r\in[\R,4\R]}{0}{r\geq 4\R}
            \end{align*}
            Then
            \begin{align*}
            \mu'(r) := \threecase{0}{r \leq \R}{-\frac{\pi}{6\R} \sin \lrp{\frac{\pi (r-\R)}{\R}}}{r\in[\R,4\R]}{0}{r\geq 4\R}
            \end{align*}
            Furthermore, $\mu'(r) \in [-\frac{\pi}{6\R},0]$
        \end{lemma}
        This Lemma can be easily verified by algebra.
    \end{subsection}
\end{section}

\begin{section}{Miscellaneous}
The following Theorem, taken from \cite{eldan2018clt}, establishes a quantitative CLT.

\begin{theorem}
\label{t:zhai}
Let $X_1...X_n$ be random vectors with mean 0, covariance $\Sigma$, and $\lrn{X_i}\leq \beta$ almost surely for each $i$. Let $S_n= \frac{1}{\sqrt{n}}\sum_{i=1}^n X_i$, and let $Z$ be a Gaussian with covariance $\Sigma$, then
\begin{align*}
W_2(S_n, Z)\leq \frac{6\sqrt{d}\beta\sqrt{\log n}}{\sqrt{n}}
\end{align*}
\end{theorem}

\begin{corollary}
\label{c:clt_sum}
Let $X_1...X_n$ be random vectors with mean 0, covariance $\Sigma$, and $\lrn{X_i}\leq \beta$ almost surely for each $i$. let $Y$ be a Gaussian with covariance $n\Sigma$. Then
\begin{align*}
W_2\lrp{\sum_i X_i, Y}\leq 6\sqrt{d}\beta \sqrt{\log n}
\end{align*}
\end{corollary}
This is simply taking the result of Theorem \ref{t:zhai} and scaling the inequality by $\sqrt{n}$ on both sides.

The following Lemma is taken from \cite{cheng2019quantitative} and included here for completeness.
\begin{lemma}\label{l:xlogxbound}
For any $c> 0$, $x> 3 \max\lrbb{\frac{1}{c} \log \frac{1}{c},0}$, the inequality
\begin{align*}
\frac{1}{c}\log(x) \leq x
\end{align*}
holds.
\end{lemma}
\begin{proof}
We will consider two cases:

\textbf{Case 1}: If $c\geq \frac{1}{e}$, then the inequality
$$\log(x) \leq c x$$ is true for all $x$.

\textbf{Case 2}: $c \leq \frac{1}{e}$.

In this case, we consider the Lambert W function, defined as the inverse of $f(x) = x e^x$. We will particularly pay attention to $W_{-1}$ which is the lower branch of $W$. (See Wikipedia for a description of $W$ and $W_{-1}$).

We can lower bound $W_{-1}(-c)$ using Theorem 1 from
\cite{chatzigeorgiou2013bounds}:
\begin{align*}
& \forall u>0,\quad W_{-1} (-e^{-u-1}) > -u-\sqrt{2u} -1\\
\text{equivalently}\quad &\forall c\in (0,1/e),\quad -W_{-1} (-c) < \log\lrp{\frac{1}{c}} + 1 + \sqrt{2\lrp{\log\lrp{\frac{1}{c}}-1}}-1 \\
&\qquad \qquad\qquad\qquad\qquad\ \ = \log\lrp{\frac{1}{c}} + \sqrt{2\lrp{\log\lrp{\frac{1}{c}}-1}}\\
&\qquad \qquad\qquad\qquad\qquad\ \ \leq 3\log \frac{1}{c}
\end{align*}

Thus by our assumption,
\begin{align*}
& x\geq 3\cdot \frac{1}{c}\log\lrp{\frac{1}{c}}\\
\Rightarrow & x\geq \frac{1}{c}\lrp{-W_{-1}(-c)}
\end{align*}

then $W_{-1}(-c)$ is defined, so
\begin{align*}
&x \geq \frac{1}{c}\max\lrbb{-W_{-1} (-c), 1}\\
\Rightarrow & (-cx) e^{-cx} \geq -c\\
\Rightarrow & x e^{-cx} \leq 1\\
\Rightarrow & \log(x) \leq cx
\end{align*}

The first implication is justified as follows:
$W_{-1}^{-1}: [-\frac{1}{\epsilon}, \infty) \to (-\infty, -1)$ is monotonically decreasing. Thus its inverse $W_{-1}^{-1}(y) = ye^y$, defined over the domain $(-\infty, -1)$ is also monotonically decreasing. By our assumption, $-cx \leq -3 \log \frac{1}{c} \leq -3$, thus $-cx \in (-\infty, -1]$, thus applying $W_{-1}^{-1}$ to both sides gives us the first implication.
\end{proof}

\end{section}

\section{Experiment Details}\label{apx:experiments}
In this section, we provide additional details of our experiments. In particular, we explain the CNN architecture that we use in our experiments. Denote a convolutional layer with $p$ input filters and $q$ output filters by $\mathsf{conv}(p, q)$, a fully connected layer with q outputs by $\mathsf{fully\_connect}(q)$, and a max
pooling operation with stride 2 as $\mathsf{pool2}$. Let $\mathsf{ReLU}(x) = \max\{x, 0\}$. Then the CNN architecture in our paper is the following:
\begin{align*}
&\mathsf{conv}(3, 32) \Rightarrow \mathsf{ReLU} \Rightarrow \mathsf{conv}(32, 64) \Rightarrow \mathsf{ReLU} \Rightarrow \mathsf{pool2} \Rightarrow \mathsf{conv}(64, 128) \Rightarrow \mathsf{ReLU} \Rightarrow
\mathsf{conv}(128, 128)  \\
&\Rightarrow \mathsf{ReLU} \Rightarrow \mathsf{pool2} \Rightarrow \mathsf{conv}(128, 256) \Rightarrow \mathsf{ReLU} \Rightarrow \mathsf{conv}(256, 256) \Rightarrow \mathsf{ReLU} \Rightarrow
\mathsf{pool2} \Rightarrow \mathsf{fully\_connect}(1024) \\
&\Rightarrow \mathsf{ReLU} \Rightarrow \mathsf{fully\_connect}(512) \Rightarrow \mathsf{ReLU} \Rightarrow \mathsf{fully\_connect}(10).
\end{align*}
\end{document}